\def\fortech{1}
\documentclass{vldb}
\usepackage{graphicx}
\usepackage[usenames, dvipsnames]{color}
\usepackage[caption=false,font=footnotesize]{subfig}
\usepackage{balance}  % for  \balance command ON LAST PAGE  (only there!)
\usepackage{amsmath}
\usepackage{amssymb}
\usepackage{enumerate}
\usepackage[pdftex, pdfborder={0 0 0}, colorlinks=false]{hyperref}
\hypersetup{pdfborder={0 0 0}, colorlinks = true, linkcolor=black, citecolor=black, filecolor=black, urlcolor=black}
\usepackage[all]{hypcap}
\usepackage{mathtools}
\usepackage{etoolbox}
\usepackage{multirow}
\usepackage[ruled]{algorithm2e}
\newsavebox\mybox
\usepackage{thmtools, thm-restate}

\allowdisplaybreaks

\newtheorem{definition}{Definition}
\newtheorem{proposition}{Proposition}

\DeclareMathOperator{\domain}{domain}
\DeclareMathOperator{\range}{range}
\DeclareMathOperator{\tovec}{vec}
\DeclareMathOperator{\tomat}{mat}
\DeclareMathOperator{\diag}{diag}
\DeclarePairedDelimiterX{\norm}[1]{\lVert}{\rVert}{#1}

\newbool{techreportbool}
\ifundef{\fortech}{
  \setbool{techreportbool}{false}
}{
  \setbool{techreportbool}{true}
}
\newcommand{\conferenceversion}[1]{\notbool{techreportbool}{#1}{}}
\newcommand{\techreport}[1]{\ifbool{techreportbool}{#1}{}}
\newcommand{\ConfOrTech}[2]{\notbool{techreportbool}{#1}{#2}}
\vldbTitle{}
\vldbAuthors{}
\vldbDOI{}

\begin{document}

% ****************** TITLE ****************************************

\title{Differentially Private Confidence Intervals for Empirical Risk Minimization}

% possible, but not really needed or used for PVLDB:
%\subtitle{[Extended Abstract]
%\titlenote{A full version of this paper is available as\textit{Author's Guide to Preparing ACM SIG Proceedings Using \LaTeX$2_\epsilon$\ and BibTeX} at \texttt{www.acm.org/eaddress.htm}}}

% ****************** AUTHORS **************************************


\numberofauthors{3} %  in this sample file, there are a *total*

\author{
% 1st. author
\alignauthor
Yue Wang\\
       %\affaddr{Computer Science and Engineering}\\
       \affaddr{The Pennsylvania State University}\\
       \email{yuw140@cse.psu.edu}
% 2nd. author
\alignauthor
Daniel Kifer\\
       %\affaddr{Computer Science and Engineering}\\
       \affaddr{The Pennsylvania State University}\\
       \email{dkifer@cse.psu.edu}
\alignauthor
Jaewoo Lee\\
       %\affaddr{Department of Computer Science}\\
       \affaddr{University of Georgia}\\
       \email{jwlee@cs.uga.edu}
% % 3rd. author
% \alignauthor Lars Th{\Large{\sf{\o}}}rv{$\ddot{\mbox{a}}$}ld\titlenote{This author is the
% one who did all the really hard work.}\\
%        \affaddr{The Th{\large{\sf{\o}}}rv{$\ddot{\mbox{a}}$}ld Group}\\
%        \affaddr{1 Th{\large{\sf{\o}}}rv{$\ddot{\mbox{a}}$}ld Circle}\\
%        \affaddr{Hekla, Iceland}\\
%        \email{larst@affiliation.org}
% \and  % use '\and' if you need 'another row' of author names
% % 4th. author
% \alignauthor Lawrence P. Leipuner\\
%        \affaddr{Brookhaven Laboratories}\\
%        \affaddr{Brookhaven National Lab}\\
%        \affaddr{P.O. Box 5000}\\
%        \email{lleipuner@researchlabs.org}
}
% There's nothing stopping you putting the seventh, eighth, etc.
% author on the opening page (as the 'third row') but we ask,
% for aesthetic reasons that you place these 'additional authors'
% in the \additional authors block, viz.
%\additionalauthors{Additional authors: John Smith (The Th{\o}rv\"{a}ld Group, {\texttt{jsmith@affiliation.org}}), Julius P.~Kumquat
%(The \raggedright{Kumquat} Consortium, {\small \texttt{jpkumquat@consortium.net}}), and Ahmet Sacan (Drexel University, {\small \texttt{ahmetdevel@gmail.com}})}
% Just remember to make sure that the TOTAL number of authors
% is the number that will appear on the first page PLUS the
% number that will appear in the \additionalauthors section.

\maketitle

\begin{abstract}
The process of data mining with differential privacy produces results that are affected by two types of noise: sampling noise due to data collection and privacy noise that is designed to prevent the reconstruction of sensitive information. In this paper, we consider the problem of designing confidence intervals for the parameters of a variety of differentially private machine learning models.  The algorithms can provide confidence intervals that satisfy differential privacy (as well as the more recently proposed concentrated differential privacy) and can be used with existing differentially private mechanisms that train models using objective perturbation and output perturbation.
\end{abstract}

\section{Introduction}\label{sec:intro}
Differential privacy \cite{dwork2006calibrating} is now seen as a gold standard for protecting individual data records while releasing aggregate information such as noisy count queries or parameters of data mining models. There has been a great deal of focus on answering queries and building models using differential privacy but much less focus on quantifying their uncertainty. Uncertainty estimates are needed by data users to understand how much they can trust a query answer or a data mining model. 

Uncertainty comes from two sources: uncertainty in the data and uncertainty due to privacy mechanisms. Uncertainty in the data is often referred to as sampling error -- the data are a sample from a larger population (so a different sample could yield different results). %\footnote{The probabilistic database literature also studies uncertainty due to errors in the data \cite{SuciuProbDB}, which are also known as \emph{measurement error models} in statistics. We defer the treatment of this type of uncertainty to future work.}
Uncertainty due to privacy mechanisms comes from the fact that any useful algorithm that satisfies differential privacy must have randomized behavior. Both must be quantified in an uncertainty estimate.

%Both Sheffet \cite{sheffet2017differentially} and Barrientos et al. \cite{barrientos2017differentially} provided differentially private algorithms for uncertainty estimates of the parameters of linear regression models under approximate $(\epsilon,\delta)$-differential privacy and pure $\epsilon$-differential privacy, respectively.

In the setting we consider, a differentially private algorithm has trained a model and released its parameters. The end user would like to obtain confidence intervals around each parameter. These confidence intervals themselves must satisfy differential privacy. There has been very little work on this topic and, to the best of our knowledge, all of it has focused on linear regression \cite{sheffet2017differentially,barrientos2017differentially}.

On the other hand, differentially private model fitting algorithms such as objective perturbation \cite{chaudhuri2011differentially} and output perturbation \cite{chaudhuri2011differentially} can train a variety of models, such as logistic regression and SVM, and achieve state-of-the-art (or near state-of-the-art) accuracy on many datasets. However, they do not come with confidence intervals.

In this paper, we propose privacy-preserving algorithms for generating confidence intervals for differentially private models trained by the techniques of Chaudhuri et al. \cite{chaudhuri2011differentially}. We provide 
versions of these algorithms for pure $\epsilon$-differential privacy, as well as the recently introduced concentrated differential privacy (zCDP) \cite{bun2016concentrated}.

There are three basic steps in our framework. The first is to use either the output or objective perturbation techniques \cite{chaudhuri2011differentially} to provide model parameters. In the case of objective perturbation, the result satisfies both $\epsilon$-differential privacy as well as $\frac{\epsilon^2}{2}$-zCDP. In the case of output perturbation, the algorithms for differentially privacy and zCDP (concentrated differential privacy) are different. 
%to generate confidence intervals for the model parameter. In the first step, we solve the empirical risk minimization problem in a differentially private manner. As in \cite{chaudhuri2011differentially}, we have both objective perturbation and output perturbation based solutions. We optimize the objective perturbation solution from \cite{chaudhuri2011differentially} in Algorithm~\ref{alg:improvederm} with improved privacy parameter and (possibly) regularization coefficient \footnote{We improve the regularization coefficient $\Delta$ in Algorithm~\ref{alg:improvederm} when it is non-zero.}.
%The algorithm satisfies $\epsilon$-differential privacy as well as $(\epsilon^2/2)$-zCDP. This gives us the flexibility to do subsequent computations with either differential privacy or zero-concentrated differential privacy. On the other hand, the output perturbation algorithm can either apply the Laplace Mechanism to achieve differential privacy or the Gaussian Mechanism to achieve zero-concentrated differential privacy, too.
%
In the second step, we use Taylor's Theorem and the Central Limit Theorem to approximate the randomness in the model coefficients that is due to both the data and the privacy mechanisms. This approximation relies on properties of the data and thus necessitates a third step of estimating them using either differential privacy or zCDP. Thus, the overall privacy budget must be split into two phases: the budget allocated to getting the model parameters and the budget allocated to estimating uncertainty in the parameters.

In our experiments, we verify the accuracy of our confidence intervals and observe that under pure differential privacy, the confidence intervals for models trained with objective perturbation are shorter than those for models trained with output perturbation. However, under concentrated differential privacy, the confidence intervals for output perturbation are much smaller. 

Note that the goal of this paper is not to introduce new model fitting algorithms. The goal is to add capabilities for quantifying uncertainty in the model coefficients.

%Finally in the last step, we compute the Hessian and the covariance matrix of the Gaussian distribution from the previous step, and perturb them for privacy concerns. This step is also made to satisfy differential privacy and zero-concentrated differential privacy. Our matrix perturbation algorithm outputs symmetric and positive definite matrices as the input matrices are so in our case.
%Based on the above information, confidence intervals can be obtained on a sampling basis given the approximate distribution of the model parameter from the second step. When the empirical risk minimization and matrix perturbation are both done with Gaussian Mechanism, we propose a more efficient algorithm to obtain the confidence intervals without sampling since the model parameter then follows a Gaussian distribution.

To summarize, our contributions are the following.
\begin{list}{$\bullet$}{
\setlength{\itemsep}{0pt}
\setlength{\topsep}{3pt}
\setlength{\parsep}{3pt}
\setlength{\partopsep}{0pt}
\setlength{\leftmargin}{1em}}
\item To the best of our knowledge, this is the first paper that provides differentially private confidence intervals for models other than linear regression and our work is \emph{not} limited to any specific model -- it works for any model that can be trained using objective perturbation \cite{chaudhuri2011differentially}.
\item The confidence intervals can be made to satisfy different variations of differential privacy, including pure $\epsilon$-differential privacy \cite{dwork2006calibrating}, zero-mean concentrated differential privacy (zCDP) \cite{bun2016concentrated}, and approximate $(\epsilon,\delta)$-differential privacy \cite{DworkOurData}.
\item We empirically validate our confidence intervals using a variety of public datasets.
\item Finally, we provide a small improvement to the original objective perturbation model fitting technique ~\cite{chaudhuri2011differentially} by improving some of the constants in the algorithm.
%
%\item Our algorithms can be used to protect differential privacy and zero-concentrated differential privacy. The algorithms can protect approximate differential privacy by a conversion from zero-concentrated differential privacy.
%\item We derive the matrices together with their $L_2$ sensitivities for the applications of logistic regression and support vector machines. We check the performance of our methods in these applications with extensive experiments.
\end{list}

We discuss related work in Section~\ref{sec:relatedworks} and introduce the preliminaries and notation in Section~\ref{sec:preliminaries}. We derive confidence intervals for models trained with objective perturbation in Section \ref{sec:obj}. We derive confidence intervals for models trained with output perturbation  in Section~\ref{sec:output}. We show how to apply our algorithms to logistic regression and support vector machines in Section~\ref{sec:applications} and present experiments in Section~\ref{sec:experiments}. We present conclusions in Section~\ref{sec:conclusions}.
% Outline:
% \begin{itemize}
% \item Section 3: preliminaries: introduce notation, data comes from some distribution, ERM formula and strong convexity and logistic regression and svm, $\theta_0$ is the minimizer in the infinite data case, differential privacy, zcdp.
% \item Section 4: objective perturbation and confidence interval algorithms, first for dp then for zcdp. subsections for logistic regression and svm. Remember "is approximated by" means it is not a theorem. So the derivation of the confidence intervals algorithm is not a theorem, it is a sequence of mathematical approximations that should be described in the main text, in this section.
% \item Section 5: output perturbation and confidence interval algorithm, both for dp then for zcdp. subsections for logistic regression and svm
% \item Experiments: discuss datasets, discuss baselines (like variability intervals, any related work that can be used for a baseline?).
% \end{itemize}

\section{Related Work}\label{sec:relatedworks}

Differentially private training of data mining models has been extensively studied, for example, in \cite{chaudhuri2011differentially,FriedmanDT,kifer2012private,yu2014differentially,zhang2012functional,wu2015revisiting,bassily2014private,kasiviswanathan2016efficient,zhang2013privgene,jain2013differentially,wang2017differentially,zhang2017efficient,rubinstein2009learning,talwar2014private,talwar2015nearly,jain2014near,kasiviswanathan2017private,ligett2017accuracy,wang2018efficient}. However, such work provides model parameters without any uncertainty estimates (such as confidence intervals) about the parameters. To the best of our knowledge, the only exceptions are for linear regression \cite{sheffet2017differentially,barrientos2017differentially}.

Chaudhuri et al. \cite{chaudhuri2011differentially} studied a general class of models that (without privacy) are trained  with empirical risk minimization. They proposed two general approaches, called \emph{objective perturbation} and \emph{output perturbation} for training such models with differential privacy. 
%We improve the actual privacy parameter and (possibly) the regularization coefficient in their objective perturbation algorithm (Section~\ref{subsec:obj}), and further propose methods to generate the differentially private confidence intervals for the model parameter.
Subsequent work increased the set of models that can be trained \cite{kifer2012private,yu2014differentially,zhang2012functional,zhang2013privgene,wu2015revisiting}.
Kifer et al.~\cite{kifer2012private} extended the algorithm of ~\cite{chaudhuri2011differentially} by removing some differentiability requirements and allowing constraints in model training. 
 Yu et al.~\cite{yu2014differentially} solved the problem of differentially private penalized logistic regression with elastic-net regularization by extending the objective perturbation technique to any convex penalty function.
Zhang et al.~\cite{zhang2012functional} proposed the functional mechanism, which approximates models by polynomials. Subsequently, Zhang et al.~\cite{zhang2013privgene}  proposed a general solution based on genetic algorithms and a novel random perturbation technique called the enhanced exponential mechanism. Wu et al. \cite{wu2015revisiting}  proposed another output perturbation technique for learning tasks with convex and Lipschitz loss functions on a bounded domain. They relaxed the condition of differentiable loss functions in \cite{chaudhuri2011differentially}. However, we found that when both methods are applicable, the noise added by the output perturbation technique of Wu et al.~\cite{wu2015revisiting} is generally larger than the noise added by the output perturbation technique of Chaudhuri et al. \cite{chaudhuri2011differentially}.

High dimensional regression problems were also studied in \cite{bassily2014private,kasiviswanathan2016efficient}.
Bassily et al. \cite{bassily2014private} proposed new algorithms for the private convex ERM problem when the loss function is only Lipschitz and the domain of the optimization is bounded. They also proposed separate algorithms when the loss function is also strongly convex. They propose algorithms for both pure and approximate differential privacy.
Kasiviswanathan and Jin \cite{kasiviswanathan2016efficient}  improved the worst-case risk bounds of Bassily et al. \cite{bassily2014private} under differential privacy with access to full data. Moreover, with access to only the projected data and the projection matrix, they derived the excess risk bounds for generalized linear loss functions.

There has been some work on quantifying the uncertainty for differentially private models, mostly in the form of confidence intervals and hypothesis testing.

Differentially private hypothesis testing has been studied in \cite{uhler2013privacy,yu2014scalable,wang2015differentially,gaboardi2016differentially,rogers2017new,kakizaki2017differentially,cai2017priv,acharya2017differentially}. Uhler et al. \cite{uhler2013privacy} and Yu et al. \cite{yu2014scalable} conducted differentially private independence testing through $\chi^2$-tests with output perturbation, and adjusted the asymptotic distribution used to compute p-values. 
 Using input perturbation, Wang et al. \cite{wang2015differentially}, and later independently Gaboardi et al. \cite{gaboardi2016differentially} proposed  differentially private hypothesis testing for independence and goodness of fit. Kifer and Rogers \cite{rogers2017new} later proposed new test statistics for chi-squared testing that are more compatible with privacy noise. 
Kakizaki et al. \cite{kakizaki2017differentially} proposed the unit circle mechanism for independence testing on $2\times 2$ tables with known marginal sums.
Cai et al. \cite{cai2017priv} studied the sample complexity to conduct differentially private goodness of fit test with guaranteed type I and II errors. Later work by Acharya et al. \cite{acharya2017differentially} derived the upper and lower bounds on the sample complexity for goodness of fit and closeness testing under $(\epsilon,\delta)$-differential privacy.

Providing diagnostics for differentially private regression analysis was studied in \cite{chen2016differentially}, where Chen et al. designed differentially private algorithms to construct residual plots for linear regression and receiver operating characteristics (ROC) curves for logistic regression.

Work on differentially private confidence intervals includes \cite{d2015differential,sheffet2017differentially,karwa2017finite}. D'Orazio et al. \cite{d2015differential} and Karwa and Vadhan~\cite{karwa2017finite} did not study models, instead they constructed differentially private confidence intervals for a mean \cite{karwa2017finite} and the difference of two means \cite{d2015differential}.
In the context of model coefficients,
Sheffet \cite{sheffet2017differentially} studied $(\epsilon,\delta)$-differentially private Ordinary Least Squares Regression (OLS) and  generated confidence intervals for the parameters.
%In~\cite{karwa2017finite}, Karwa and Vadhan studied finite sample confidence intervals for the mean of a normal distribution with approximate differential privacy.% They improved upon previous works by removing the boundedness assumption on the data.
Barrientos et al. \cite{barrientos2017differentially} used differential privacy to quantify the uncertainty of the coefficients of differentially private linear regression models. They generated differentially private t statistics for each coefficient.

Thus the closest work related to ours is Sheffet \cite{sheffet2017differentially} and Barrientos et al. \cite{barrientos2017differentially}. While their work only targets linear regression, our work targets any models that can be trained under the objective perturbation and output perturbation techniques of Chaudhuri et al. \cite{chaudhuri2011differentially}, which include many models such as logistic regression and SVM, but excludes linear regression.

%Our work is different from the above ones since \cite{d2015differential,karwa2017finite} were focusing on quantifying for the mean estimator, \cite{sheffet2017differentially,barrientos2017differentially} worked on linear regression, and \cite{chen2016differentially} measured the differentially private linear regression model with residual plots and the logistic regression model with ROC curves. Though some of the work used the notion of approximate differential privacy, our algorithms can also be converted to protect approximate differential privacy from zCDP.

To obtain confidence intervals, we also need to privately estimate second order matrices from the data.
Perturbing second order matrices for data are common in privacy-preserving principal component analysis (PCA).
Chaudhuri et al. proposed to perturb the second order matrices with the exponential mechanism to achieve differential privacy in~\cite{chaudhuri2012near}.
With the SuLQ framework \cite{blum2005practical}, Blum et al. added Gaussian noise to the second moment matrix and used it in the PCA to protect a notion of $(\epsilon,\delta,T)$-Privacy.
Jiang et al.~\cite{jiang2016wishart} studied the problem of publishing differentially private second order matrices by adding proper Laplace or Wishart noise.
Dwork et al. worked on projecting the second moment matrix of data into the low dimensional space using the notion of approximate differential privacy in \cite{dwork2014analyze}.
Later in \cite{sheffet2015private}, Sheffet also discussed three techniques to get the second moment matrix while preserving the approximate differential privacy, with the matrices being positive-definite.     

Due to the structure of the matrices needed by our techniques, a spherical version of the Laplace Mechanism, introduced in the objective perturbation method \cite{chaudhuri2011differentially} to achieve differential privacy, or the Gaussian Mechanism \cite{bun2016concentrated} to achieve zero-mean concentrated differential privacy \cite{bun2016concentrated}, are most appropriate.

%In a more general setting, differential privacy has also been applied to robust statistics \cite{dwork2009differential} and M-Estimators \cite{lei2011differentially}. In \cite{smith2011privacy}, Smith showed that for a large class of statistical estimators that are asymptotically normal, their differentially private counterparts converge to the same asymptotic distributions as the non-private estimators with sufficiently large samples. Chaudhuri and Hsu worked out lower and upper bounds for the convergence rates of differentially private approximations to statistical estimators in \cite{chaudhuri2012convergence}.
\section{Preliminaries and Notation}\label{sec:preliminaries}
In this section,  we introduce notation used in the paper and then review the background of differential privacy and its variants, empirical risk minimization, and its applications to logistic regression and support vector machines.

Let $\mathcal{D}=\{(\vec{x}_1,y_1), \dots, (\vec{x}_n, y_n)\}$ be a set of $n$ records. Each record $i$ has a $d$-dimensional vector $\vec{x}_i$ of real numbers known as a \emph{feature vector} and each $y_i\in\{-1, 1\}$ is called the \emph{target}. Following \cite{chaudhuri2011differentially}, we require that each record is normalized so that $||\vec{x}_i||_2=1$.

\subsection{Differential Privacy}\label{subsec:reviewdp}

\begin{definition}{(Differential Privacy~\cite{dwork2006calibrating}).}
Given an $\epsilon>0$ and $\delta\geq 0$, a randomized mechanism $\mathcal{M}$ satisfies $(\epsilon,\delta)$-differential privacy if for all pairs of databases $\mathcal{D}, \mathcal{D}^{\prime}$ differing on the value of a record, and all $V\subseteq \range(\mathcal{M})$,
\[Pr(\mathcal{M}(\mathcal{D})\in V)\leq e^{\epsilon}\; Pr(\mathcal{M}(\mathcal{D}^{\prime})\in V) + \delta.\]
\end{definition}
When $\delta=0$, we refer to it as both $\epsilon$-differential privacy and pure differential privacy. When $\delta>0$, we refer to it as both $(\epsilon,\delta)$-differential privacy and approximate differential privacy. Another relaxation of differential privacy is known as zero-mean concentrated differential privacy, or $\rho$-zCDP for short.

\begin{definition}{(Zero-Concentrated Differential Privacy \\ (zCDP)~\cite{bun2016concentrated}).}\label{def:zcdp}
A randomized mechanism $\mathcal{M}$ satisfies $\rho$-zero-concentrated differential privacy (i.e., $\rho$-zCDP) if for all pairs of databases $\mathcal{D}$ and $\mathcal{D}^\prime$ that differ on the value of a single record and all $\alpha\in(1,\infty)$,
\[D_{\alpha}(\mathcal{M}(\mathcal{D})||\mathcal{M}(\mathcal{D}^{\prime}))\leq \rho\alpha,\]
where $D_{\alpha}(\mathcal{M}(\mathcal{D})||\mathcal{M}(\mathcal{D}^{\prime}))$ is the $\alpha$-$R\acute{e}nyi$ divergence between the distribution of $\mathcal{M}(\mathcal{D})$ and $\mathcal{M}(\mathcal{D}^{\prime})$.
\end{definition}

$\rho$-zCDP is weaker than pure differential privacy and stronger than approximate differential privacy. The following results make the relations between them precise.

\begin{proposition}{\cite{bun2016concentrated}}.\label{pro:dptozcdp}
If $\mathcal{M}$ satisfies $\epsilon$-differential privacy, then $\mathcal{M}$ satisfies $(\epsilon^2/2)$-zCDP.
\end{proposition}

\begin{proposition}{\cite{bun2016concentrated}}.\label{pro:zcdptodp}
If $\mathcal{M}$ satisfies $\rho$-zCDP then it satisfies $(\rho + 2\sqrt{\rho\log(1/\delta)},\delta)$-differential privacy.
\end{proposition}

Thus, we only focus on pure differential privacy and $\rho$-zCDP in this paper. All $\rho$-zCDP algorithms can be converted into algorithms for approximate differential privacy using Proposition \ref{pro:zcdptodp}.

The algorithms studied in this paper rely on the concept of $L_2$ sensitivity:
\begin{definition}{($L_2$-Sensitivity~\cite{chaudhuri2011differentially,bun2016concentrated}).}
The $L_2$-sensitivity for a (scalar- or vector-valued) function $f$ is
\[\Delta_2(f)=\max_{\mathcal{D},\mathcal{D}^{\prime}}\|f(\mathcal{D})-f(\mathcal{D}^{\prime})\|_2\]
for all pairs of databases $\mathcal{D},\mathcal{D}^{\prime}\in\domain(f)$ differing on the value of at most one entry.
\end{definition}

For example, the $L_2$ sensitivity is used to set the variance of the Gaussian Mechanism for $\rho$-zCDP.
\begin{proposition}{(Gaussian Mechanism \cite{bun2016concentrated}).}\label{pro:gaussmechanism}
Let $f$ be a vector-valued function (whose output is a vector of dimension $d$) with $L_2$ sensitivity $\Delta_2(f)$. Let $\sigma=\Delta_2(f)/\sqrt{2\rho}$. The Gaussian Mechanism, which outputs $f(\mathcal{D}) + N(\vec{0}, \sigma^2\mathbf{I}_d)$, satisfies $\rho$-zCDP.
%If the mechanism $\mathcal{M}:\mathcal{X}^n\rightarrow \mathbb{R}^d$ outputs a sample from the (multivariate) normal distribution $N(f(x), \sigma^2\mathbf{I}_d)$ for some function $f:\mathcal{X}^n\rightarrow \mathbb{R}^d$ on input $x$, then $\mathcal{M}$ satisfies $\rho$-zCDP for
%\[\rho=\frac{1}{2\sigma^2}\sup_{\substack{x,x^{\prime} \in\mathcal{X}^n\\ \text{differing on one entry}}} \|f(x)-f(x^{\prime})\|_2^2.\]
\end{proposition}

Both differential privacy and $\rho$-zCDP are invariant under post-processing \cite{dwork2006calibrating,bun2016concentrated}. That is, if a mechanism $\mathcal{M}$ satisfies $\epsilon$-differential  privacy (resp., $\rho$-zCDP), and if $A$ is any algorithm whose input is the output of $\mathcal{M}$, then the composite algorithm, which first runs $\mathcal{M}$ on the input data and then runs $A$ on the result satisfies $\epsilon$-differential privacy (resp., $\rho$-zCDP).

%\begin{proposition}{(Post-Processing).}
%Let $\mathcal{M}$ be a randomized mechanism that satisfies $\epsilon$-differential privacy (or $\rho$-zCDP) and $g$ be a randomized mechanism which takes the output of $\mathcal{M}$ as input. Then $f(\mathcal{M}(\cdot))$ satisfies $\epsilon$-differential privacy (or $\rho$-zCDP).
%\end{proposition}

Another useful property of these definitions is \emph{composition}, which allows the privacy parameter of a complicated algorithm be derived from the privacy parameters of its sub-components.

\begin{proposition}{(Composition \cite{dwork2006calibrating,bun2016concentrated}).}
Let $\mathcal{M}$ be a randomized mechanism that satisfies $\epsilon$-differential privacy (resp., $\rho$-zCDP) and $\mathcal{M}^{\prime}$ be a randomized mechanism that satisfies $\epsilon^{\prime}$-differential privacy (resp.,  $\rho^{\prime}$-zCDP). Then the composite algorithm $\mathcal{M}^*$ that, on input $\mathcal{D}$ outputs the tuple $(\mathcal{M}(\mathcal{D}), \mathcal{M}(\mathcal{D}^\prime))$  satisfies $(\epsilon+\epsilon^{\prime})$-differential privacy (resp., $(\rho+\rho^{\prime})$-zCDP).
\end{proposition}

\subsection{Empirical Risk Minimization}\label{subsec:reviewerm}
Empirical risk minimization is a common way of training data mining models. There is an assumption that the dataset $\mathcal{D}=\{(\vec{x}_1,y_1), \dots, (\vec{x}_n,y_n)\}$ is independently sampled 
 from some unknown distribution $F_0$. In this setting, the model has a parameter vector $\theta$ and a prediction function $g$. Its prediction for $y$ is $g(\vec{x},\theta)$. 
 
 To train the model, in the setting assumed by Chaudhuri et al. \cite{chaudhuri2011differentially}, one specifies a loss function in the form of $f(\vec{x},y,\theta) = f(y\theta^T\cdot\vec{x})$, and finds the $\theta$ that minimizes the empirical risk:
\begin{equation}\label{eq:nperm}
\hat{\theta}=\arg\min_{\theta}\frac{1}{n}\sum_{i=1}^n\left[f(\vec{x}_i,y_i,\theta)+c\|\theta\|_2^2\right].
\end{equation}

To satisfy differential privacy, Chaudhuri et al.  \cite{chaudhuri2011differentially}, proposed the \emph{objective perturbation technique} to add noise to the objective function and then produce minimizer of the perturbed objective:
\begin{align}
&\arg\min_\theta J_n(\theta,\mathcal{D}) = \arg\min_\theta \left[L_n(\theta,\mathcal{D})+\frac{1}{n}\beta^T\theta\right]\nonumber\\
\equiv & \arg\min_\theta \left(\frac{1}{n}\sum\limits_{i=1}^n \left[f(\vec{x}_i,y_i,\theta) + c||\theta||_2^2\right] + \frac{1}{n}\beta^T\theta\right),\label{eq:privateerm}
\end{align}
where $\beta$ is a zero-mean random variable with density
\begin{equation}\label{eq:noisedist}
\mathtt{v}(\beta)=\frac{1}{u}e^{-\gamma\|\beta\|_2},
\end{equation}
where $u$ is the normalizing constant, and $\gamma$ depends on the privacy budget and the $L_2$-sensitivity of $L_n(\cdot)$.

Their proof of privacy depends on the concept of strong convexity:
\begin{definition}{(Strong-Convexity).}
A function $f(\theta)$ over $\theta\in\mathbb{R}^d$ is said to be $\lambda$-strongly convex if for all $\alpha\in(0,1)$, $\theta$ and $\eta$,
\[f(\alpha\theta+(1-\alpha)\eta)\leq\alpha f(\theta)+(1-\alpha)f(\eta)-\frac{1}{2}\lambda\alpha(1-\alpha)\|\theta-\eta\|_2^2.\]
\end{definition}

%Throughout the paper, without loss of generality, we assume $\mathcal{D}=\{(\vec{x}_1,y_1),\dots$, $(\vec{x}_{n-1},y_{n-1})$, $(\vec{x}_n,y_n)\}$ and its neighbor $\mathcal{D}^{\prime}=\{(\vec{x}_1,y_1)$, $\dots$, $(\vec{x}_{n-1},y_{n-1})$, $(\vec{x}_z,y_z)\}$, so they only differ in some entry from the last record. 

\subsection{Logistic Regression and SVM}\label{subsec:reviewregression}
In the paper, we will work with the applications of logistic regression and support vector machines (SVM)\footnote{We use these two applications as examples, but our algorithms are not restricted to them.}.
In logistic regression, the goal is to predict $P(y=1~|~\vec{x})$ and this is done by modeling $P(Y=1~|~\vec{x})=S(\theta\cdot\vec{x})$, where $S$ is the sigmoid function:
\[S(z)=\frac{1}{1+\exp{(-z)}}=\frac{\exp{(z)}}{1+\exp{(z)}},\]
Logistic regression is trained in the ERM framework using the loss function
\[f(\vec{x},y,\theta)=\log{[1+ \exp{(-y\theta\cdot\vec{x})}]}.\]
%\[f(\vec{x},y,\theta)=f(y\theta^T\vec{x})= \log{[1+ \exp{(-y\theta^T\vec{x}})]}.\]

In support vector machines, the prediction for $y$ is 1 if $\theta\cdot\vec{x}\geq 0$ and is $-1$ otherwise. To train it in the ERM framework, we will use the Huberized hinge Loss \cite{chapelle2007training}, defined as follows:
\begin{equation*}
f_{Huber}(\vec{x},y,\theta)=\begin{cases}
0 & \text{if $z>1+h$}\\
\frac{1}{4h}(1+h-z)^2 & \text{if $|1-z|\leq h$}\\
1-z & \text{if $z<1-h$,}
\end{cases}
\end{equation*}
where $z=y\theta\cdot \vec{x}$ and where $h$ is a fixed constant \cite{chapelle2007training}.

\section{Confidence Intervals for Objective Perturbation}\label{sec:obj}
In this section, we show how to obtain confidence intervals for models trained by objective perturbation \cite{chaudhuri2011differentially}. For completeness, we present a slightly improved version of the algorithm in Section \ref{subsec:obj} and then derive the confidence interval algorithm in Sections \ref{subsec:foundation}, \ref{subsec:hessiancov}, and \ref{subsec:generation}.

%we work on the ERM with objective perturbation and propose algorithms to generate the confidence intervals for the model parameter.
%Our algorithms can be used to protect both pure differential privacy and zCDP. We will break the process into the steps of solving the ERM with objective perturbation, deriving the theoretical foundations to generate the confidence intervals, computing the matrices, and summarizing with a complete algorithm for the confidence intervals generation.

\subsection{Objective Perturbation}\label{subsec:obj}
The objective perturbation algorithm modifies the ERM framework by randomly drawing a noise vector $\beta$ from a spherical version of the Laplace distribution (see Equation~\ref{eq:noisedist}). Then, instead of minimizing the original ERM objective $\frac{1}{n}\sum\limits_{i=1}^n \left[f(\vec{x}_i,y_i,\theta) + c||\theta||_2^2\right]$, it modifies it by adding  $\frac{1}{n}\beta^T\theta$ and then minimizes it with respect to $\theta$. The version of the techniques shown in Algorithm \ref{alg:improvederm} slightly differs from the original \cite{chaudhuri2011differentially} in the first line, allowing it to use less noise.

%To solve the empirical risk minimization with objective perturbation (Equation~\ref{eq:privateerm}), we improve upon Algorithm 2 from \cite{chaudhuri2011differentially} by optimizing the actual privacy parameter used in it, i.e., we improve the logarithm term by a factor of 2 when computing $\epsilon^{\prime}$. We present the improved algorithm in Algorithm~\ref{alg:improvederm}.

\setlength{\algoheightrule}{0pt}
\setlength{\algotitleheightrule}{0pt}
\begin{algorithm}[h]
 \LinesNumbered
 \DontPrintSemicolon
 %\SetAlgoLined
 \SetKwInOut{Input}{input}
 \SetKwInOut{Output}{output}
 \begin{lrbox}{\mybox}
 \begin{minipage}{\hsize}
 \caption{Objective Perturbation} \label{alg:improvederm}
 \Input{Data $\mathcal{D}=\{(\vec{x}_i,y_i)\}_{i=1}^n$, privacy budget $\epsilon$, loss function $f$ with $|f^{\prime\prime}(\cdot)|\leq t$, coefficient $c$ with $c\geq \frac{t}{2n(e^{\epsilon}-1)}$}
 $\epsilon^{\prime}\leftarrow \epsilon-\log{\left(1+\frac{t}{2nc}\right)}$\;
% If $\epsilon^\prime > 0$ then $\Delta\gets 0$ else  $\Delta\leftarrow \frac{t}{n(e^{\epsilon/2}-1)}-2c$ and $\epsilon^{\prime}=\frac{\epsilon}{2}$~\;
%
% \If{$\epsilon^{\prime}>0$}{
% $\Delta\leftarrow 0$\;
% }\Else{
% $\Delta\leftarrow \frac{t}{n(e^{\epsilon/2}-1)}-2c$\;
% $\epsilon^{\prime}=\frac{\epsilon}{2}$\;
% }
 Sample a $d$-dimensional vector $\beta$ with density from Equation~\ref{eq:noisedist} with $\gamma=\epsilon^{\prime}/2$\;
$\tilde{\theta}\gets \arg\min_\theta \left(\frac{1}{n}\sum\limits_{i=1}^n \left[f(\vec{x}_i,y_i,\theta) + c||\theta||_2^2\right] + \frac{1}{n}\beta^T\theta\right)$\;
% Output $\tilde{\theta}\leftarrow\arg\min_{\theta}\left[J_n(\theta,\mathcal{D})+\frac{1}{2}\Delta\|\theta\|_2^2\right]$\;
 Output $\tilde{\theta}$
 \end{minipage}%
 \end{lrbox}
 \hspace*{-10pt}\framebox[\columnwidth]{\hspace*{15pt}\usebox\mybox\par}
\end{algorithm}

\begin{restatable}{theorem}{thmerm}
\label{thm:erm2}
If the loss function $f(\cdot)$ is convex and doubly differentiable, with $|f^{\prime}(\cdot)|\leq 1$ and $|f^{\prime\prime}(\cdot)|\leq t$, then Algorithm~\ref{alg:improvederm} satisfies $\epsilon$-differential privacy whenever all the feature vectors $\vec{x}_i$ have $||\vec{x}_i||_2\leq 1$.
\end{restatable}
The proof of Theorem~\ref{thm:erm2} is in \ConfOrTech{the online Appendix~\ref{app:erm2}}{Appendix~\ref{app:erm2}}.

In order to achieve $\rho$-zCDP, we use Proposition \ref{pro:dptozcdp} to conclude that the algorithm satisfies $\frac{\epsilon^2}{2}$-zCDP.

%We run Algorithm~\ref{alg:improvederm} to solve the ERM problem in order to protect differential privacy. To protect zCDP for it, we can just run the same algorithm and achieve $(\epsilon^2/2)$-zCDP.

\subsection{Confidence Intervals}\label{subsec:foundation}
In this section, we describe one of our main contributions -- the construction of confidence intervals for objective perturbation.
Set $J_n(\theta) = \frac{1}{n}\sum\limits_{i=1}^n \left[f(\vec{x}_i,y_i,\theta) + c||\theta||_2^2\right] + \frac{1}{n}\beta^T\theta.$

 Let  $\tilde{\theta}$ be the privacy preserving parameters output by the objective perturbation algorithm. Let $\theta_0$ be the non-private solution we would get if we had infinite data (i.e. the true parameter vector). Since the noise in Algorithm \ref{alg:improvederm} is divided by $n$, then $\theta_0$ is also the privacy-preserving solution one would obtain with infinite data and $E[\nabla J_n(\theta_0, \mathcal{D})]=\vec{0}$, where the expectation is taken over the data and $\beta$ (note that $\beta$ has $\vec{0}$ mean).

Expanding the Taylor series of $J_n$ around $\tilde{\theta}$ and noting that the gradient of $J_n$ at $\tilde{\theta}$ is 0 by construction (since $\tilde{\theta}$ minimizes $J_n$), we have
%of the objective perturbation solution to Equation~\ref{eq:privateerm} (i.e. numerical optimizer based on the specific sample data).
%Let $\theta_0$ be the infinite sample minimizer to Equation~\ref{eq:nperm}. Since the noise added for objective perturbation has zero-mean, $\theta_0$ is also the infinite sample minimizer to Equation~\ref{eq:privateerm}.
%Using Taylor's theorem around $\tilde{\theta}$ applied to the gradient of the objective:
\begin{align*}
\nabla J_n(\theta_0) &\approx \nabla J_n(\tilde{\theta}) + H[J_n(\tilde{\theta})](\theta_0 - \tilde{\theta})\\
&= H[J_n(\tilde{\theta})](\theta_0 - \tilde{\theta}),
\end{align*}
where $H[J_n(\tilde{\theta})]$ is the Hessian (matrix of second derivatives) of $J_n$ evaluated at $\tilde{\theta}$.

Now, $\nabla J_n(\theta_0)$ is equal to $\frac{1}{n}\beta^T$ plus the average $n$ terms -- one for each $\vec{x}_i$. This means that by the Central Limit Theorem, $\sqrt{n} \nabla J_n(\theta_0)$ can be approximated by the sum of $\frac{1}{\sqrt{n}}\beta^T$ and $N(\vec{0}, \Sigma)$, where $N(\vec{0},\Sigma)$ is a zero-mean Gaussian with covariance matrix:
%Since $\theta_0$ is the infinite sample minimizer to $J_n(\cdot)$, $\sqrt{n}\nabla J_n(\theta_0) $ has mean $\vec{0}$ and so can be approximated by the convolution of a Gaussian distribution $N(\vec{0}, \Sigma)$ and a noise distribution,
%where 
\[\Sigma = E\left[\left(\nabla \Big(f(\vec{x}, y,\theta_0) + c||\theta_0||_2^2\Big)\right)\left(\nabla \Big(f(\vec{x},y, \theta_0) + c||\theta_0||_2^2\Big)\right)^T\right].\]
%and the noise distribution is sampled as $\frac{1}{\sqrt{n}}$ times a sample from the noise distribution of $\beta$.

If the Hessian and covariance matrices were known, we could combine the two approximations for $\nabla J_n(\theta_0)$ as follows. Let $G$ be a random variable with distribution  $N(\vec{0}, \Sigma)$. Then
%
%Therefore we can get samples from $(\theta_0 - \tilde{\theta})$ by sampling from the appropriate noisy gaussian and then multiplying it by the inverse of $H[J_n(\tilde{\theta})]$:
\begin{align*}
%\sqrt{n}\nabla J_n(\theta_0) &\approx G+\tilde{\beta}/\sqrt{n}\\
\sqrt{n}H[J_n(\tilde{\theta})](\theta_0-\tilde{\theta}) &\approx G+\tilde{\beta}/\sqrt{n},
\end{align*}
where $\tilde{\beta}$ is a fresh random variable that follows the same distribution as $\beta$.
However, since the Hessian and the covariance matrix of $G$, are unknown, we will need to obtain privacy preserving estimates $\tilde{H}$ and  $\tilde{G}$.
Substituting these privacy-preserving estimates and performing simple algebra, we obtain:
%, $G\sim N(\vec{0}, \Sigma)$, and the above equation gives
\begin{equation}\label{eq:objsample}
\theta_0-\tilde{\theta}\approx \tilde{H}[J_n(\tilde{\theta})]^{-1}(\tilde{G}+\tilde{\beta}/\sqrt{n})/\sqrt{n}
\end{equation}
(where $\tilde{G}=N(\vec{0},\tilde{\Sigma})$).
%due to the positive definiteness of the Hessian.
We next discuss how to estimate the Hessian and covariance matrix and how to use them with Equation \ref{eq:objsample} to produce confidence intervals for each element of $\theta_0$.

\subsection{Computations of the Hessian and Covariance Matrix}\label{subsec:hessiancov}

If privacy was not a concern, 
the Hessian would be computed as:
\begin{equation}\label{eq:hessian}
H[J_n(\theta)]=\frac{1}{n}\sum_{i=1}^n H[f(\vec{x}_i,y_i,\theta)]+2c\mathbf{I},
%&=\frac{1}{n}\sum_{i=1}^n S(-y_i\theta^T\vec{x}_i)S(y_i\theta^T\vec{x}_i)\vec{x}_i\vec{x}_i^T+2c\mathbf{I}.
\end{equation}
and the covariance matrix $\Sigma$ would be estimated as:
\begin{align}
\Sigma =& E\left[\left(\nabla \Big(f(\vec{x},y, \theta_0) + c||\theta_0||_2^2\Big)\right)\left(\nabla \Big(f(\vec{x},y, \theta_0) + c||\theta_0||_2^2\Big)\right)^T\right]\nonumber\\
=& E\left[\left(\nabla f(\vec{x},y,\theta_0)+2c\theta_0\right)\left(\nabla f(\vec{x},y,\theta_0)+2c\theta_0\right)^T\right]\nonumber\\
=& E\left\{\nabla(f(\vec{x},y,\theta_0))[\nabla f(\vec{x},y,\theta_0)]^T\right\}+2cE[\nabla f(\vec{x},y,\theta_0)]\theta_0^T\nonumber\\
&+2c\theta_0E[\nabla f(\vec{x},y,\theta_0)]^T+4c^2\theta_0\theta_0^T\nonumber\\
=& E\left\{\nabla(f(\vec{x},y,\theta_0))[\nabla f(\vec{x},y,\theta_0)]^T\right\}-4c^2\theta_0\theta_0^T\nonumber\\
\approx&\frac{1}{n}\sum_{i=1}^n\nabla f(\vec{x}_i,y_i,\tilde{\theta})[\nabla f(\vec{x}_i,y_i,\tilde{\theta})]^T -4c^2\tilde{\theta}\tilde{\theta}^T
,\label{eq:cov}
\end{align}
where second-to-last step is obtained from the fact that $E[\nabla J_n(\theta_0)]=\vec{0}$ from which it follows that $E[\nabla f(\vec{x},y,\theta_0)]+2c\theta_0=\vec{0}$.

%In practice, we usually estimate the true mean with the empirical mean, so the covariance matrix $\Sigma$ can be estimated by: 
%\begin{align*}
%\frac{1}{n}\sum_{i=1}^n\nabla f(\vec{x}_i,y_i,\theta_0)[\nabla f(\vec{x}_i,y_i,\theta_0)]^T -4c^2\theta_0\theta_0^T,
%\approx &\frac{1}{n}\sum_{i=1}^n \nabla f(\vec{x}_i,y_i,\tilde{\theta})[\nabla f(\vec{x}_i,y_i,\tilde{\theta})]^T-4c^2\tilde{\theta}\tilde{\theta}^T.
%= &\frac{1}{n}\sum_{i=1}^n S(-y_i\theta_0^T\vec{x}_i)^2\vec{x}_i\vec{x_i}^T-4c^2\theta_0\theta_0^T\\
%\approx &\frac{1}{n}\sum_{i=1}^n S(-y_i\tilde{\theta}^T\vec{x}_i)^2\vec{x}_i\vec{x_i}^T-4c^2\tilde{\theta}\tilde{\theta}^T.
%\end{align*}
%and we also estimate $\theta_0$ with $\tilde{\theta}$ for $\Sigma$.

However, since privacy is indeed a concern, we need to obtain estimates of the Hessian and covariance matrix using either $\epsilon$-differential privacy or $\rho$-zCDP. The same algorithm works for both matrices and is shown in Algorithm \ref{alg:privspdmat}. 

The algorithm takes the matrix $M$, which is either the Hessian (computed as in Equation~\ref{eq:hessian}) or the covariance matrix (computed as in Equation~\ref{eq:cov}). It also takes the $L_2$ sensitivity of these matrices (we show how to compute the sensitivities for logistic regression and SVM in Section \ref{sec:applications}). It uses the $L_2$ sensitivity to determine the variance of the noise that must be added.\footnote{ 
Since the data are spherical (i.e. $||\vec{x}_i||_2=1$), it turns out that using the $L_2$ sensitivity is most appropriate.} The distribution of this noise depends on whether we want to use differential privacy or zCDP.

These resulting noisy matrices might not be symmetric positive-semidefinite (even though the Hessian and covariance matrices must have those properties). Thus we add a postprocessing step to make the matrix symmetric and have all eigenvalues at least $2c$.

%We will need to perturb the Hessian (Equation~\ref{eq:hessian}) and the covariance matrix (Equation~\ref{eq:cov}) for privacy concerns.
%To protect differential privacy, we apply the Laplace Mechanism to perturb the matrices based on their $L_2$-sensitivity to get the differentially private counterpart.
%To protect zCDP, we apply the Gaussian Mechanism stated in Proposition~\ref{pro:gaussmechanism} to perturb them based on their $L_2$ sensitivity.
%Algorithm~\ref{alg:privspdmat} summarizes the process.
%Moreover, since we will be working on perturbing symmetric and positive definite matrices,  Algorithm~\ref{alg:privspdmat} has additional steps to turn the resulting noisy matrix into symmetric and positive definite ones.

\setlength{\algoheightrule}{0pt}
\setlength{\algotitleheightrule}{0pt}
\begin{algorithm}[h]
 \LinesNumbered
 \DontPrintSemicolon
 %\SetAlgoLined
 \SetKwInOut{Input}{input}
 \SetKwInOut{Output}{output}
 \begin{lrbox}{\mybox}
 \begin{minipage}{\hsize}
 \caption{Private Symmetric Positive Definite Matrix (PrivSPDMat)} \label{alg:privspdmat}
 \Input{Matrix $M \in \mathbb{R}_{d\times d}$, $L_2$ sensitivity $Sens(M)$, privacy budget $\phi$, parameter $c$}
 \uIf{Requiring $\epsilon$-differential privacy with $\epsilon=\phi$}{
 Sample a $d^2$-dimensional vector $\eta$ with density from Equation~\ref{eq:noisedist} with $\gamma=\frac{\phi}{Sens(M)}$\;
 }
 \Else{
 \tcp{ for $\rho$-zCDP with $\rho=\phi$}
 Sample a noise vector $\eta$ from $N\left(\vec{0}, \frac{Sens(M)^2}{2\phi}\mathbf{I}_{d}\right)$\;
 }
 Reshape $\eta$ to a $d\times d$ matrix $\tomat{(\eta)}$\;
 $\tilde{M}\leftarrow M+\tomat{(\eta)}$\;
 $\tilde{M}\leftarrow (\tilde{M}+\tilde{M}^T)/2$\;
 Let $V\diag{(\Lambda)}=\tilde{M}V$ be the eigen-decomposition for $\tilde{M}$\;
 \tcp{columns of $V$ are orthonormal eigenvectors}
 \For{$i\gets 1$ \KwTo $d$}{
 $\Lambda[i]\leftarrow \max{(\Lambda[i], 2c)}$\;
 }
 $\tilde{M}\leftarrow V\diag{(\Lambda)}V^{T}$\;
 Output $\tilde{M}$\;
 \end{minipage}%
 \end{lrbox}
 \hspace*{-10pt}\framebox[\columnwidth]{\hspace*{15pt}\usebox\mybox\par}
\end{algorithm}

\begin{restatable}{lemma}{privspdmat}
\label{lem:privspdmat}
Algorithm~\ref{alg:privspdmat} satisfies $\phi$-differential privacy and $\phi$-zCDP.
\end{restatable}
The proof of Lemma~\ref{lem:privspdmat} is in \ConfOrTech{the online Appendix~\ref{app:privspdmat}}{Appendix~\ref{app:privspdmat}}.

\subsection{Putting It All Together: Confidence Intervals Generation}\label{subsec:generation}
The overall algorithm is shown in Algorithm \ref{alg:ciobj}. It first splits the privacy budget into 3 pieces $\phi_1, \phi_2, \phi_3$. Using privacy budget $\phi_1$, it runs the objective perturbation algorithm to provide privacy-preserving model parameters $\tilde{\theta}$. Privacy budget $\phi_2$ is used to provide a privacy-preserving estimate of the Hessian $\tilde{H}$ and privacy budget $\phi_3$ is used to provide a privacy preserving estimate of the covariance matrix $\tilde{\Sigma}$. Once these quantities are obtained, it can use Equation~\ref{eq:objsample}. This equation says that the distribution of $\theta_0 - \tilde{\theta}$ can be approximated by sampling $\tilde{G}$ from $N(\vec{0},\tilde{\Sigma})$, $\tilde{\beta}$ from Equation \ref{eq:noisedist} and then plugging them into Equation \ref{eq:objsample} with $\tilde{H}$ instead of the true Hessian. By obtaining many such samples $z_1,\dots, z_m$ where each $z_i$ is a $d$-dimensional vector (because $\theta_0$ and $\tilde{\theta}$ are $d$-dimensional), for each dimension $j$ we take an interval $(a_j,b_j)$ that covers $1-\alpha$ (e.g., 95\%) of the $z_i[j]$. Then the estimated confidence interval for $\theta_0[j]$ is $(\tilde{\theta}[j] +a_j, \tilde{\theta}[j] + b_j)$. Note that this sampling step is strict postprocessing -- never accesses the original data and it only uses privacy preserving estimates from the previous steps. %Thus by the composition theorems of differential privacy and $\rho$-zCDP, the overall privacy cost is $\phi_1+\phi_2+\phi_3$.

%Based on the objective perturbation solution to the ERM problem discussed in Section~\ref{subsec:obj}, the theoretical foundations to generate the confidence intervals in Section~\ref{subsec:foundation} and the computations of the matrices in Section~\ref{subsec:hessiancov}, 
%we use Algorithm~\ref{alg:ciobj} to get the private confidence intervals (with either differential privacy or zCDP) for $\theta_0$ in the objective perturbation based ERM.

\setlength{\algoheightrule}{0pt}
\setlength{\algotitleheightrule}{0pt}
\begin{algorithm}[h]
 \LinesNumbered
 \DontPrintSemicolon
 %\SetAlgoLined
 \SetKwInOut{Input}{input}
 \SetKwInOut{Output}{output}
 \begin{lrbox}{\mybox}
 \begin{minipage}{\hsize}
 \caption{Private $(1-\alpha)$-Confidence Intervals for $\theta_0$ trained with Objective Perturbation} \label{alg:ciobj}
 \Input{Data $\mathcal{D}=\{(\vec{x}_i,y_i)\}_{i=1}^n$, privacy budgets $\phi_1$, $\phi_2$ and $\phi_3$, parameters $c$, $t$, $f$ used by objective perturbation (Algorithm \ref{alg:improvederm}), the number of postprocessing samples $m$ to generate, confidence level $\alpha$}
 $\tilde{\theta}\leftarrow$ ObjPerturb($\mathcal{D}$, $\phi_1$, $t$, $c$) \tcp{calling Algorithm~\ref{alg:improvederm}}
 $\epsilon^{\prime}\leftarrow \epsilon^{\prime}$ in Algorithm~\ref{alg:improvederm}\;
 $H[J_n(\tilde{\theta})] \leftarrow \frac{1}{n}\sum_{i=1}^n H[f(\vec{x}_i,y_i,\tilde{\theta})]+2c\mathbf{I}$\;
 $\tilde{H}[J_n(\tilde{\theta})] \leftarrow$ PrivSPDMat($H[J_n(\tilde{\theta})]$, $Sens(H[J_n(\tilde{\theta})])$,  $\phi_2$, $c$) \tcp{calling Algorithm~\ref{alg:privspdmat}}
 $\Sigma \leftarrow \frac{1}{n}\sum_{i=1}^n \nabla f(\vec{x}_i,y_i,\tilde{\theta})[\nabla f(\vec{x}_i,y_i,\tilde{\theta})]^T-4c^2\tilde{\theta}\tilde{\theta}^T$\;
  $\tilde{\Sigma}\leftarrow$ PrivSPDMat($\Sigma$, $Sens(\Sigma)$, $\phi_3$, $c$)\;
  Generate $m$ i.i.d. samples $G_i$ $(i=1,\dots,m)$ from $N(\vec{0},\tilde{\Sigma})$\;
  Generate $m$ i.i.d samples $\beta_i$ $(i=1,\dots,m)$ with density from Equation~\ref{eq:noisedist} with $\gamma=\epsilon^{\prime}/2$ (same $\gamma$ parameter as used in Algorithm \ref{alg:improvederm})\;
  \For{$i\gets 1$ \KwTo $m$}{
  $\theta^{(i)}\leftarrow\tilde{\theta}+\tilde{H}[J_n(\tilde{\theta})]^{-1}(G_i+\frac{1}{\sqrt{n}}\beta_i)/\sqrt{n}$\;
  }
  \For{$j\gets 1$ \KwTo $d$}{
    $(\theta_L[j], \theta_R[j])\leftarrow (1-\alpha)$-confidence interval for $\theta^{(1)}[j],\dots,\theta^{(m)}[j]$\;
  }
  Output $\theta_L$, $\theta_R$\;
 \end{minipage}%
 \end{lrbox}
 \hspace*{-10pt}\framebox[\columnwidth]{\hspace*{15pt}\usebox\mybox\par}
\end{algorithm}

\begin{restatable}{theorem}{thmciobj}
\label{thm:ciobj}
Under the conditions of Theorem \ref{thm:erm2},
%If the loss function $f(\cdot)$ is convex and doubly differentiable, with $|f^{\prime}(\cdot)|\leq 1$ and $|f^{\prime\prime}(\cdot)|\leq t$, then 
Algorithm~\ref{alg:ciobj} satisfies $(\phi_1+\phi_2+\phi_3)$-differential privacy and $(\phi_1^2/2+\phi_2+\phi_3)$-zCDP.
\end{restatable}
The proof of Theorem~\ref{thm:ciobj} is in \ConfOrTech{the online Appendix~\ref{app:ciobj}}{Appendix~\ref{app:ciobj}}.

\section{Confidence Intervals for Output Perturbation}\label{sec:output}
In this section, we provide confidence intervals for model parameters learned with output perturbation rather than objective perturbation. Again, we will have algorithms for both differential privacy and zCDP. We will follow similar steps as Section~\ref{sec:obj} to obtain the intervals.

\subsection{Output Perturbation}\label{subsec:output}
We first review the output perturbation method of Chaudhuri et al. \cite{chaudhuri2011differentially}. Then we will explain how to obtain confidence intervals for the resulting parameters in Section \ref{subsec:foundationoutput} (recall that they must account for noise due to the data being a sample as well as noise due to privacy).

In output perturbation, the first step is to compute the non-private parameters $\hat{\theta}$:
\begin{equation}
\hat{\theta}=\arg\min_{\theta}\frac{1}{n}\sum_{i=1}^n\left[f(\vec{x}_i,y_i,\theta)+c\|\theta\|_2^2\right],
\end{equation}
and then add noise to them \cite{chaudhuri2011differentially}. The $L_2$ sensitivity of $\hat{\theta}$ is $1/(nc)$ \cite{chaudhuri2011differentially} and so for $\epsilon$-differential privacy, they release $\hat{\theta} + \beta$, where $\beta$ has the distribution from Equation~\ref{eq:noisedist} with parameter $\gamma=nc\epsilon$. To obtain $\rho$-zCDP one uses the Gaussian Mechanism instead, and samples $\beta$ from the   multivariate normal distribution $N\left(\vec{0},\frac{1}{2\rho(nc)^2}\mathbf{I}_d\right)$.
%Let $\hat{\theta}$ be the non-private solution to the ERM problem (Equation~\ref{eq:nperm}) based on the samples $\mathcal{D}$, i.e., $\hat{\theta}=\arg\min_{\theta}L_n(\theta,\mathcal{D})$.
%
%
%To protect differential privacy for $\hat{\theta}$ with output perturbation, we simply apply the Laplace Mechanism.
%Corollary~8 from~\cite{chaudhuri2011differentially} states that the $L_2$-sensitivity for $\hat{\theta}$ is $1/(nc)$, then adding a noise vector $\beta$ with density from Equation~\ref{eq:noisedist} with $\gamma=nc\epsilon$ to $\hat{\theta}$ protects $\epsilon$-differential privacy \footnote{Note the regularization coefficient $\Lambda$ from \cite{chaudhuri2011differentially} is twice ours, i.e., $\Lambda=2c$ since $\frac{1}{2}\|\theta\|_2^2$ is 1-strongly convex. The proof for its 1-strong convexity is in \ConfOrTech{the online Appendix~\ref{app:convexity}}{Appendix~\ref{app:convexity}}.}. 
%
%Perturbing $\hat{\theta}$ by the Gaussian Mechanism satisfies zCDP. By Proposition~\ref{pro:gaussmechanism} and the  $L_2$-sensitivity, adding a noise vector $\beta$ from the multivariate normal distribution $N\left(\vec{0},\frac{1}{2\rho(nc)^2}\mathbf{I}_d\right)$ to $\hat{\theta}$ satisfies $\rho$-zCDP.
%
%We write $\tilde{\theta}=\hat{\theta}+\beta$ to denote the perturbed coefficient. 
%
Algorithm~\ref{alg:ermoutput} summarizes their output perturbation technique.

\setlength{\algoheightrule}{0pt}
\setlength{\algotitleheightrule}{0pt}
\begin{algorithm}[h]
 \LinesNumbered
 \DontPrintSemicolon
 %\SetAlgoLined
 \SetKwInOut{Input}{input}
 \SetKwInOut{Output}{output}
 \begin{lrbox}{\mybox}
 \begin{minipage}{\hsize}
 \caption{Output Perturbation (ERMOutput)} \label{alg:ermoutput}
 \Input{Data $\mathcal{D}=\{(\vec{x}_i,y_i)\}_{i=1}^n$,  privacy budget $\phi$, regularization coefficient $c$.}
 $\hat{\theta}\leftarrow \arg\min_{\theta}\frac{1}{n}\sum_{i=1}^n\left[f(\vec{x}_i,y_i,\theta)+c\|\theta\|_2^2\right]$\;
 \uIf{Requiring $\epsilon$-differential privacy with $\epsilon=\phi$}{
 Sample a noise vector $\beta$ with density from Equation~\ref{eq:noisedist}~ with $\gamma=nc\phi$\;
 }
 \Else{
 \tcp{for $\rho$-zCDP with $\rho=\phi$}
 Sample a noise vector $\beta\sim N\left(\vec{0},\frac{1}{2\phi(nc)^2}\mathbf{I}_d\right)$\;
 }
 $\tilde{\theta}\leftarrow \hat{\theta}+\beta$\;
  Output $\tilde{\theta}$\;
 \end{minipage}%
 \end{lrbox}
 \hspace*{-10pt}\framebox[\columnwidth]{\hspace*{15pt}\usebox\mybox\par}
\end{algorithm}

\begin{restatable}{theorem}{thmermoutput}(\cite{chaudhuri2011differentially})
\label{thm:ermoutput}
If the loss function $f(\cdot)$ is convex and differentiable with $|f^{\prime}(\cdot)|\leq 1$, then Algorithm~\ref{alg:ermoutput} satisfies $\phi$-differential privacy and $\phi$-zCDP.
\end{restatable}
%The proof of Theorem~\ref{thm:ermoutput} is in \ConfOrTech{the online Appendix~\ref{app:ermoutput}}{Appendix~\ref{app:ermoutput}}.

\subsection{Confidence Intervals}\label{subsec:foundationoutput}
Now we discuss our main contribution in this section, obtaining confidence intervals for the parameters returned by output perturbation. %We account for the noise added to $\hat{\theta}$ when generating the confidence intervals for $\theta_0$. 
Recall $\theta_0$ is the infinite sample minimizer to $L_n(\theta)=\frac{1}{n}\sum_{i=1}^n\left[f(\vec{x}_i,y_i,\theta)+c\|\theta\|_2^2\right]$ (i.e. when $n\rightarrow \infty$) while $\hat{\theta}$ is the finite sample minimizer and $\tilde{\theta}$ is the privacy preserving output of Algorithm \ref{alg:ermoutput} that we get by using privacy budget $\phi_1$.
%This time, we will use different strategies for differential privacy and zCDP.

We apply Taylor's theorem around $\hat{\theta}$ to $\nabla L_n(\theta_0)$:
\begin{align*}
\nabla L_n(\theta_0) &\approx \nabla L_n(\hat{\theta})+ H[L_n(\hat{\theta})](\theta_0-\hat{\theta})\\
&= H[L_n(\hat{\theta})](\theta_0-\hat{\theta}).
%&= H(\hat{\theta})[\theta_0-(\tilde{\theta}-\beta^{\prime})].
\end{align*}

As in Section \ref{sec:obj}, we define $G$ to be the Gaussian  $N(\vec{0}, \Sigma)$, and the formulas for Hessian $H[L_n(\hat{\theta})]$ (which is equal to $H[J_n(\hat{\theta})]$) and covariance matrix $\Sigma$ are the same as Equations \ref{eq:hessian} and~\ref{eq:cov}, respectively, from Section \ref{subsec:foundation}. 

Similar to Section~\ref{subsec:foundation}, $\sqrt{n}L_n(\theta_0)$ follows the $N(\vec{0},\Sigma)$ distribution, so we get
\begin{align*}
G &\approx \sqrt{n}H[L_n(\hat{\theta})](\theta_0-\hat{\theta})\\
\theta_0 -\hat{\theta}&\approx H[L_n(\hat{\theta})]^{-1}G/\sqrt{n}\\
\theta_0 - \tilde{\theta}&\approx  H[L_n(\hat{\theta})]^{-1}G/\sqrt{n} - \tilde{\beta},
\end{align*}
where $\tilde{\beta}$ has the same distribution as the $\beta$ in Algorithm~\ref{alg:ermoutput}. Using the same methods as in Section \ref{subsec:hessiancov}, we obtain differentially private estimates for the Hessian using privacy budget $\phi_2$ and covariance matrix using privacy budget $\phi_3$ (and define $\tilde{G}=N(\vec{0},\tilde{\Sigma})$). Plugging those in, we get
\begin{equation}\label{eq:outputsample}
\theta_0 -\tilde{\theta} \approx \tilde{H}[L_n(\tilde{\theta})]^{-1}\tilde{G}/\sqrt{n} - \tilde{\beta}.
\end{equation}
%where $G\sim N(\vec{0}, \Sigma)$, $H[L_n(\hat{\theta})]=H[J_n(\hat{\theta})]$ as is defined in Equation~\ref{eq:hessian}, and $\Sigma$ is defined in Equation~\ref{eq:cov}.
Note that this equation says that the difference between $\theta_0$ and the privacy preserving estimate is approximately the same as the distribution on the right hand side, which only depends on privacy preserving quantities (and not the original data).

%To protect privacy (either differential privacy or zCDP) for $\hat{\theta}$, we will estimate $\hat{\theta}$ with $\tilde{\theta} - \tilde{\beta}$ in Equation~\ref{eq:outputsample}, where $\tilde{\beta}$ follows the same distribution as $\beta$.
%In Equation~\ref{eq:outputsample}, the computations of the Hessian and the covariance matrix also reveal information about the the dataset $\mathcal{D}$. To protect privacy (either differential privacy or zCDP) for them, we estimate $\hat{\theta}$ with $\tilde{\theta}$ in them and then perturb them through Algorithm~\ref{alg:privspdmat}, as is discussed in Section~\ref{subsec:hessiancov}.

For differentially private confidence intervals, as before, we sample many times from the distribution of $\tilde{G}$ and $\tilde{B}$ and use the right hand side of Equation ~\ref{eq:outputsample} to obtain approximate samples $z_1,\dots, z_m$ from the distribution of $\theta_0-\widetilde{\theta}$. For each $j$, we find an interval $(a_j, b_j)$ that contains $(1-\alpha)$ of the $z_i[j]$. Since $\tilde{\theta}$ is a privacy preserving estimate, our privacy preserving confidence interval for $\theta_0[j]$ is $(\tilde{\theta} + a_j, \tilde{\theta} + b_j)$.
%To generate the differentially private confidence intervals for $\theta_0$, we sample from the differentially private counterpart of Equation~\ref{eq:outputsample}.

On the other hand, if we are computing zCDP confidence intervals, the algorithm is much more efficient. In this case both $\tilde{\beta}$ and $\tilde{G}$ are multivariate Gaussians and so their sum is the multivariate Gaussian %, the random variable $\tilde{\theta}-\tilde{\beta}+\tilde{H}[L_n(\tilde{\theta})]^{-1}G/\sqrt{n}$ (i.e., private counterpart of Equation~\ref{eq:outputsample}) is also multivariate normally distributed. We denote it as 
$N(\vec{0},U)$ where
\[U=\frac{1}{2\phi(nc^2)}\mathbf{I}_d+\frac{1}{n}\tilde{H}[L_n(\tilde{\theta})]^{-1}\tilde{\Sigma} \tilde{H}[L_n(\tilde{\theta})]^{-1},\]
and $\phi$ is the privacy budget used in Algorithm~\ref{alg:ermoutput} to perturb $\hat{\theta}$. %Thus, in this case there is no need for Monte Carlo sampling.
%, $\tilde{H}[L_n(\tilde{\theta})]$ and $\tilde{\Sigma}$ are the zCDP counterparts for $H[L_n(\tilde{\theta})]$ and $\Sigma$, respectively. 
Therefore we could compute the confidence intervals for $\theta_0$ directly instead of doing Monte Carlo sampling. For each $j$, we directly compute the confidence interval for $\theta_0[j]$ as
\[\left[\tilde{\theta}[j]-z_{\alpha/2}\sqrt{U_{jj}} \ , \ \tilde{\theta}[j]+z_{\alpha/2}\sqrt{U_{jj}}\right],\]
where $z_{\alpha/2}$ is the $(1-\alpha/2)$-quantile of the standard normal distribution. The complete algorithm is shown in Algorithm~\ref{alg:cioutput}. Note that once we have privacy preserving estimates of $\tilde{\theta}$, $\tilde{H}$, and $\tilde{\Sigma}$ using privacy budgets $\phi_1,\phi_2,\phi_3$, respectively, everything else is post-processing and thus does not affect the privacy cost.  

%\subsection{Confidence Intervals Generation}\label{subsec:generationoutput}
%Based on the above discussions, we use Algorithm~\ref{alg:cioutput} to get the private confidence intervals for $\theta_0$ in the output perturbation based ERM.

\setlength{\algoheightrule}{0pt}
\setlength{\algotitleheightrule}{0pt}
\begin{algorithm}[h]
 \LinesNumbered
 \DontPrintSemicolon
 %\SetAlgoLined
 \SetKwInOut{Input}{input}
 \SetKwInOut{Output}{output}
 \begin{lrbox}{\mybox}
 \begin{minipage}{\hsize}
 \caption{Private $(1-\alpha)$-Confidence Intervals for $\theta_0$ trained with Output Perturbation} \label{alg:cioutput}
 \Input{Data $\mathcal{D}=\{(\vec{x}_i,y_i)\}_{i=1}^n$, privacy budgets $\phi_1$, $\phi_2$ and $\phi_3$, regularization coefficient $c$, the number of samples $m$, confidence level $\alpha$.}
 $\tilde{\theta}\leftarrow$ ERMOutput($\mathcal{D}$, $\phi_1$, c) \tcp{Calling Algorithm~\ref{alg:ermoutput}}
 $H[L_n(\tilde{\theta})] \leftarrow \frac{1}{n}\sum_{i=1}^n H[f(\vec{x}_i,y_i,\tilde{\theta})]+2c\mathbf{I}$\;
 $\tilde{H}[L_n(\tilde{\theta})] \leftarrow$ PrivSPDMat($H[L_n(\tilde{\theta})]$, $Sens(H[L_n(\tilde{\theta})])$, $\phi_2$, $c$) \tcp{calling Algorithm~\ref{alg:privspdmat}}
 $\Sigma \leftarrow \frac{1}{n}\sum_{i=1}^n \nabla f(\vec{x}_i,y_i,\tilde{\theta})[\nabla f(\vec{x}_i,y_i,\tilde{\theta})]^T-4c^2\tilde{\theta}\tilde{\theta}^T$\;
  $\tilde{\Sigma}\leftarrow$ PrivSPDMat($\Sigma$, $Sens(\Sigma)$, $\phi_3$, $c$)\;
  \uIf{Requiring differential privacy}{
  Generate $m$ i.i.d. samples $G_i$ $(i=1,\dots,m)$ from $N(\vec{0},\tilde{\Sigma})$\;
  Generate $m$ i.i.d samples $\beta_i$ $(i=1,\dots,m)$ with density from Equation~\ref{eq:noisedist} with $\gamma=nc\phi_1$\;
  \For{$i\gets 1$ \KwTo $m$}{
  $\theta^{(i)}\leftarrow\tilde{\theta}-\beta_i+\tilde{H}[L_n(\tilde{\theta})]^{-1}G_i/\sqrt{n}$\;
  }
  \For{$j\gets 1$ \KwTo $d$}{
    $(\theta_L[j], \theta_R[j])\leftarrow (1-\alpha)$-confidence interval for $\theta^{(1)}[j],\dots,\theta^{(m)}[j]$\;
  }
  }
  \Else{
 \tcp{for zCDP}
  $z_{\alpha/2}\leftarrow (1-\alpha/2)$-quantile of standard normal\;
  $U=\frac{1}{2\phi_1(nc)^2}\mathbf{I}_d+\frac{1}{n}\tilde{H}[L_n(\tilde{\theta})]^{-1}\tilde{\Sigma} \tilde{H}[L_n(\tilde{\theta})]^{-1}$\;
  \For{$j\gets 1$ \KwTo $d$}{
  $\theta_L[j]\leftarrow\tilde{\theta}[j]-z_{\alpha/2}\sqrt{U_{jj}}$\;
    $\theta_R[j]\leftarrow  \tilde{\theta}[j]+z_{\alpha/2}\sqrt{U_{jj}}$\;
  }  
  }
  Output $\theta_L$, $\theta_R$\;
 \end{minipage}%
 \end{lrbox}
 \hspace*{-10pt}\framebox[\columnwidth]{\hspace*{15pt}\usebox\mybox\par}
\end{algorithm}

\begin{restatable}{theorem}{thmcioutput}
\label{thm:cioutput}
%If the loss function $f(\cdot)$ is convex and differentiable with $|f^{\prime}(\cdot)|\leq 1$, then 
Under the same conditions as Theorem \ref{thm:ermoutput},
Algorithm~\ref{alg:cioutput} satisfies $(\phi_1+\phi_2+\phi_3)$-differential privacy and $(\phi_1+\phi_2+\phi_3)$-zCDP.
\end{restatable}
The proof of Theorem~\ref{thm:cioutput} is in \ConfOrTech{the online Appendix~\ref{app:cioutput}}{Appendix~\ref{app:cioutput}}.

\section{Applications to Logistic Regression and SVM}\label{sec:applications}
We now apply our confidence interval algorithms to logistic regression and support vector machines.
Both models can be learned by objective and output perturbation \cite{chaudhuri2011differentially}.
In order to apply our confidence interval algorithms, we need to compute the $L_2$ sensitivity of the Hessian and covariance matrices, as those quantities are needed to calibrate the amount of perturbation of those matrices that we need to protect privacy (Algorithm \ref{alg:privspdmat}).

%We work with the applications of logistic regression and SVM in the section.
%In order to run Algorithm~\ref{alg:ciobj} or Algorithm~\ref{alg:cioutput} to construct the private confidence intervals, we need to compute the Hessian and the covariance matrix for each application, and show certain conditions are met.
 For logistic regression, the gradient and Hessian are well known:
\begin{align*}
\nabla f(y\theta^T\vec{x})&=
%\frac{-y\exp{(-y\theta^T\vec{x})}}{1+\exp{(-y\theta^T\vec{x})}}\vec{x}=
-yS(-y\theta^T\vec{x})\vec{x},\\
H[f(y\theta^T\vec{x})]&
%\frac{y^2\exp{(-y\theta^T\vec{x})}}{[1+\exp{(-y\theta^T\vec{x})}]^2}\vec{x}\vec{x}^T\\
%=&\frac{\exp{(-y\theta^T\vec{x})}}{[1+\exp{(-y\theta^T\vec{x})}]^2}\vec{x}\vec{x}^T
=S(-y\theta^T\vec{x})S(y\theta^T\vec{x})\vec{x}\vec{x}^T,
\end{align*}
where $S$ is the sigmoid function.
It is also well known that the loss function is convex and doubly differentiable with $|f^{\prime}(z)|\leq 1$ and $|f^{\prime\prime}(z)|\leq 1/4$.
%:
%\[|f^{\prime}(z)|=\frac{\exp(-z)}{1+\exp(-z)}=\frac{1}{1+\exp(z)}\leq 1,\]
%\[|f^{\prime\prime}(z)|=\frac{\exp(z)}{[1+\exp(z)]^2}=\frac{1}{2+\exp(z)+\exp(-z)}\leq 1/4.\]

For SVM, it is well-known that the piecewise gradient and Hessian for the Huber loss $f_{Huber}(y\theta^T\vec{x})$ are:
\begin{equation*}
\nabla f_{Huber}(y\theta^T\vec{x})=
\begin{cases}
\vec{0} & \text{if $y\theta^T\vec{x}>1+h$}\\
\frac{y}{2h}(y\theta^T\vec{x}-1-h)\vec{x} & \text{if $|1-y\theta^T\vec{x}|\leq h$}\\
-y\vec{x} & \text{if $y\theta^T\vec{x}<1-h$,}
\end{cases}
\end{equation*}
and
\begin{equation*}
H[f_{Huber}(y\theta^T\vec{x})]=
\begin{cases}
\frac{y^2}{2h}\vec{x}\vec{x}^T & \text{if $|1-y\theta^T\vec{x}|\leq h$}\\
0_{d\times d} & \text{otherwise.}
\end{cases}
\end{equation*}

Huber loss is convex and differentiable, and piecewise doubly-differentiable, with   $|f_{Huber}^{\prime}(z)|\leq 1$ and $|f_{Huber}^{\prime\prime}(\cdot)|\leq \frac{1}{2h}$ \cite{chaudhuri2011differentially}. Even though the second derivative does not exist at a few isolated points, Chaudhuri et al. \cite{chaudhuri2011differentially} proved that objective and output perturbation algorithms for SVM still preserve privacy.
%Corollary 13 from~\cite{chaudhuri2011differentially} proves that using Huber loss and the regularization term $\frac{1}{2}\|\theta\|_2^2$ protects differential privacy for ERM with objective perturbation. It is easy to see $|f_{Huber}^{\prime}(z)|\leq 1$:
%\[|f_{Huber}^{\prime}(z)|=
%\begin{cases}
%0 & \text{if $z>1+h$}\\
%\frac{1}{2h}|1+h-z|\leq 1 & \text{if $|1-z|\leq h$}\\
%1 & \text{if $z<1-h$.}
%\end{cases}\]

We now derive the $L_2$ sensitivity for the Hessian and the covariance matrix for logistic regression and SVM.
\begin{restatable}{lemma}{senscovlr}
\label{lem:senscovlr}
The $L_2$-sensitivity of the covariance matrix $\Sigma$ (defined in Equation~\ref{eq:cov}) for logistic regression is at most $2S(\|\theta_0\|_2)^2/n$.
\end{restatable}
The proof of Lemma~\ref{lem:senscovlr} is in \ConfOrTech{the online Appendix~\ref{app:senscovlr}}{Appendix~\ref{app:senscovlr}}.

\begin{restatable}{lemma}{senshessianlr}
\label{lem:senshessianlr}
The $L_2$-sensitivity of the Hessian $H[J_n(\tilde{\theta})]$ (defined in Equation~\ref{eq:hessian}) for logistic regression is at most $1/(2n)$.
\end{restatable}
The proof of Lemma~\ref{lem:senshessianlr} is in \ConfOrTech{the online Appendix~\ref{app:senshessianlr}}{Appendix~\ref{app:senshessianlr}}.

\begin{restatable}{lemma}{senscovsvm}
\label{lem:senscovsvm}
The $L_2$-sensitivity of the covariance matrix $\Sigma$ (defined in Equation~\ref{eq:cov}) for SVM is at most $2/n$.
\end{restatable}
The proof of Lemma~\ref{lem:senscovsvm} is in \ConfOrTech{the online Appendix~\ref{app:senscovsvm}}{Appendix~\ref{app:senscovsvm}}.

\begin{restatable}{lemma}{senshessiansvm}
\label{lem:senshessiansvm}
The $L_2$-sensitivity of the Hessian $H[J_n(\tilde{\theta})]$ (defined in Equation~\ref{eq:hessian}) for SVM is at most $1/(nh)$.
\end{restatable}
The proof of Lemma~\ref{lem:senshessiansvm} is in \ConfOrTech{the online Appendix~\ref{app:senshessiansvm}}{Appendix~\ref{app:senshessiansvm}}.

\section{Experiments}\label{sec:experiments}
We run experiments on several real datasets: Adult and KDDCUP99 data sets from \cite{Lichman:2013}, the Banking data set \cite{moro2014data}, the IPUMS-US \cite{ipums:us} dataset and the IPUMS-BR \cite{ipums:br} dataset. Adult \cite{Lichman:2013} is a dataset extracted from the 1994 Census database and contains 30,162 records on demographic information. A common task based on it is predicting whether annual income exceeds $\$50K$. KDDCUP99 \cite{Lichman:2013} is the dataset used for the Third International Knowledge Discovery and Data Mining Tools Competition which contains 4,898,431 records. It contains network traffic data simulated in a military network environment and the goal is to distinguish network attacks. Banking \cite{moro2014data} contains 45,211 records on the direct marketing phone calls of a Portuguese banking institution, and is used for predicting whether the client will subscribe a term deposit. US \cite{ipums:us} and BR \cite{ipums:br} are from IPUMS that provides census and survey data from around the world integrated across time and space. Users can freely choose the data samples and the variables to be used to create data extracts. In the paper, we use the versions from \cite{zhang2013privgene} where US has 39,928 records and BR has 38,000 records. The targets for both of them are predicting whether personal income exceeds some thresholds.

All the datasets contain both numerical and categorical attributes. As was done in \cite{chaudhuri2011differentially,zhang2013privgene}, following common practice for regression problems, we binarize each categorical attribute so that an attribute with cardinality $k$ becomes $k$ binary attributes. We then standardize each attribute so its maximum attribute value becomes 1 and then normalize each record to ensure its $L_2$ norm is upper bounded by 1. As for the target column, it is mapped to either -1 or 1. 
\techreport{After pre-processing, the dimensionality of each dataset is given in Table~\ref{tab:data}.}
To experiment with the sample size and dimensionality, we may extract sub-datasets from those datasets by first randomly permuting the dataset and then taking the first $n_1$ samples and/or the first $d_1$ features from it.
We also add a column of ones as the constant feature to each final dataset, and normalize each record again to bound the $L_2$ norm by 1. But the dimensionality $d$ reported with the experimental results is the one before adding the constant feature.

%%%%%%%% begin techreport %%%%%%%%%
\techreport{
\begin{table}
\centering
\caption{Real Datasets Summary}
\label{tab:data}
\begin{tabular}{|c|cc|}
\hline
Dataset&$n$&$d$\\
\hline
Adult&30,162&37\\
KDDCUP99&4,898,431&90\\
Banking&45,211&34\\
IPUMS-US&39,928&57\\
IPUMS-BR&38,000&53\\
\hline
\end{tabular}
\end{table}
}
%%%%%%%%%%% end techreport %%%%%%%%%

\subsection{Measures}\label{subsec:measures}
The quality of confidence intervals is evaluated using two complementary measures, coverage percentage and length. Coverage percentage is the fraction of times they cover the true parameters, so a putative 95\% confidence interval should cover the true parameter at least 95\% of the time. However, an infinitely long confidence interval can also cover the true parameter at least 95\% of the time, so the goal is to evaluate how short the confidence intervals are. We compare the length to a lower bound we call the \emph{variability interval} -- no confidence interval can be shorter than the variability interval. We describe how we measure coverage percentage and compute the variability intervals next.
%We use the following measures to test the validity and performance of the private confidence intervals computed from our methods.

\vspace{0.1cm}
\noindent\textbf{Coverage percentage}. For each dataset $\mathcal{D}$, we treat its empirical distribution as the true distribution and the non-private model parameters learned on the data as $\theta_0$. To simulate the effects of sampling, we create multiple ``sampled'' datasets $\mathcal{D}_1,\dots, \mathcal{D}_k$ by sampling with replacement from the original dataset $\mathcal{D}$. Each such dataset $\mathcal{D}_i$ is called a \emph{bootstrap replicate} and there are $k=1,000$ of them. To estimate coverage, for each $\mathcal{D}_i$ we use our algorithm to compute the privacy-preserving confidence interval $\left(\theta_L^{(i)}[j],\theta_R^{(i)}[j]\right)$ for each coordinate $j$ of $\theta_0$. The coverage percentage for a parameter $\theta_0[j]$ is then the fraction of the privacy-preserving confidence intervals that contain $\theta_0[j]$. The overall coverage is then the average coverage over all parameters:
\[\frac{\sum_{i=1}^k\sum_{j=1}^{d^{\prime}}\mathbf{1}_{\theta_L^{(i)}[j]\leq \theta_0[j]\leq \theta_R^{(i)}[j]}}{kd^{\prime}},\]
where $\mathbf{1}$ is the indicator function and $d^{\prime}$ is the dimensionality of $\theta_0$.

\vspace{0.1cm}
\noindent\textbf{Variability Intervals VI}. The variability interval directly measures the actual variation in parameter estimate due to sampling noise and due to the algorithm that estimates the parameters (e.g., output perturbation or objective perturbation). This is possible to obtain in controlled experiments. On the other hand, confidence intervals are an estimate (not a direct measurement) of this variability. We obtain variability intervals as follows:

For each dataset $\mathcal{D}$, we treat its empirical distribution as the true distribution and the non-private model parameters learned on the data as $\theta_0$. To simulate the effects of sampling, we create multiple ``sampled'' datasets $\mathcal{D}_1,\dots, \mathcal{D}_m$ by sampling with replacement from the original dataset $\mathcal{D}$. Each such dataset $\mathcal{D}_i$ is called a \emph{bootstrap replicate} and there are $m=10,000$ of them. This simulates variability due to sampling. On each $\mathcal{D}_i$ we run the privacy-preserving ERM algorithm (either output or objective perturbation) to get the estimate  $\tilde{\theta}^{(i)}$. The variability in these $\tilde{\theta}^{(i)}$ is thus solely due to sampling and privacy noise used to create the parameter estimates (with privacy budget $\phi_1$) -- in other words, it is not affected by the $\phi_2$ and $\phi_3$ that are used in our confidence interval algorithms.

For $1\leq j\leq d$, let $\tilde{\theta}_L[j]$ and $\tilde{\theta}_R[j]$ be the $\alpha/2$-quantile and $(1-\alpha/2)$-quantile of $\tilde{\theta}^{(1)}[j],\cdots$, $\tilde{\theta}^{(m)}[j]$, respectively. Then the $1-\alpha$ variability interval for coordinate $j$ is $(\tilde{\theta}_L[j],\tilde{\theta}_R[j])$. 

Clearly, any $1-\alpha$ confidence interval must therefore be at least as long as the $1-\alpha$ variability interval and so the quality of a confidence interval is measured as how long it is compared to the variability interval. Thus we plot average length of confidence intervals vs. average length of variability intervals.

Throughout our experiments, we use $\alpha=0.05$, $k=1,000$, $m=10,000$ (recall $m$ is the number of post-processing samples used in Algorithms \ref{alg:ciobj} and \ref{alg:cioutput} to estimate confidence intervals). For simplicity, in the figures, we use \textbf{DP} for differential privacy, \textbf{zCDP} for zero-concentrated differential privacy, \textbf{CI} for confidence interval, \textbf{VI} for variability interval, \textbf{obj} for ERM with objective perturbation, \textbf{output} for ERM with output perturbation, \textbf{LR} for logistic regression and \textbf{SVM} for support vector machines.

In our experimental results, we will report the privacy parameters $\epsilon$ for differential privacy and $\rho$ for zCDP. To compare differentially private and zCDP algorithms on the same plot, we set $\rho=\frac{\epsilon^2}{2}$. This is the closest possible apples-to-apples comparison, as any algorithm satisfying $\epsilon$-differential privacy satisfies $\rho=\frac{\epsilon^2}{2}$ zCDP \cite{bun2016concentrated}.
%to stand for $\phi$ from Algorithms \ref{alg:ciobj} and~\ref{alg:cioutput}. That is because in order to compare $\epsilon$-differential privacy with $\rho$-zCDP on the same scale, we need to make $\rho=\epsilon^2/2$.% Therefore, in the results, we only report the privacy parameters used for differential privacy, and those for zCDP can be obtained accordingly.

\subsection{Coverage Percentage Experiments}\label{subsec:validity}
We first check the coverage percentage of the private confidence intervals computed from our algorithms for the applications of logistic regression and SVM.
With $\alpha=0.05$, the expected coverage is about 0.95.
\ConfOrTech{In Tables \ref{tab:kddcup-n} through~\ref{tab:us-eps2}}{In Tables \ref{tab:kddcup-n} through~\ref{tab:br-eps3}}, we present the coverage percentage for the private confidence intervals computed from each of our algorithms with varying parameters.
That is, we test for both differential privacy and zCDP and we evaluate how good the intervals are for parameters obtained through objective perturbation and output perturbation. 
We can see that the coverage from our experiments are all close to $0.95$, which means the private confidence intervals generated from our methods are valid.

\begin{table*}[h!]
  \centering
  \caption{Coverage Percentage for KDDCUP99 with $d=10$, $\epsilon_1=0.5$, $\rho_1=0.125$, $\epsilon_2=\epsilon_3=0.25$, $\rho_2=\rho_3=0.03125$, $c=0.001$, $h=1.0$.}
  \label{tab:kddcup-n}
  \begin{tabular}{|c|c|c|c|c|c|c|c|c|c|c|c|}
    \hline
    \multicolumn{3}{|c|}{}\multirow{2}{*}{}&\multicolumn{9}{c|}{$n$}\\ \cline{4-12}
    \multicolumn{3}{|c|}{}&50,000&100,000&150,000&200,000&250,000&300,000&350,000&400,000&450,000\\ \cline{4-12}
    \hline
    \multirow{4}{*}{DP}&\multirow{2}{*}{obj}&LR&0.9680 &0.9794 &0.9838 &0.9880 &0.9921 &0.9920 &0.9934 &0.9943 &0.9951\\
    &&SVM&0.9545 &0.9738 &0.9760 &0.9865 &0.9859 &0.9893 &0.9906 &0.9915 &0.9938\\ \cline{2-12}
    &\multirow{2}{*}{output}&LR&0.9695 &0.9758 &0.9775 &0.9793 &0.9825 &0.9839 &0.9880 &0.9888 &0.9894\\
    &&SVM&0.9710 &0.9717 &0.9771 &0.9780 &0.9795 &0.9805 &0.9824 &0.9850 &0.9867\\
    \hline
    \multirow{4}{*}{zCDP}&\multirow{2}{*}{obj}&LR&0.9659 &0.9765 &0.9822 &0.9870 &0.9921 &0.9914 &0.9938 &0.9948 &0.9962\\
    &&SVM&0.9614 &0.9715 &0.9791 &0.9845 &0.9900 &0.9889 &0.9917 &0.9920 &0.9925\\ \cline{2-12}
    &\multirow{2}{*}{output}&LR&0.9915 &0.9961 &0.9960 &0.9975 &0.9975 &0.9975 &0.9985 &0.9982 &0.9984\\
    &&SVM&0.9915 &0.9936 &0.9936 &0.9948 &0.9955 &0.9949 &0.9953 &0.9953 &0.9948\\
    \hline
  \end{tabular}
\end{table*}

\begin{table*}[h!]
  \centering
  \caption{Coverage Percentage for Adult with $n=30,162$, $\epsilon_1=0.5$, $\rho_1=0.125$, $\epsilon_2=\epsilon_3=0.25$, $\rho_2=\rho_3=0.03125$, $c=0.001$, $h=1.0$.}
  \label{tab:adult-d}
  \begin{tabular}{|c|c|c|c|c|c|c|c|c|c|c|c|c|}
    \hline
    \multicolumn{3}{|c|}{}\multirow{2}{*}{}&\multicolumn{10}{c|}{$d$}\\ \cline{4-13}
    \multicolumn{3}{|c|}{}&1&2&3&4&5&6&7&8&9&10\\ \cline{4-13}
    \hline
    \multirow{4}{*}{DP}&\multirow{2}{*}{obj}&LR&0.9575 &0.9523 &0.9593 &0.9632 &0.9642 &0.9681 &0.9657 &0.9654 &0.9708 &0.9713\\
    &&SVM&0.9515 &0.9490 &0.9630 &0.9634 &0.9628 &0.9667 &0.9655 &0.9627 &0.9690 &0.9724\\ \cline{2-13}
    &\multirow{2}{*}{output}&LR&0.9500 &0.9577 &0.9545 &0.9598 &0.9603 &0.9629 &0.9590 &0.9559 &0.9653 &0.9658\\
    &&SVM&0.9545 &0.9550 &0.9547 &0.9590 &0.9525 &0.9610 &0.9593 &0.9648 &0.9615 &0.9605\\
    \hline
    \multirow{4}{*}{zCDP}&\multirow{2}{*}{obj}&LR&0.9645 &0.9580 &0.9513 &0.9572 &0.9592 &0.9594 &0.9584 &0.9556 &0.9577 &0.9631\\
    &&SVM&0.9580 &0.9480 &0.9473 &0.9600 &0.9462 &0.9603 &0.9594 &0.9547 &0.9542 &0.9576\\ \cline{2-13}
    &\multirow{2}{*}{output}&LR&0.9615 &0.9560 &0.9583 &0.9646 &0.9643 &0.9631 &0.9691 &0.9693 &0.9744 &0.9715\\
    &&SVM&0.9565 &0.9510 &0.9523 &0.9600 &0.9558 &0.9541 &0.9647 &0.9659 &0.9654 &0.9650\\
    \hline
  \end{tabular}
\end{table*}

\begin{table*}[h!]
  \centering
  \caption{Coverage Percentage for Banking with $n=45,211$, $d=10$, $\epsilon_2=\epsilon_3=0.25$, $\rho_2=\rho_3=0.03125$, $c=0.001$, $h=1.0$.}
  \label{tab:banking-eps}
  \begin{tabular}{|c|c|c|c|c|c|c|c|c|c|c|}
    \hline
    \multicolumn{3}{|c|}{}\multirow{2}{*}{}&\multicolumn{8}{c|}{$\epsilon_1\quad$ ($\rho_1=\epsilon_1^2/2$)}\\ \cline{4-11}
    \multicolumn{3}{|c|}{}&0.3&0.4&0.5&0.6&0.7&0.8&0.9&1.0\\ \cline{4-11}
    \hline
    \multirow{4}{*}{DP}&\multirow{2}{*}{obj}&LR&0.9568 &0.9646 &0.9760 &0.9752 &0.9805 &0.9846 &0.9868 &0.9850\\
    &&SVM&0.9743 &0.9798 &0.9835 &0.9871 &0.9898 &0.9921 &0.9923 &0.9932\\ \cline{2-11}
    &\multirow{2}{*}{output}&LR&0.9533 &0.9637 &0.9645 &0.9728 &0.9745 &0.9751 &0.9799 &0.9821\\
    &&SVM&0.9599 &0.9591 &0.9634 &0.9657 &0.9705 &0.9727 &0.9755 &0.9792\\
    \hline
    \multirow{4}{*}{zCDP}&\multirow{2}{*}{obj}&LR&0.9546 &0.9515 &0.9581 &0.9635 &0.9643 &0.9657 &0.9715 &0.9729\\
    &&SVM&0.9550 &0.9525 &0.9601 &0.9610 &0.9577 &0.9654 &0.9656 &0.9695\\ \cline{2-11}
    &\multirow{2}{*}{output}&LR&0.9697 &0.9716 &0.9733 &0.9770 &0.9781 &0.9791 &0.9780 &0.9803\\
    &&SVM&0.9615 &0.9603 &0.9667 &0.9697 &0.9695 &0.9731 &0.9721 &0.9754\\
    \hline
  \end{tabular}
\end{table*}

\begin{table*}[h!]
  \centering
  \caption{Coverage Percentage for US with $n=39,928$, $d=10$, $\epsilon_1=0.5$,
    $\rho_1=0.125$, $\epsilon_3=0.25$, $\rho_3=0.03125$, $c=0.001$, $h=1.0$.}
  \label{tab:us-eps2}
  \begin{tabular}{|c|c|c|c|c|c|c|c|c|c|c|c|c|}
    \hline
    \multicolumn{3}{|c|}{}\multirow{2}{*}{}&\multicolumn{10}{c|}{$\epsilon_2\quad$ ($\rho_2=\epsilon_2^2/2$)}\\ \cline{4-13}
    \multicolumn{3}{|c|}{}&0.1&0.2&0.3&0.4&0.5&0.6&0.7&0.8&0.9&1.0\\ \cline{4-13}
    \hline
    \multirow{4}{*}{DP}&\multirow{2}{*}{obj}&LR&0.9751 &0.9756 &0.9740 &0.9731 &0.9715 &0.9705 &0.9692 &0.9717 &0.9690 &0.9717\\
    &&SVM&0.9720 &0.9732 &0.9725 &0.9678 &0.9739 &0.9674 &0.9643 &0.9653 &0.9662 &0.9649\\ \cline{2-13}
    &\multirow{2}{*}{output}&LR&0.9596 &0.9624 &0.9607 &0.9613 &0.9580 &0.9619 &0.9572 &0.9593 &0.9570 &0.9578\\
    &&SVM&0.9685 &0.9641 &0.9608 &0.9618 &0.9611 &0.9567 &0.9597 &0.9555 &0.9593 &0.9574\\
    \hline
    \multirow{4}{*}{zCDP}&\multirow{2}{*}{obj}&LR&0.9546 &0.9535 &0.9535 &0.9558 &0.9495 &0.9516 &0.9554 &0.9549 &0.9551 &0.9527\\
    &&SVM&0.9547 &0.9543 &0.9497 &0.9479 &0.9505 &0.9461 &0.9487 &0.9505 &0.9484 &0.9502\\ \cline{2-13}
    &\multirow{2}{*}{output}&LR&0.9626 &0.9577 &0.9612 &0.9626 &0.9585 &0.9576 &0.9606 &0.9595 &0.9620 &0.9595\\
    &&SVM&0.9588 &0.9546 &0.9524 &0.9467 &0.9512 &0.9490 &0.9489 &0.9535 &0.9505 &0.9499\\
    \hline
  \end{tabular}
\end{table*}

%%%%%%%%%%% begin techreport %%%%%%%%%%
\techreport{
\begin{table*}
  \centering
  \caption{Coverage Percentage for BR with $n=38,000$, $d=10$, $\epsilon_1=0.5$,
    $\rho_1=0.125$, $\epsilon_2=0.25$, $\rho_2=0.03125$, $c=0.001$, $h=1.0$.}
    \label{tab:br-eps3}
  \begin{tabular}{|c|c|c|c|c|c|c|c|c|c|c|c|c|}
    \hline
    \multicolumn{3}{|c|}{}\multirow{2}{*}{}&\multicolumn{10}{c|}{$\epsilon_3\quad$ ($\rho_3=\epsilon_3^2/2$)}\\ \cline{4-13}
    \multicolumn{3}{|c|}{}&0.1&0.2&0.3&0.4&0.5&0.6&0.7&0.8&0.9&1.0\\ \cline{4-13}
    \hline
    \multirow{4}{*}{DP}&\multirow{2}{*}{obj}&LR&0.9829 &0.9722 &0.9682 &0.9620 &0.9595 &0.9499 &0.9602 &0.9540 &0.9578 &0.9561\\
    &&SVM&0.9815 &0.9679 &0.9618 &0.9629 &0.9635 &0.9603 &0.9592 &0.9567 &0.9617 &0.9545\\ \cline{2-13}
    &\multirow{2}{*}{output}&LR&0.9705 &0.9663 &0.9619 &0.9574 &0.9574 &0.9593 &0.9557 &0.9589 &0.9580 &0.9619\\
    &&SVM&0.9735 &0.9660 &0.9616 &0.9638 &0.9634 &0.9588 &0.9593 &0.9607 &0.9592 &0.9572\\
    \hline
    \multirow{4}{*}{zCDP}&\multirow{2}{*}{obj}&LR&0.9629 &0.9595 &0.9618 &0.9604 &0.9585 &0.9645 &0.9568 &0.9601 &0.9622 &0.9608\\
    &&SVM&0.9554 &0.9555 &0.9577 &0.9590 &0.9575 &0.9561 &0.9575 &0.9586 &0.9555 &0.9551\\ \cline{2-13}
    &\multirow{2}{*}{output}&LR&0.9757 &0.9735 &0.9715 &0.9746 &0.9730 &0.9720 &0.9745 &0.9733 &0.9737 &0.9722\\
    &&SVM&0.9695 &0.9682 &0.9650 &0.9677 &0.9672 &0.9689 &0.9695 &0.9673 &0.9650 &0.9685\\
    \hline
  \end{tabular}
\end{table*}
}
%%%%%%%%%%% end techreport %%%%%%%%%%%%%%

\begin{figure*}[h!]
\centering
\subfloat[KDDCUP99, LR, $d=10$]{
\includegraphics[width=0.33\textwidth]{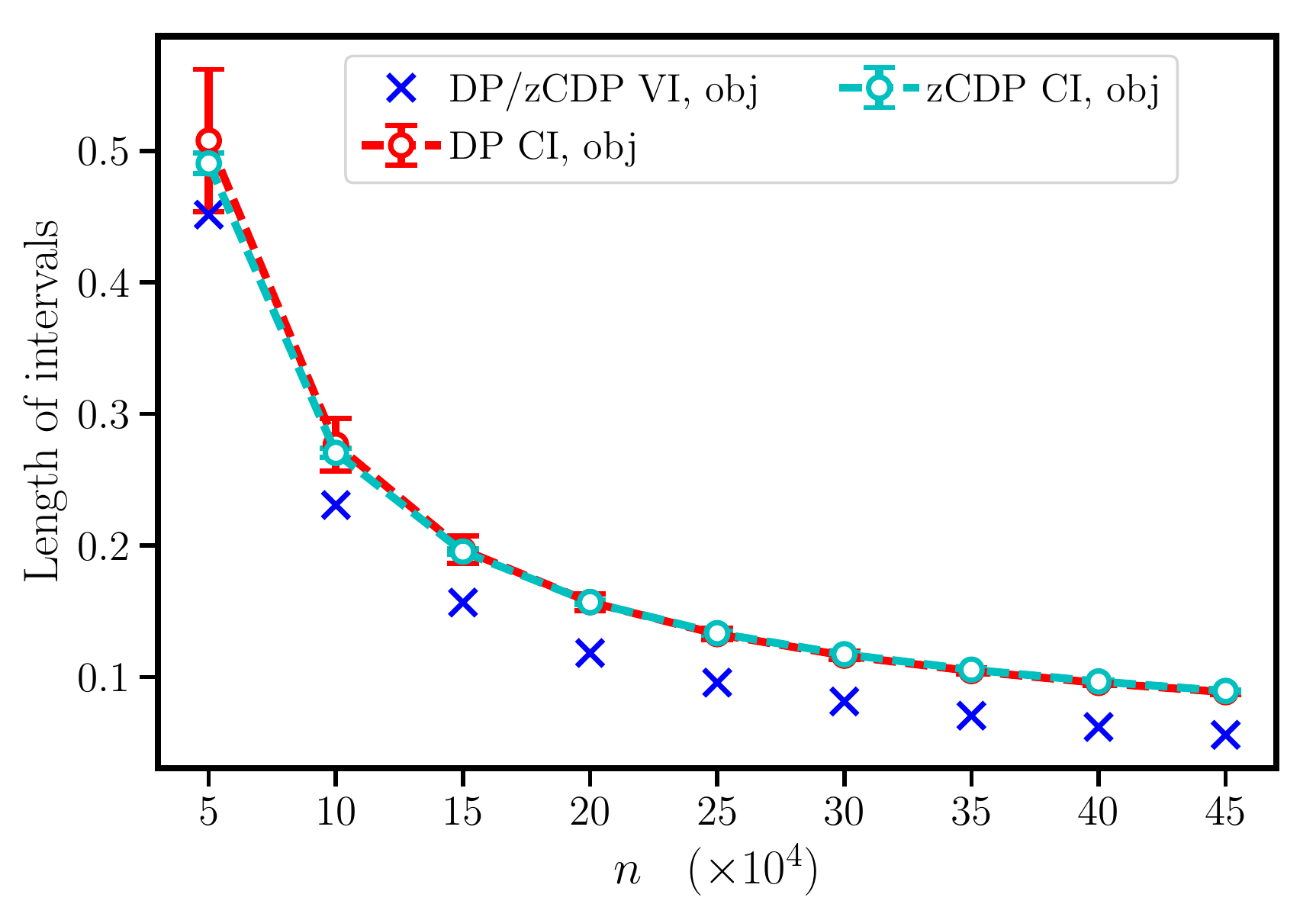}
}%
\hskip -8pt
\subfloat[Adult, LR, $n=30,162$]{
\includegraphics[width=0.33\textwidth]{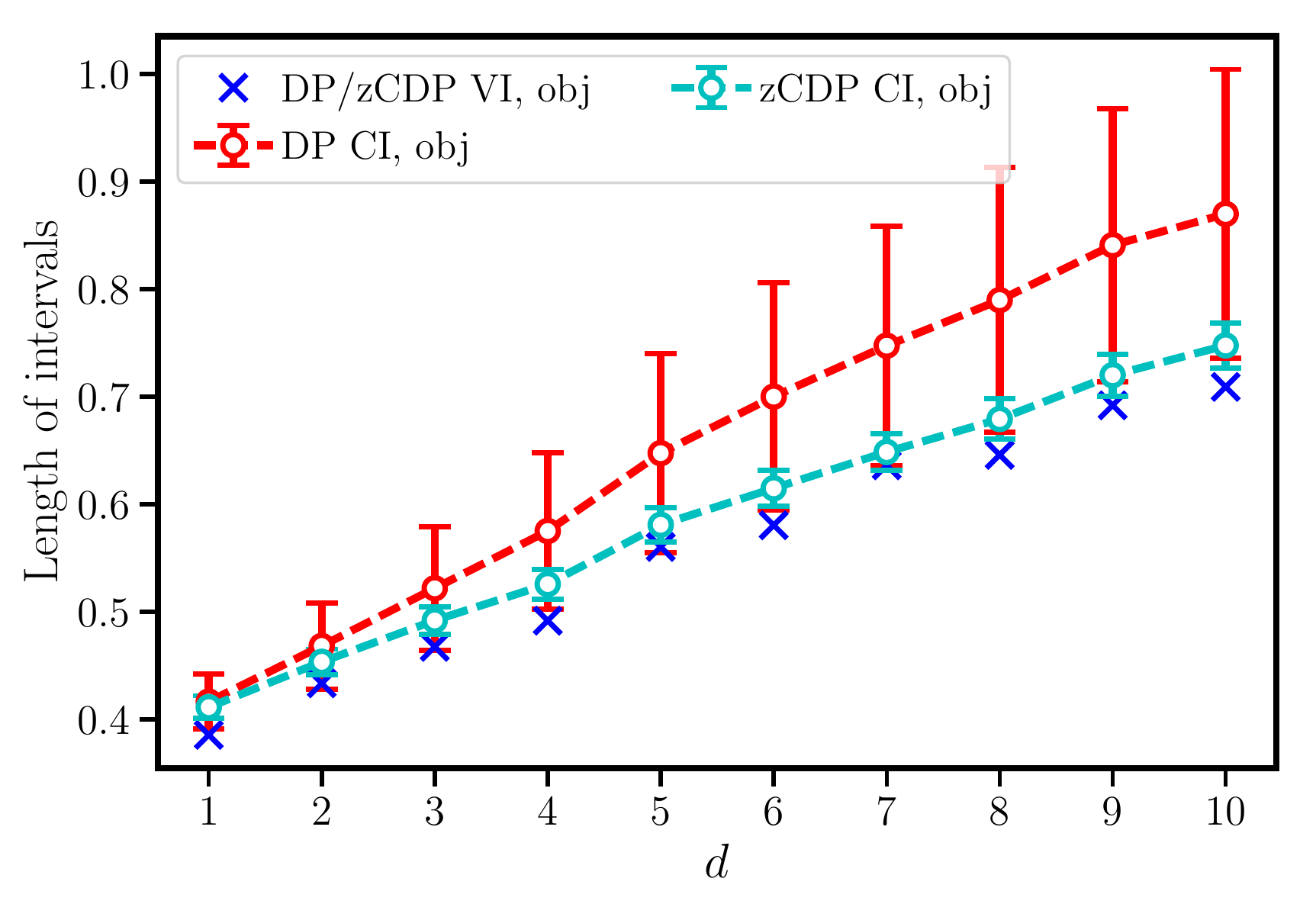}
}%
\hskip -8pt
\subfloat[Adult, SVM, $n=30,162$]{
\includegraphics[width=0.33\textwidth]{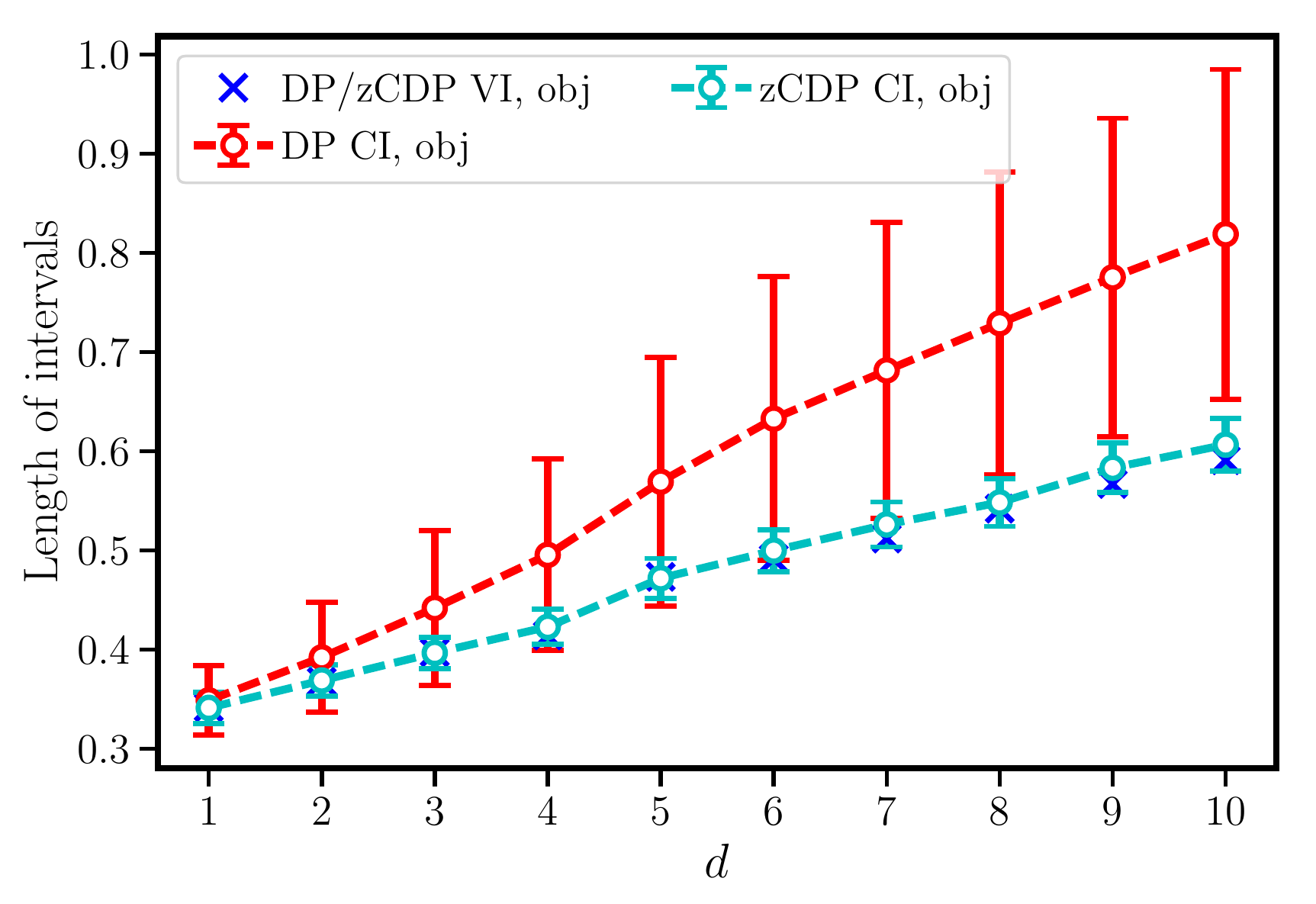}
}%
\hskip -8pt
\subfloat[Banking, LR, $\rho_1=\epsilon_1^2/2$]{
\includegraphics[width=0.33\textwidth]{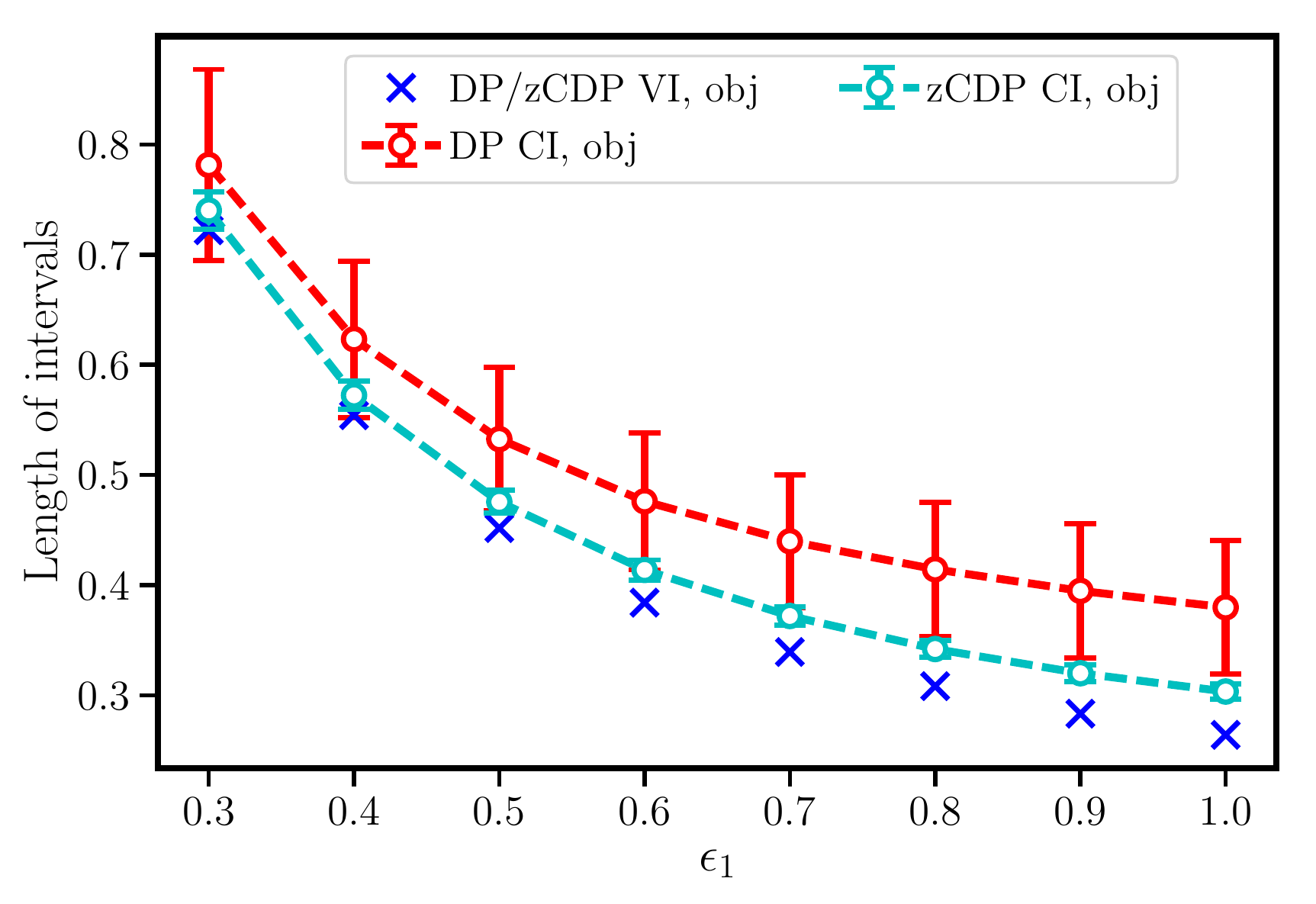}
}%
\hskip -8pt
\subfloat[Banking, SVM, $\rho_1=\epsilon_1^2/2$]{
\includegraphics[width=0.33\textwidth]{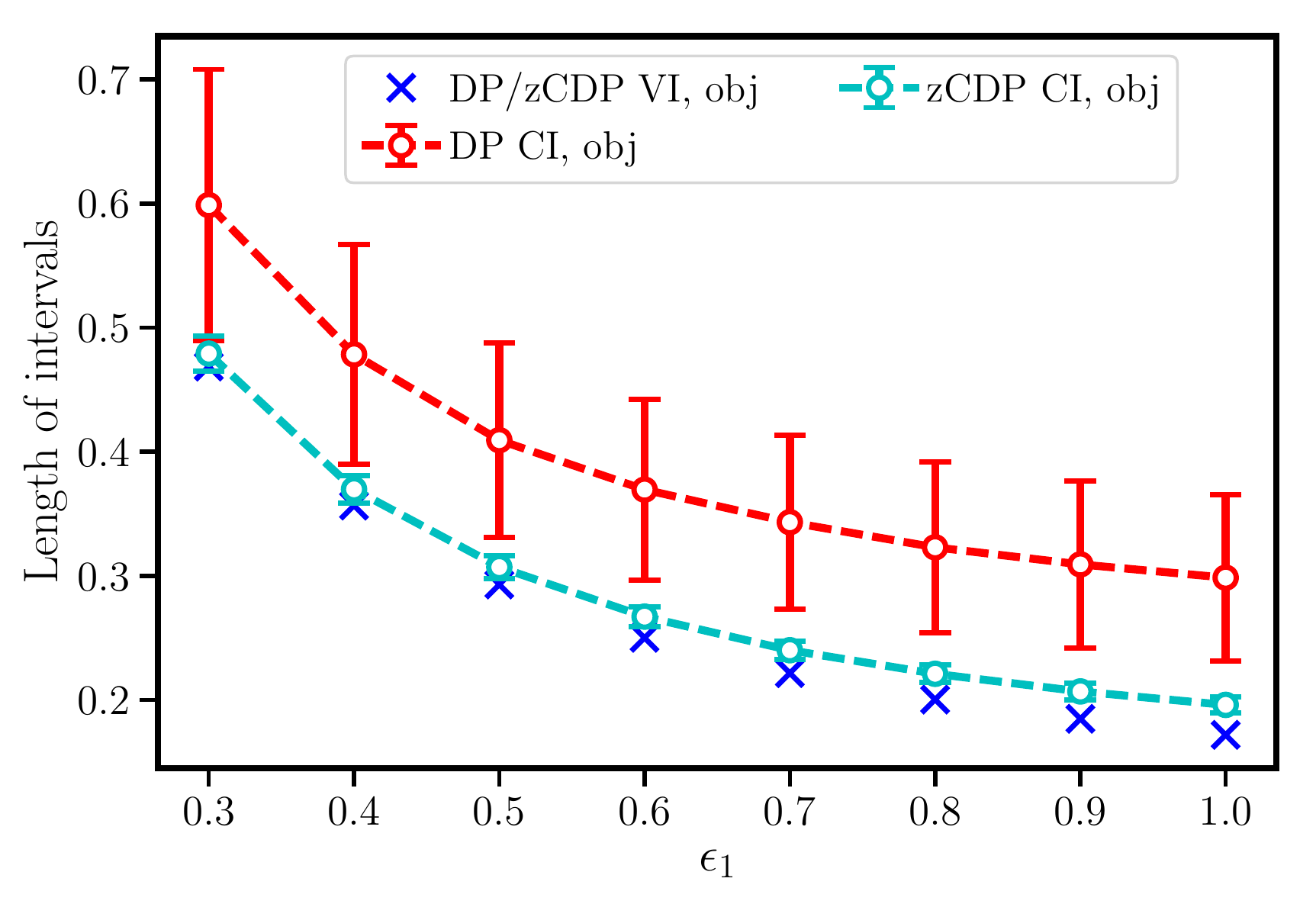}
}%
\hskip -8pt
\subfloat[BR, LR, $\rho_1=\epsilon_1^2/2$]{
\includegraphics[width=0.33\textwidth]{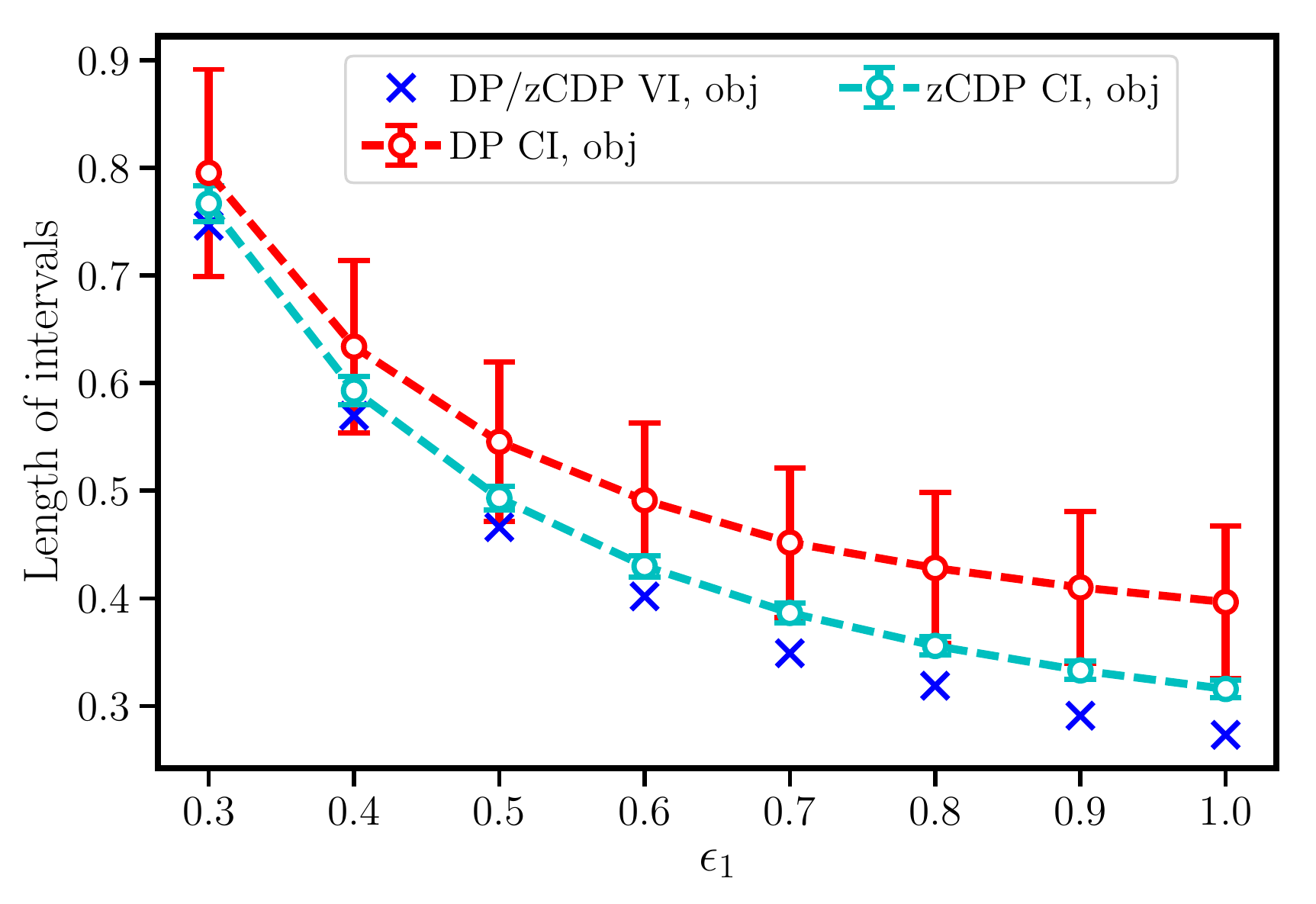}
}%
\hskip -8pt
\subfloat[US, LR, $\rho_2=\epsilon_2^2/2$]{
\includegraphics[width=0.33\textwidth]{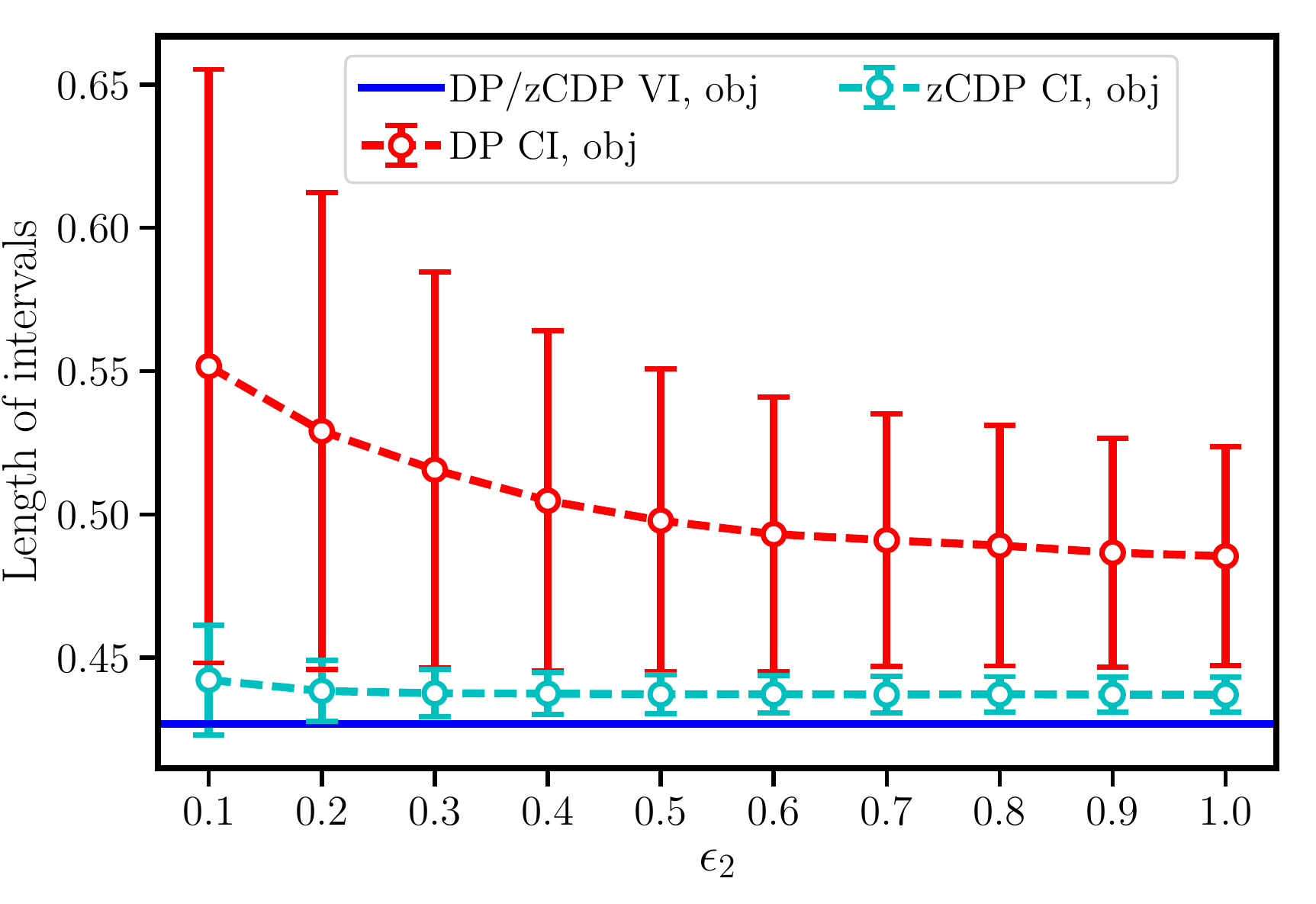}
}%
\hskip -8pt
\subfloat[US, SVM, $\rho_2=\epsilon_2^2/2$]{
\includegraphics[width=0.33\textwidth]{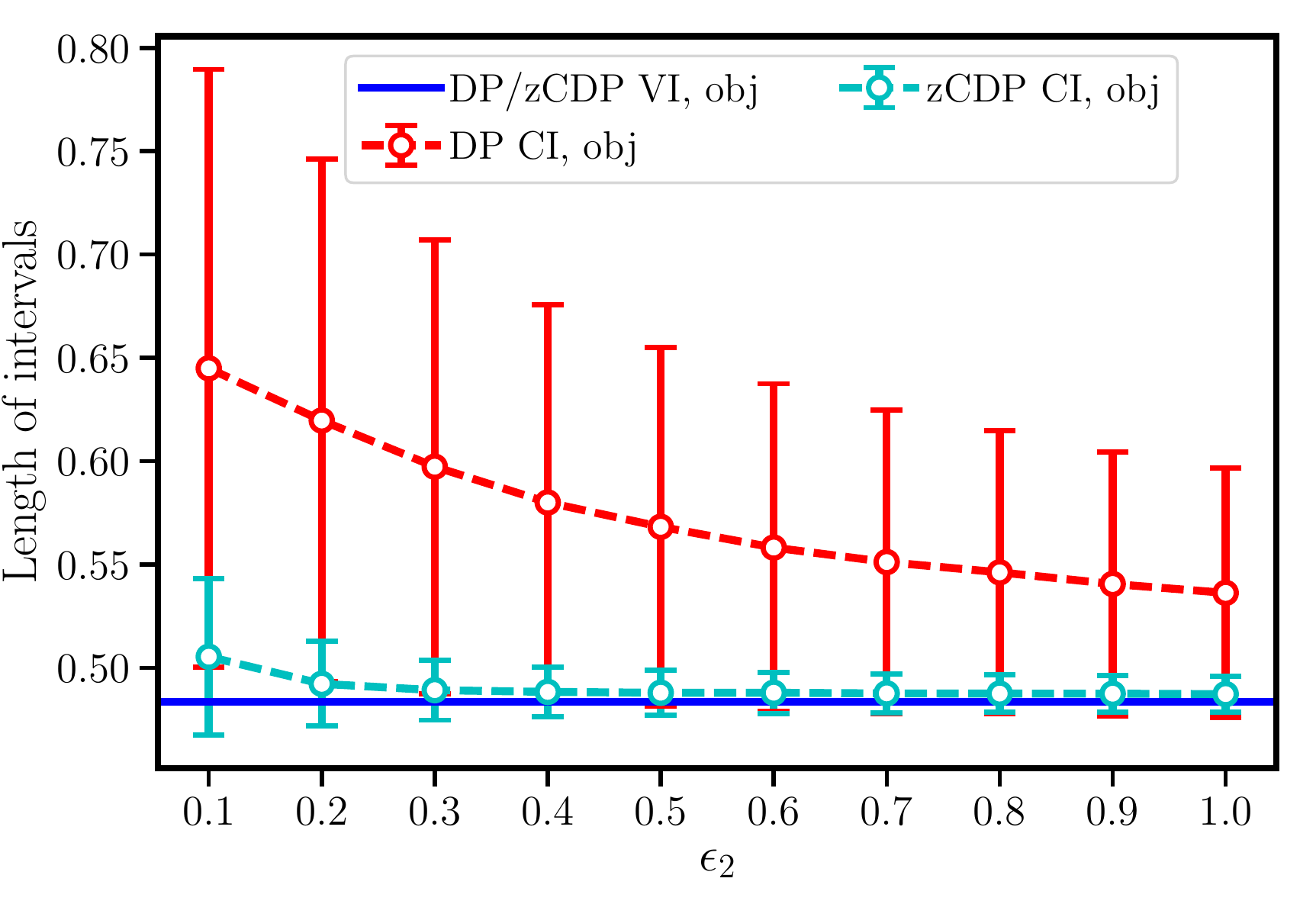}
}%
\hskip -8pt
\subfloat[BR, SVM, $\rho_3=\epsilon_3^2/2$]{
\includegraphics[width=0.33\textwidth]{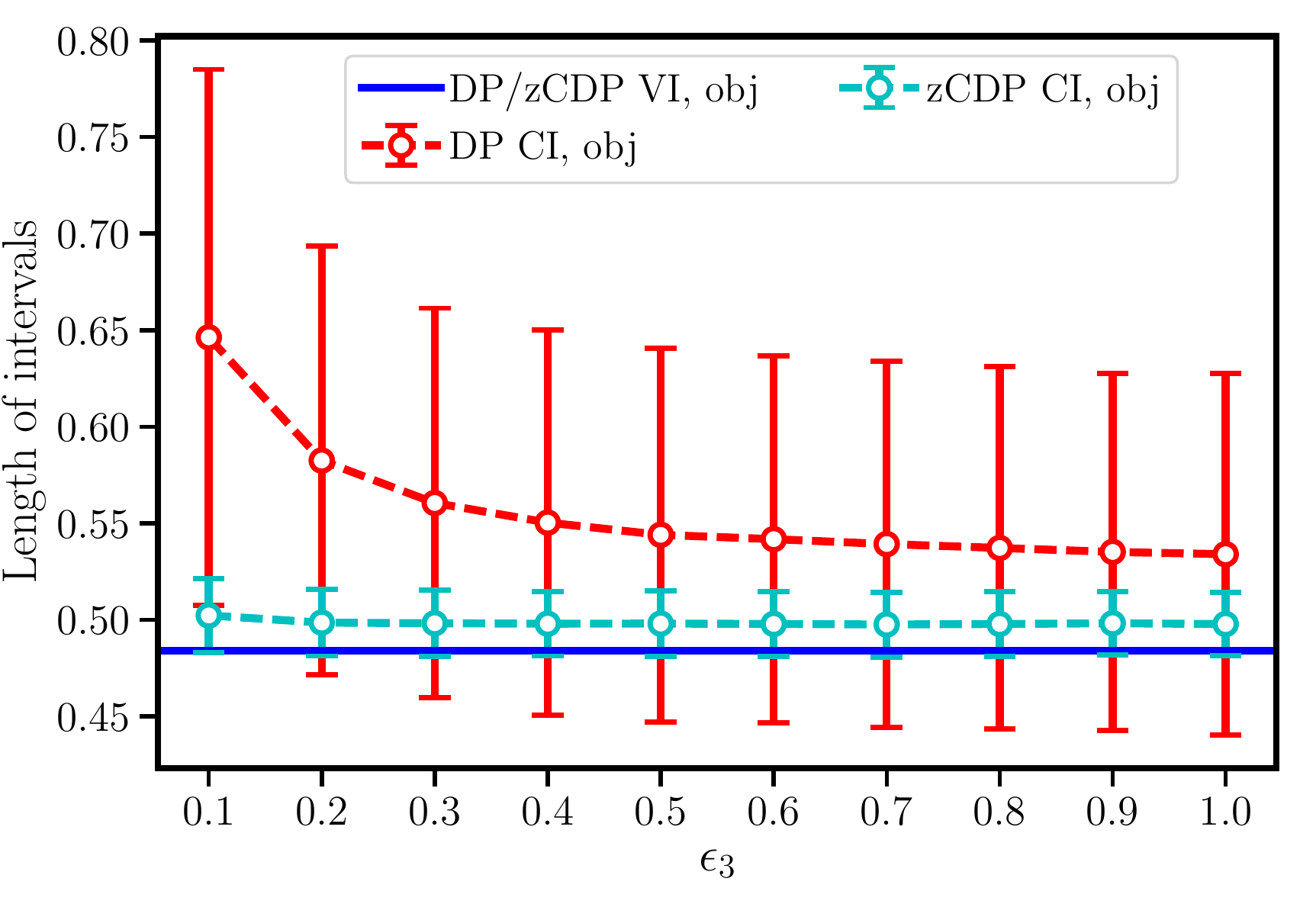}
}%
\vskip -8pt
\caption{Length of intervals for DP vs zCDP with objective perturbation ERM. Error bars correspond to one standard deviation for CI. (a)-(c): $\epsilon_1=0.5$, $\rho_1=0.125$, $\epsilon_2=\epsilon_3=0.25$, $\rho_2=\rho_3=0.03125$; (d)-(f): $n=45,211$ for Banking, $n=38,000$ for BR, $d=10$, $\epsilon_2=\epsilon_3=0.25$, $\rho_2=\rho_3=0.03125$; (g)-(h): $n=39,928$, $d=10$, $\epsilon_1=0.5$, $\rho_1=0.125$, $\epsilon_3=0.25$, $\rho_3=0.03125$; (i): $n=38,000$, $d=10$, $\epsilon_1=0.5$, $\rho_1=0.125$, $\epsilon_2=0.25$, $\rho_2=0.03125$. Common parameters: $c=0.001$, $h=1.0$.}
\label{fig:obj}
\end{figure*}

\subsection{Comparison among the Private Confidence Intervals}\label{subsec:comparison}

We test the performance of the private confidence intervals on the applications of logistic regression and SVM. We will compare the length of our confidence intervals to that of the variability intervals.

We will directly compare different configurations, like intervals using differential privacy vs zCDP, intervals for coefficients obtained through objective perturbation vs. intervals obtained through output perturbation. In all of those cases, we vary various parameters, like the sample size, dimensionality and privacy parameters.
When we are varying the dimensionality $d$, we report the length of the intervals from the first coordinate (even though we estimate all of them). This is so that we are comparing equivalent parameters that have the same meaning. For example it makes no sense to compare the interval for the first parameter when $d=1$ to the average of the first two parameters when $d=2$.

When dimensionality is not being varied, we report the average length of the intervals across all coordinates.

\subsubsection{DP vs zCDP with Objective Perturbation ERM}

In Figure~\ref{fig:obj}, we experiment with the different privacy definitions for parameters obtained with objective perturbation. Figure~\ref{fig:obj}a shows results for varying sample size on the KDDCUP99 dataset. Differential privacy and zCDP performs similarly except when $n=50,000$ where zCDP confidence intervals are a little bit shorter in length. For the rest of the graphs in Figure~\ref{fig:obj} where we vary parameters other than the sample size, zCDP outperforms differential privacy by both average value and variance for length of confidence intervals. This is because  zCDP is a relaxation of differential privacy. Moreover, the zCDP confidence intervals perform very well by showing closeness in length to the variability intervals.

\ConfOrTech{
%%%%%%%%%% conference
\begin{figure*}[h!]
\centering
\subfloat[KDDCUP99, SVM, $d=10$]{
\includegraphics[width=0.33\textwidth]{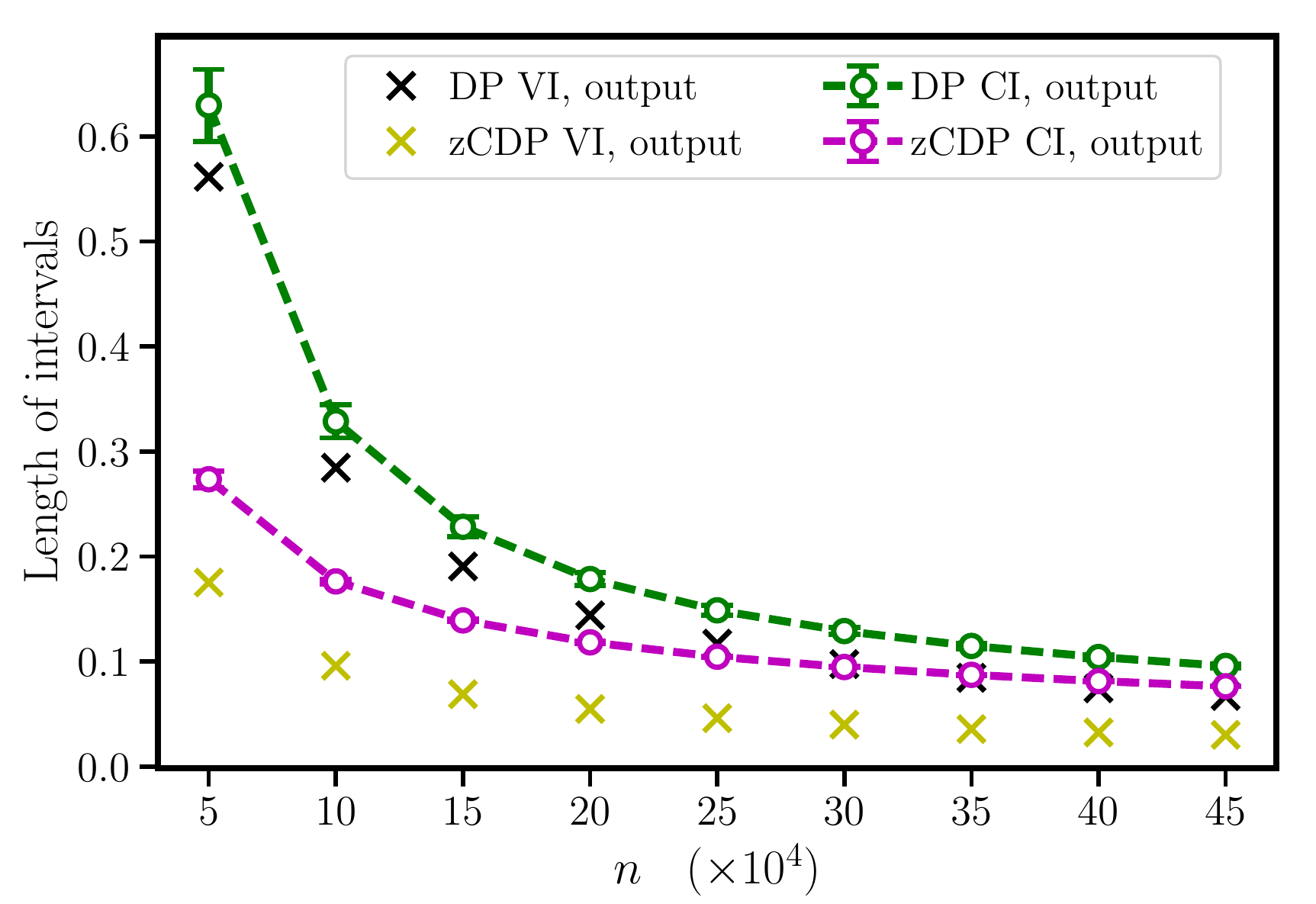}
}%
\hskip -8pt
\subfloat[Banking, LR, $n=45,211$]{
\includegraphics[width=0.33\textwidth]{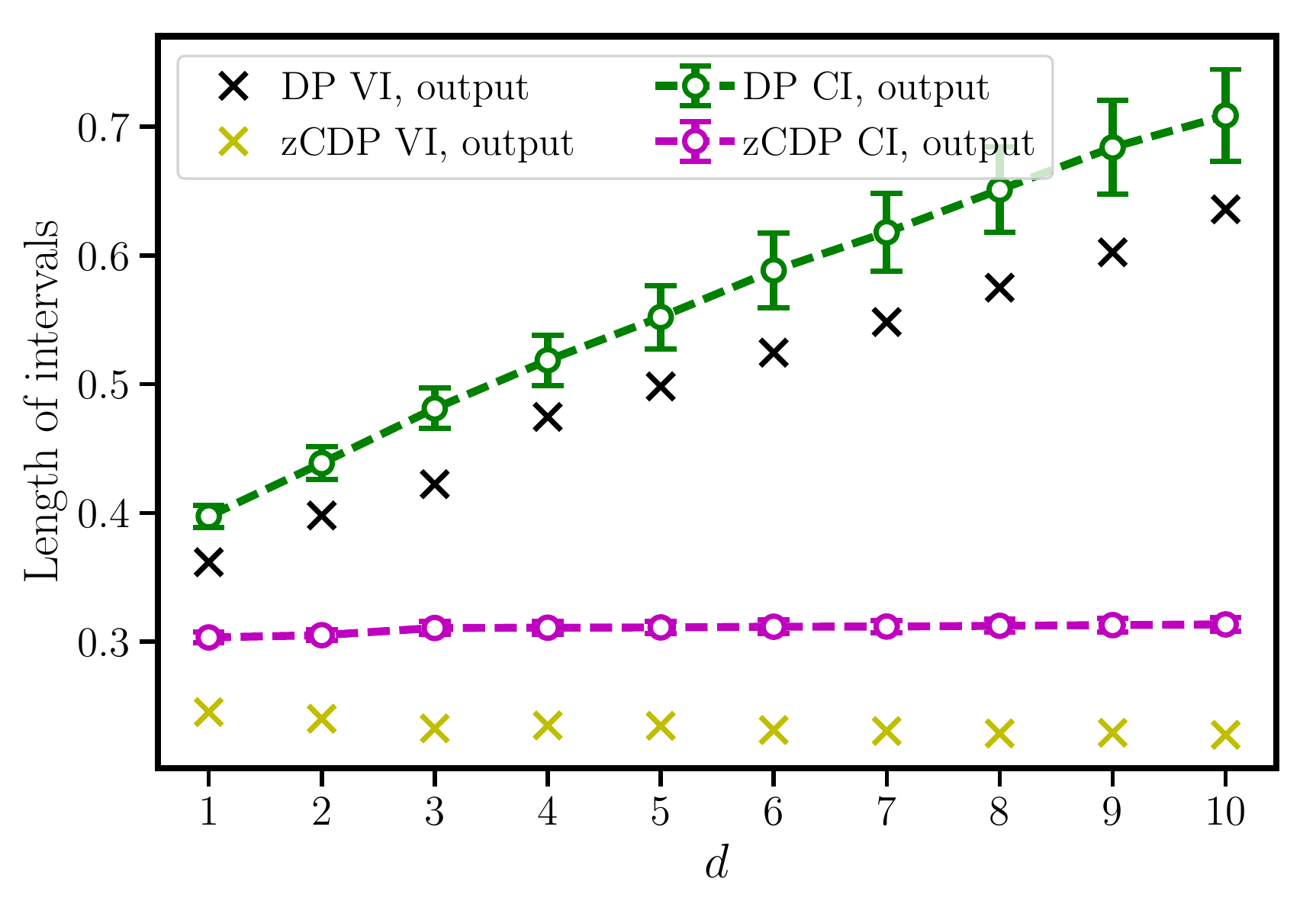}
}%
\hskip -8pt
\subfloat[Adult, LR, $\rho_1=\epsilon_1^2/2$]{
\includegraphics[width=0.33\textwidth]{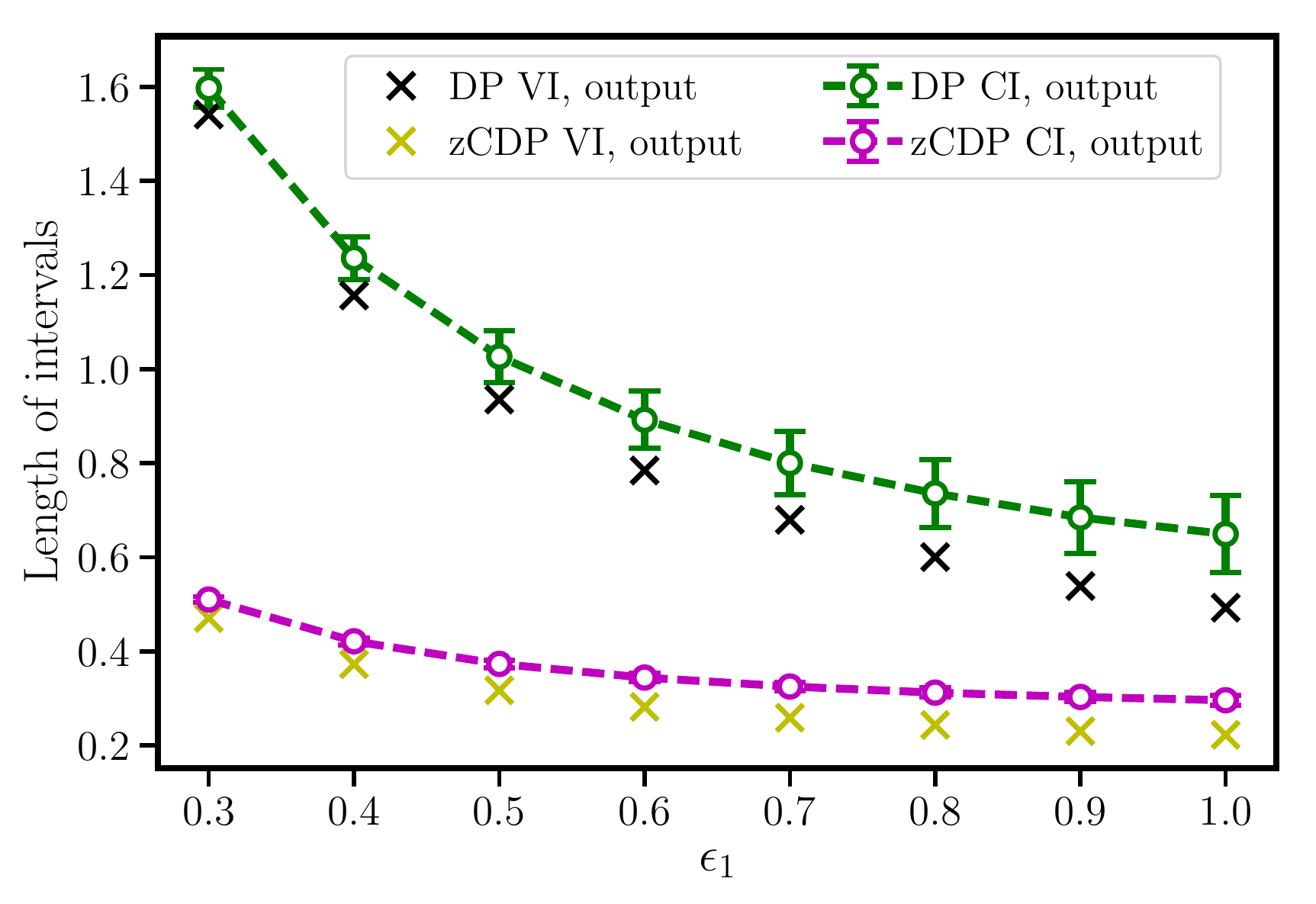}
}%
\hskip -8pt
\subfloat[US, LR, $\rho_1=\epsilon_1^2/2$]{
\includegraphics[width=0.33\textwidth]{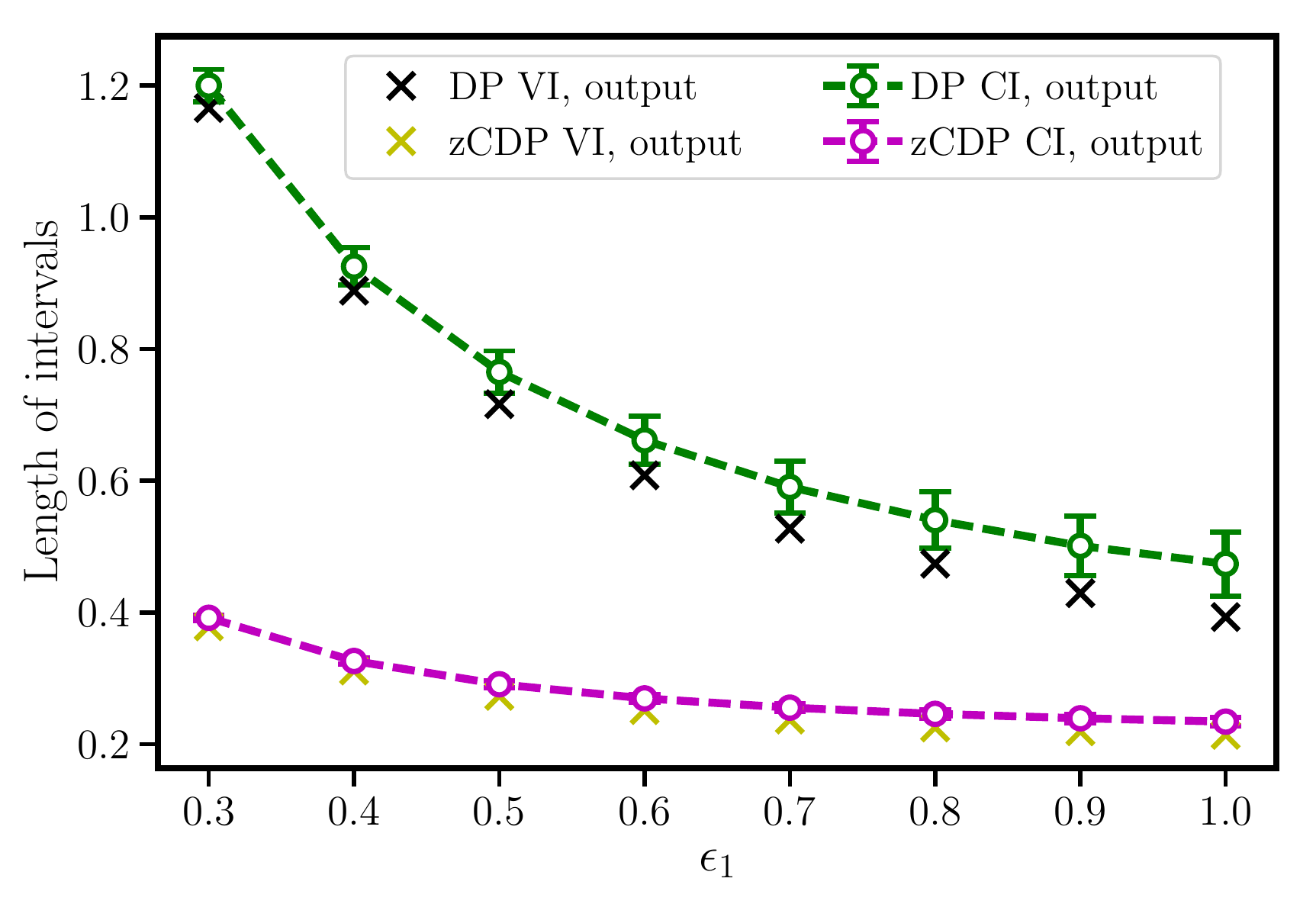}
}%
\hskip -8pt
\subfloat[BR, SVM, $\rho_2=\epsilon_2^2/2$]{
\includegraphics[width=0.33\textwidth]{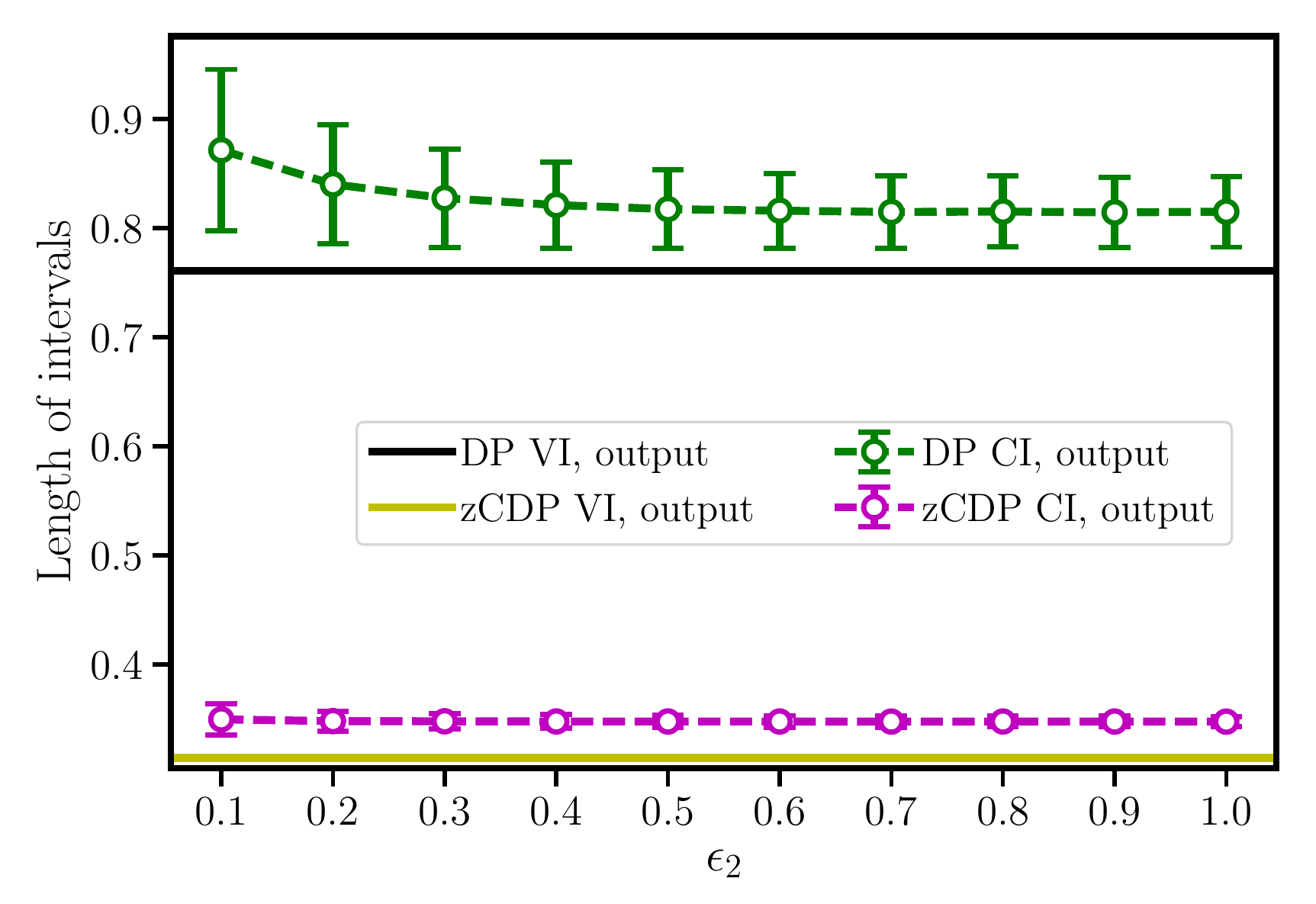}
}%
\hskip -8pt
\subfloat[BR, LR, $\rho_3=\epsilon_3^2/2$]{
\includegraphics[width=0.33\textwidth]{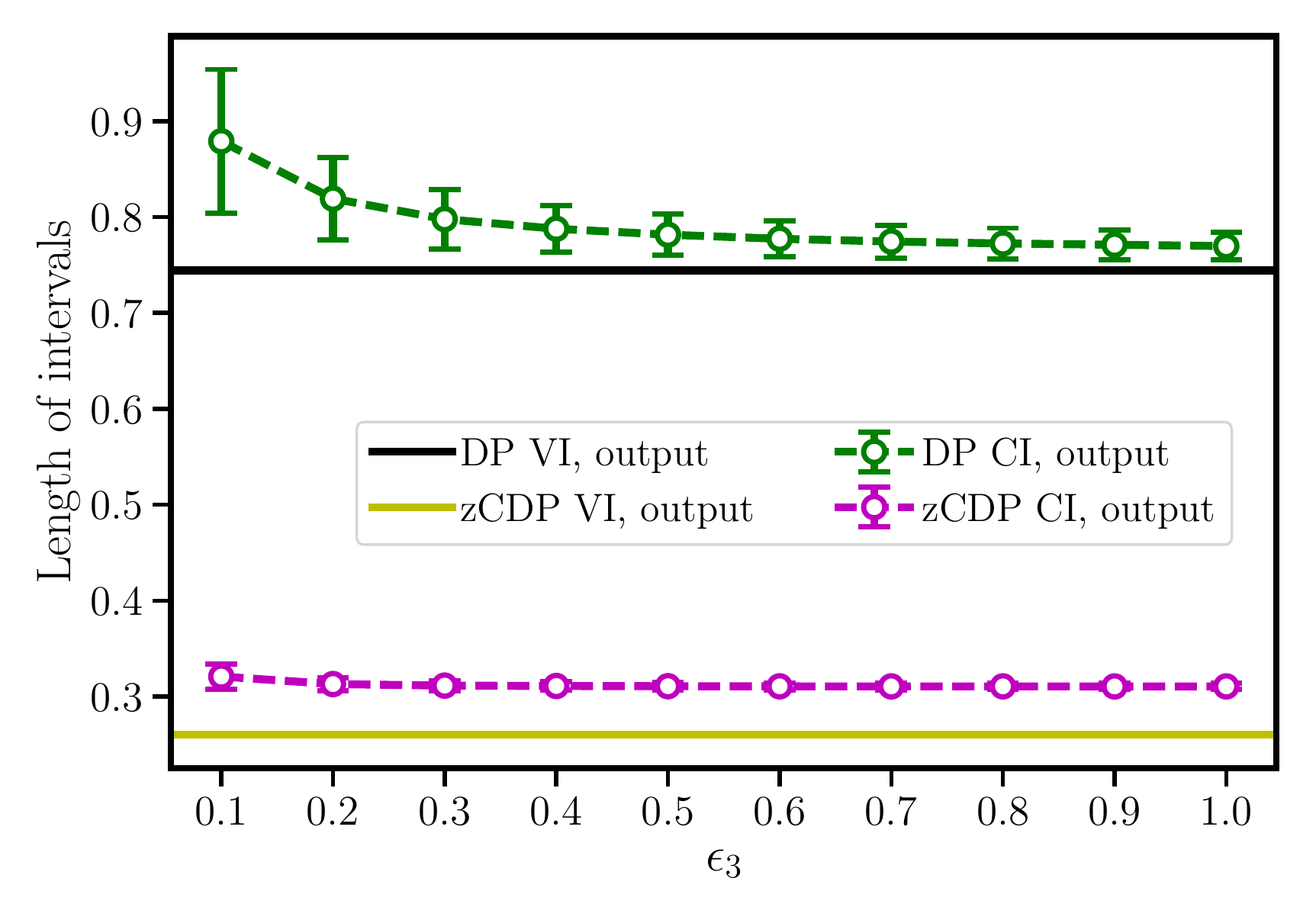}
}%
\vskip -8pt
\caption{Length of intervals for DP vs zCDP with output perturbation ERM. Error bars correspond to one standard deviation for CI. (a)-(b): $\epsilon_1=0.707$, $\rho_1=0.25$, $\epsilon_2=\epsilon_3=0.25$, $\rho_2=\rho_3=0.03125$; (c)-(d): $n=30,162$ for Adult, $n=39,928$ for US, $d=10$, $\epsilon_2=\epsilon_3=0.25$, $\rho_2=\rho_3=0.03125$; (e)-(f): $n=38,000$, $d=10$, $\epsilon_1=0.707$, $\rho_1=0.25$, $\epsilon_3=0.25$ and $\rho_3=0.03125$ for (e), $\epsilon_2=0.25$ and $\rho_2=0.03125$ for (f). Common parameters: $c=0.001$, $h=1.0$.}
\label{fig:output}
\end{figure*}
}
{
%%%%%%%%%%% tech report
\begin{figure*}[h!]
\centering
\subfloat[KDDCUP99, LR, $d=10$]{
\includegraphics[width=0.33\textwidth]{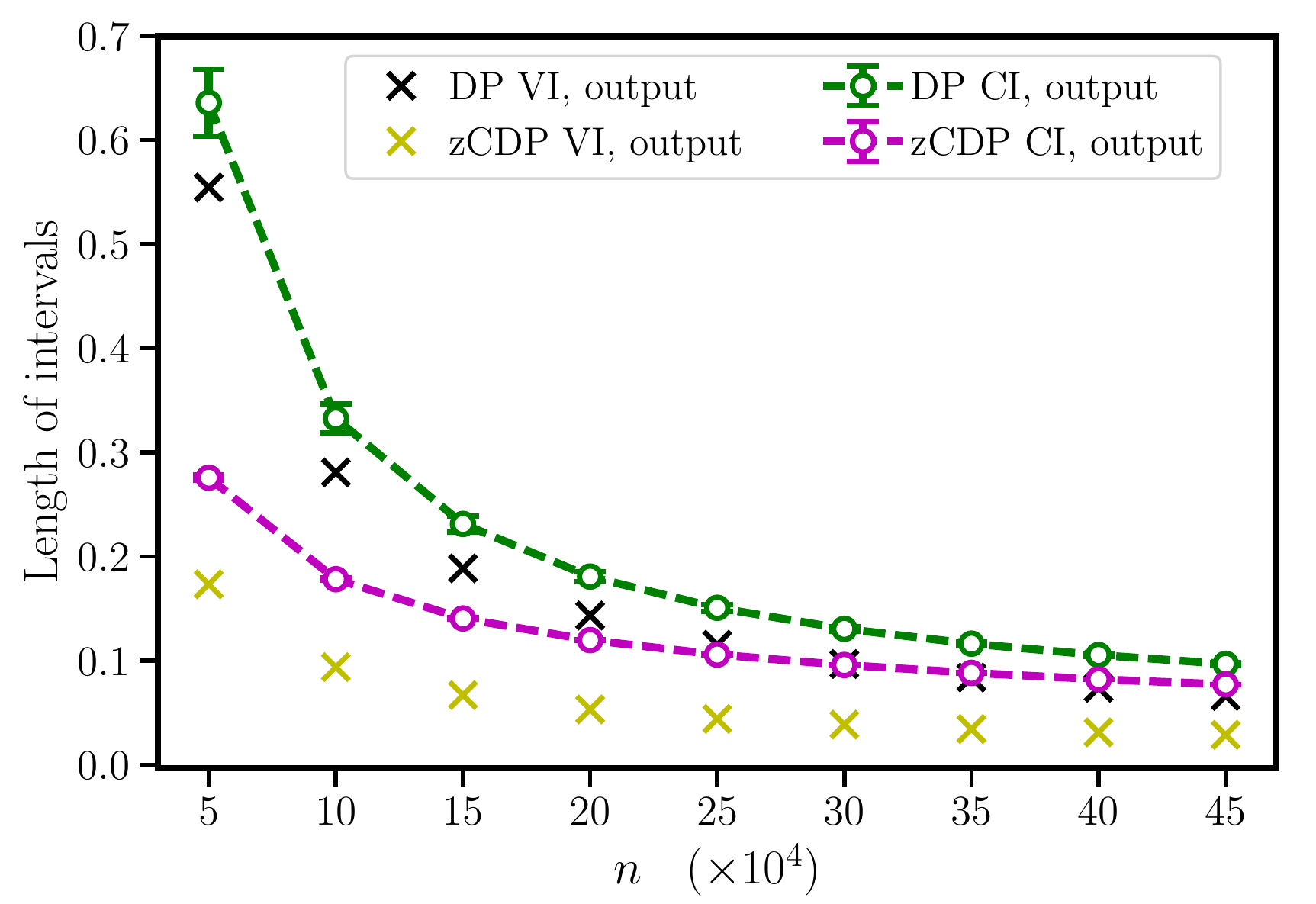}
}%
\hskip -8pt
\subfloat[KDDCUP99, SVM, $d=10$]{
\includegraphics[width=0.33\textwidth]{figs/kddcup-output-svm-n-d10-c0_001}
}%
\hskip -8pt
\subfloat[Banking, LR, $n=45,211$]{
\includegraphics[width=0.33\textwidth]{figs/banking-output-lr-n45211-d-c0_001}
}%
\hskip -8pt
\subfloat[Adult, SVM, $n=30,162$]{
\includegraphics[width=0.33\textwidth]{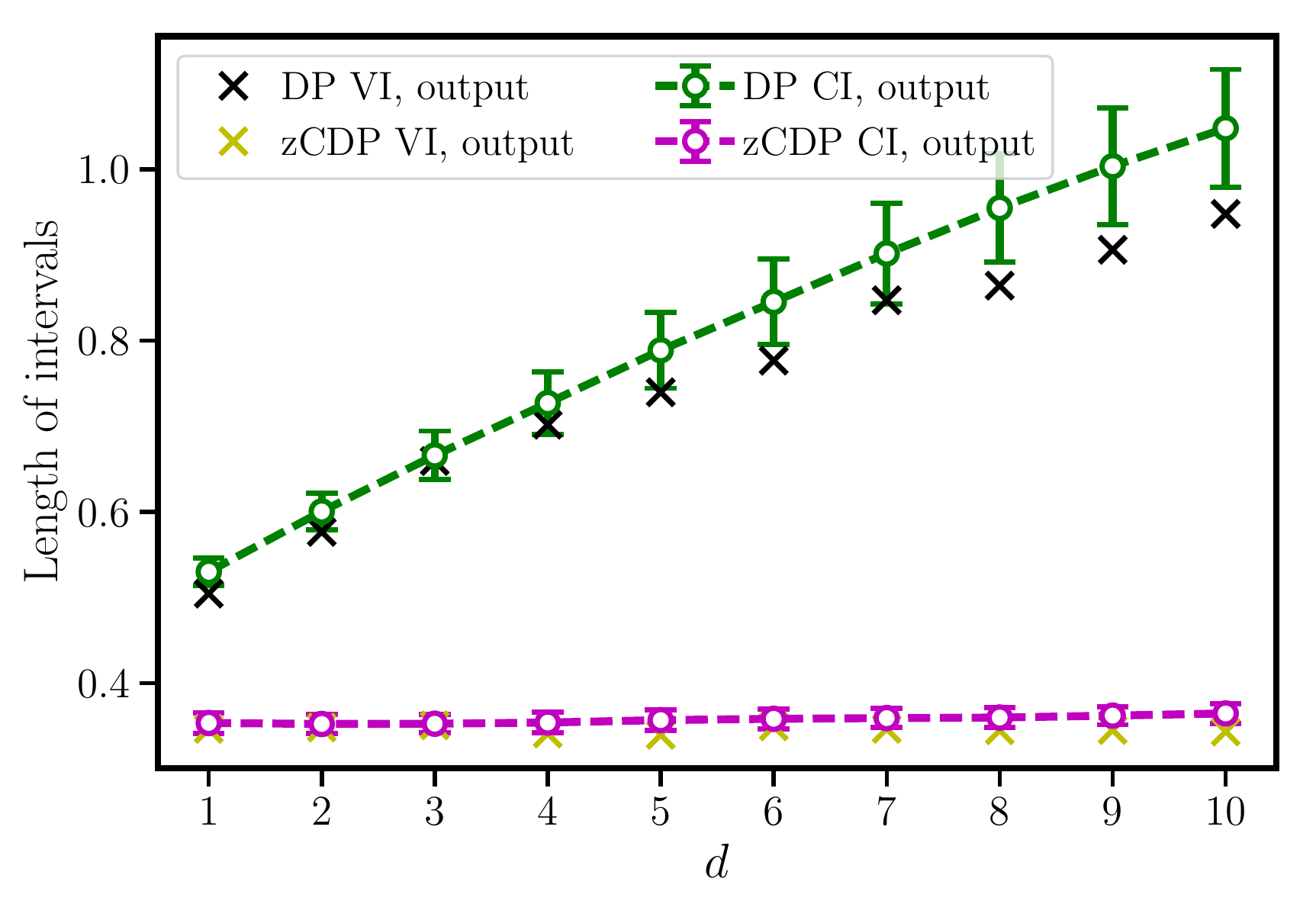}
}%
\hskip -8pt
\subfloat[Adult, LR, $\rho_1=\epsilon_1^2/2$]{
\includegraphics[width=0.33\textwidth]{figs/adult-output-lr-n30162-d10-c0_001-budget1}
}%
\hskip -8pt
\subfloat[US, LR, $\rho_1=\epsilon_1^2/2$]{
\includegraphics[width=0.33\textwidth]{figs/us-output-lr-n39928-d10-c0_001-budget1}
}%
\hskip -8pt
\subfloat[BR, SVM, $\rho_2=\epsilon_2^2/2$]{
\includegraphics[width=0.33\textwidth]{figs/br-output-svm-n38000-d10-c0_001-budget2}
}%
\hskip -8pt
\subfloat[BR, LR, $\rho_3=\epsilon_3^2/2$]{
\includegraphics[width=0.33\textwidth]{figs/br-output-lr-n38000-d10-c0_001-budget3}
}%
\hskip -8pt
\subfloat[Banking, SVM, $\rho_3=\epsilon_3^2/2$]{
\includegraphics[width=0.33\textwidth]{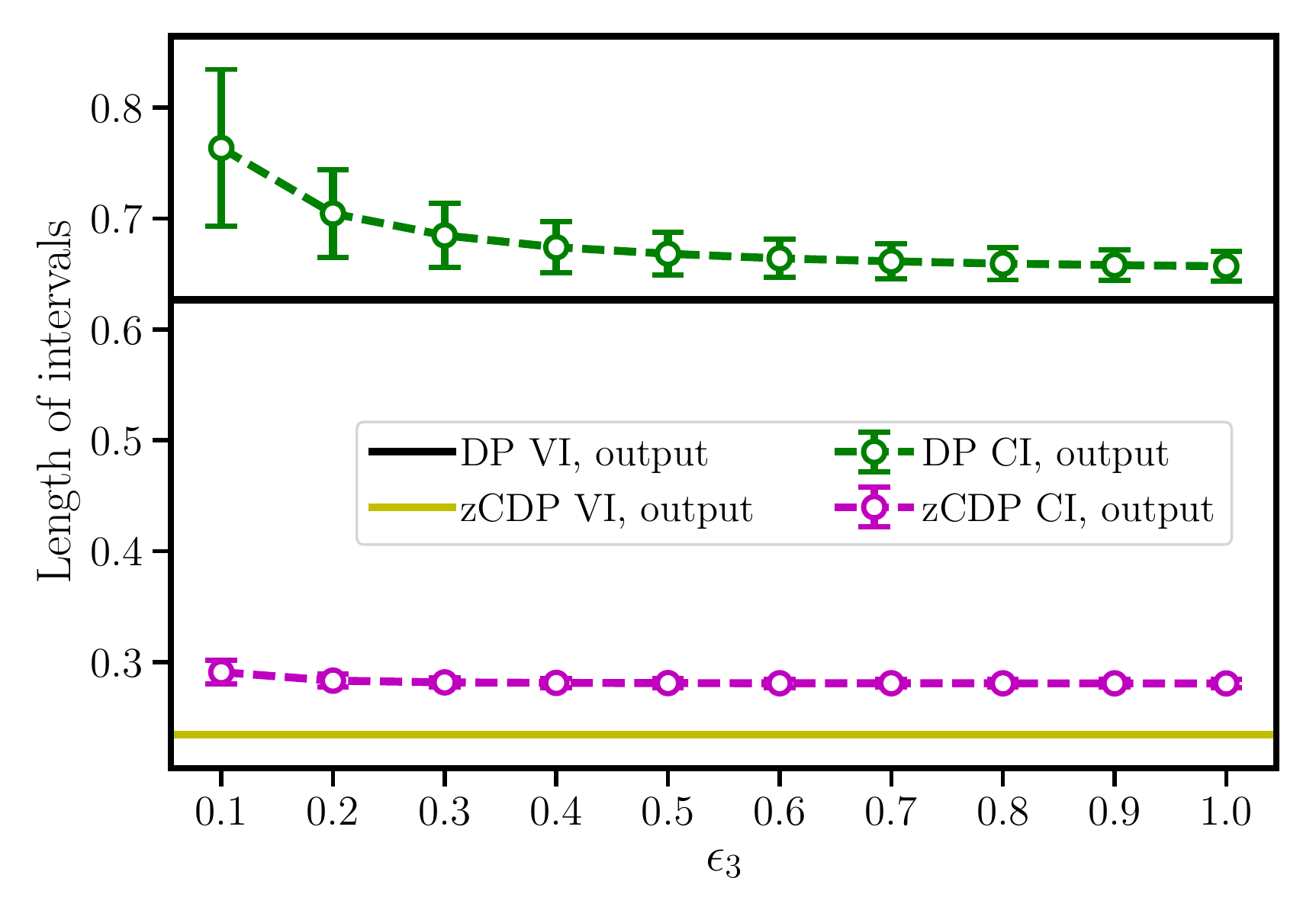}
}%
\vskip -8pt
\caption{Length of intervals for DP vs zCDP with output perturbation ERM. Error bars correspond to one standard deviation for CI. (a)-(d): $\epsilon_1=0.5$, $\rho_1=0.125$, $\epsilon_2=\epsilon_3=0.25$, $\rho_2=\rho_3=0.03125$; (e)-(f): $n=30,162$ for Adult, $n=39,928$ for US, $d=10$, $\epsilon_2=\epsilon_3=0.25$, $\rho_2=\rho_3=0.03125$; (g)-(i): $n=38,000$ for BR, $n=45,211$ for Banking, $d=10$, $\epsilon_1=0.5$, $\rho_1=0.125$, $\epsilon_3=0.25$ and $\rho_3=0.03125$ for (g), $\epsilon_2=0.25$ and $\rho_2=0.03125$ for (h) and (i). Common parameters: $c=0.001$, $h=1.0$.}
\label{fig:output}
\end{figure*}
}

\ConfOrTech{
%%%%%%%%% conference
\begin{figure*}[h!]
\centering
\subfloat[KDDCUP99, LR, $d=10$]{
\includegraphics[width=0.33\textwidth]{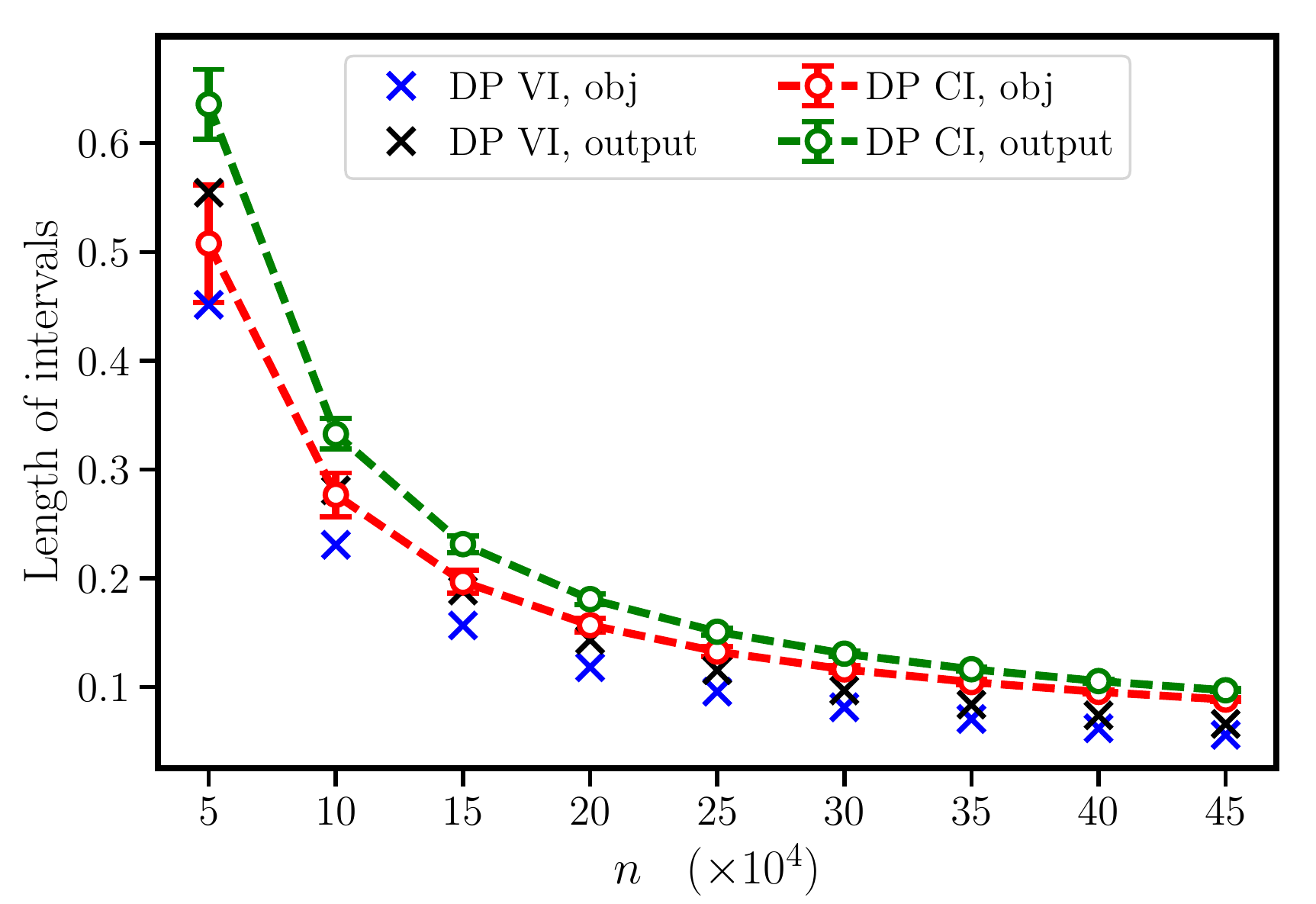}
}%
\hskip -8pt
\subfloat[US, SVM, $n=39,928$]{
\includegraphics[width=0.33\textwidth]{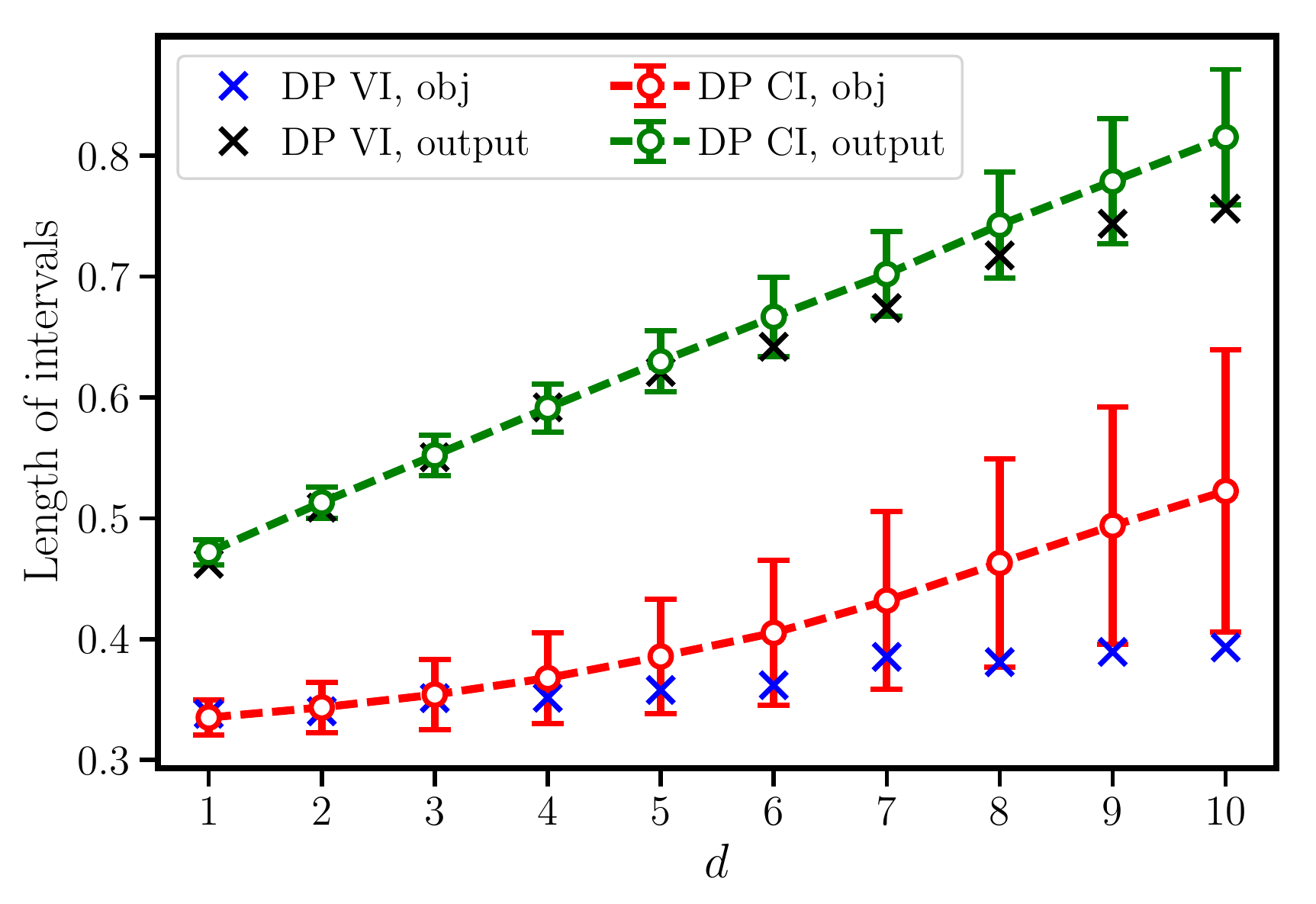}
}%
\hskip -8pt
\subfloat[BR, LR, $n=38,000$]{
\includegraphics[width=0.33\textwidth]{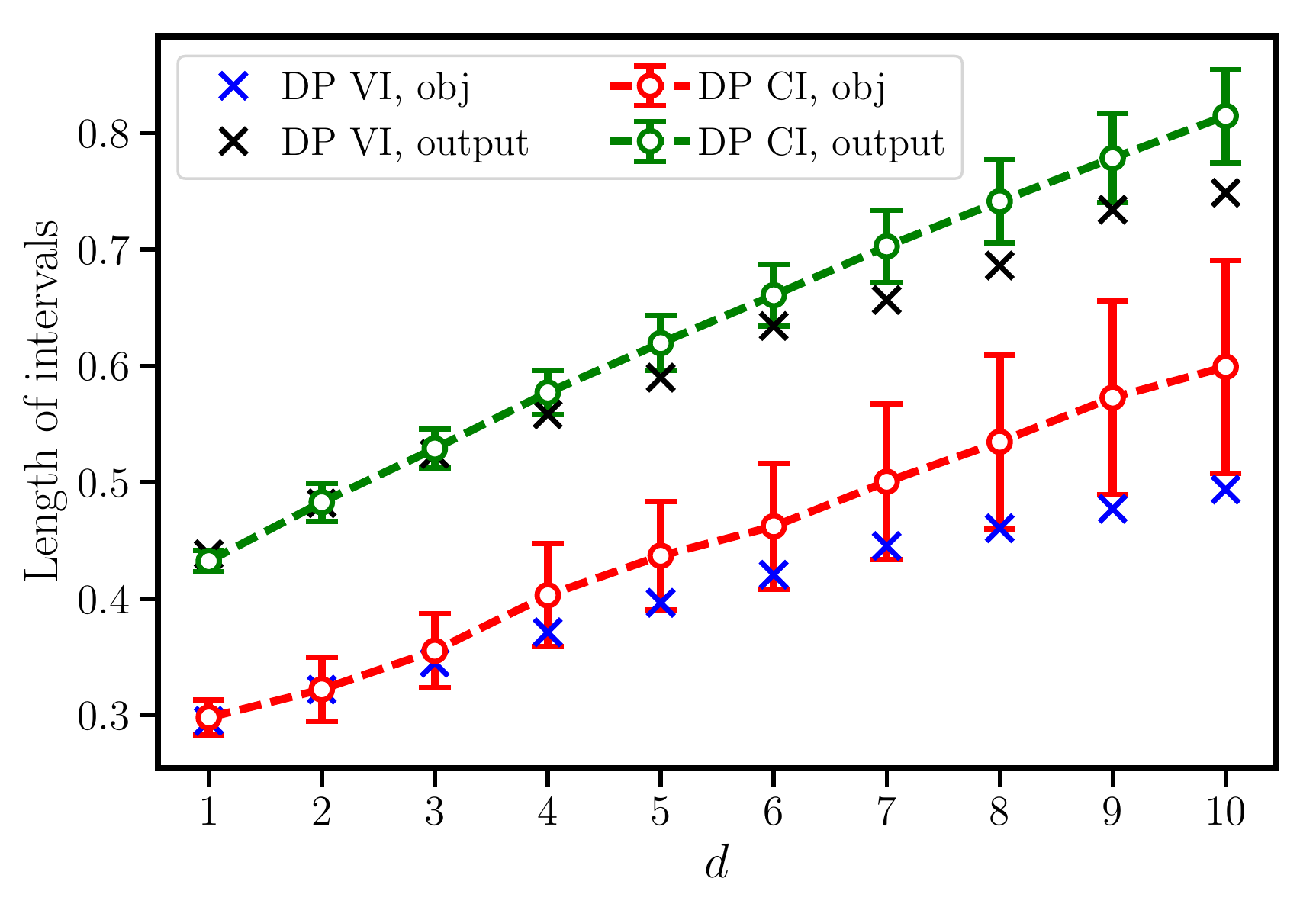}
}%
\hskip -8pt
\subfloat[Adult, LR, $\rho_1=\epsilon_1^2/2$]{
\includegraphics[width=0.33\textwidth]{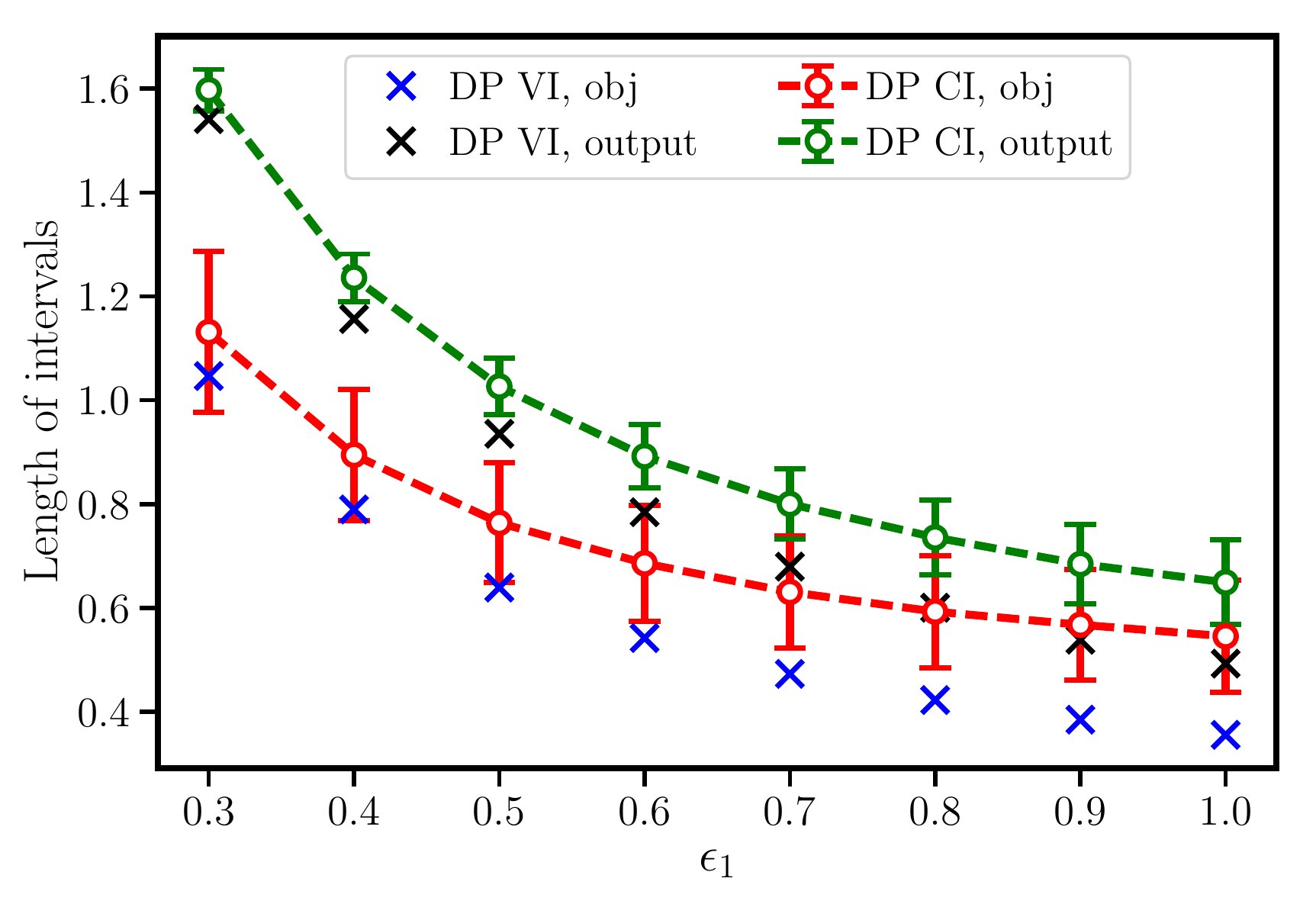}
}%
\hskip -8pt
\subfloat[Banking, SVM, $\rho_2=\epsilon_2^2/2$]{
\includegraphics[width=0.33\textwidth]{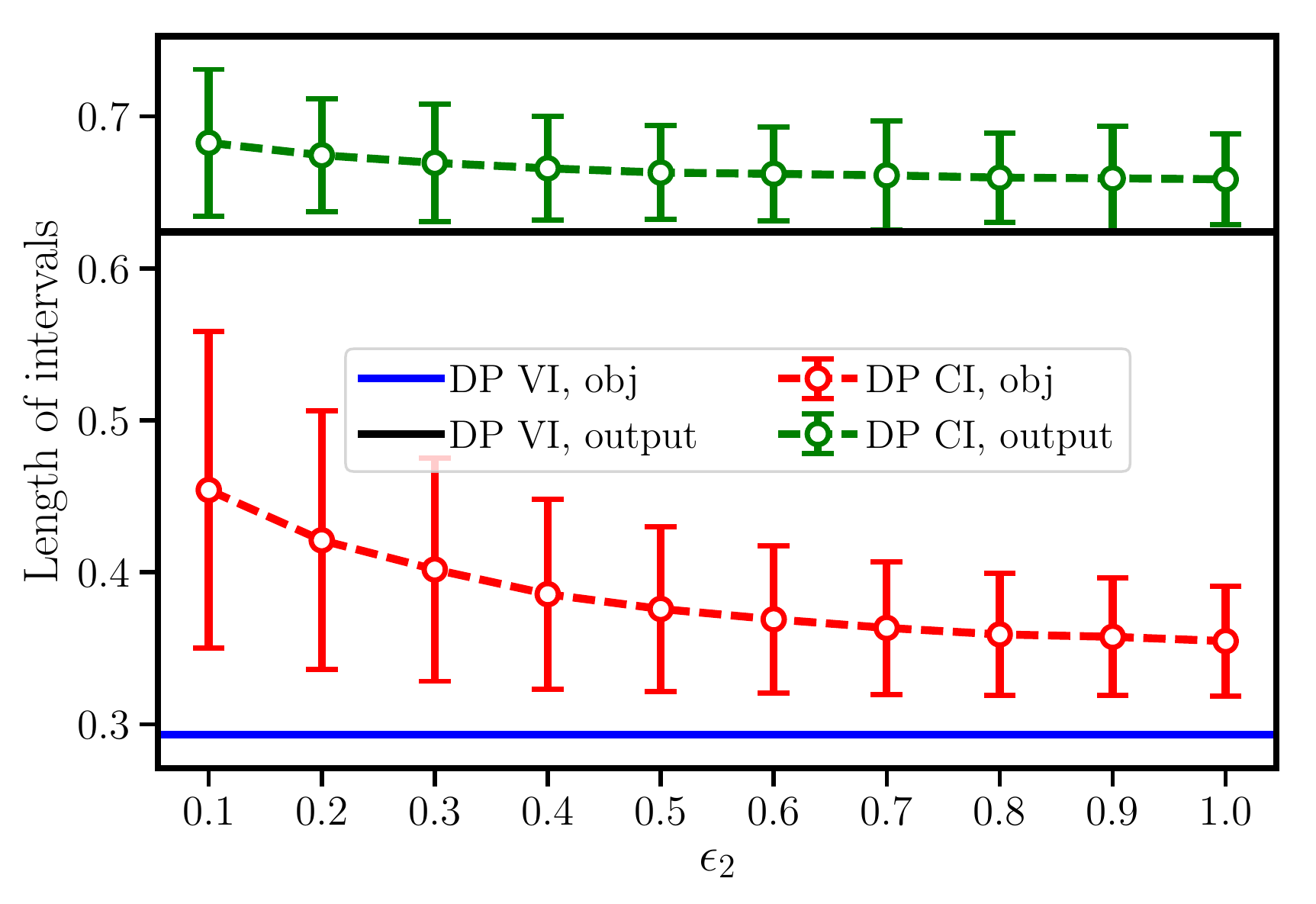}
}%
\hskip -8pt
\subfloat[Banking, LR, $\rho_3=\epsilon_3^2/2$]{
\includegraphics[width=0.33\textwidth]{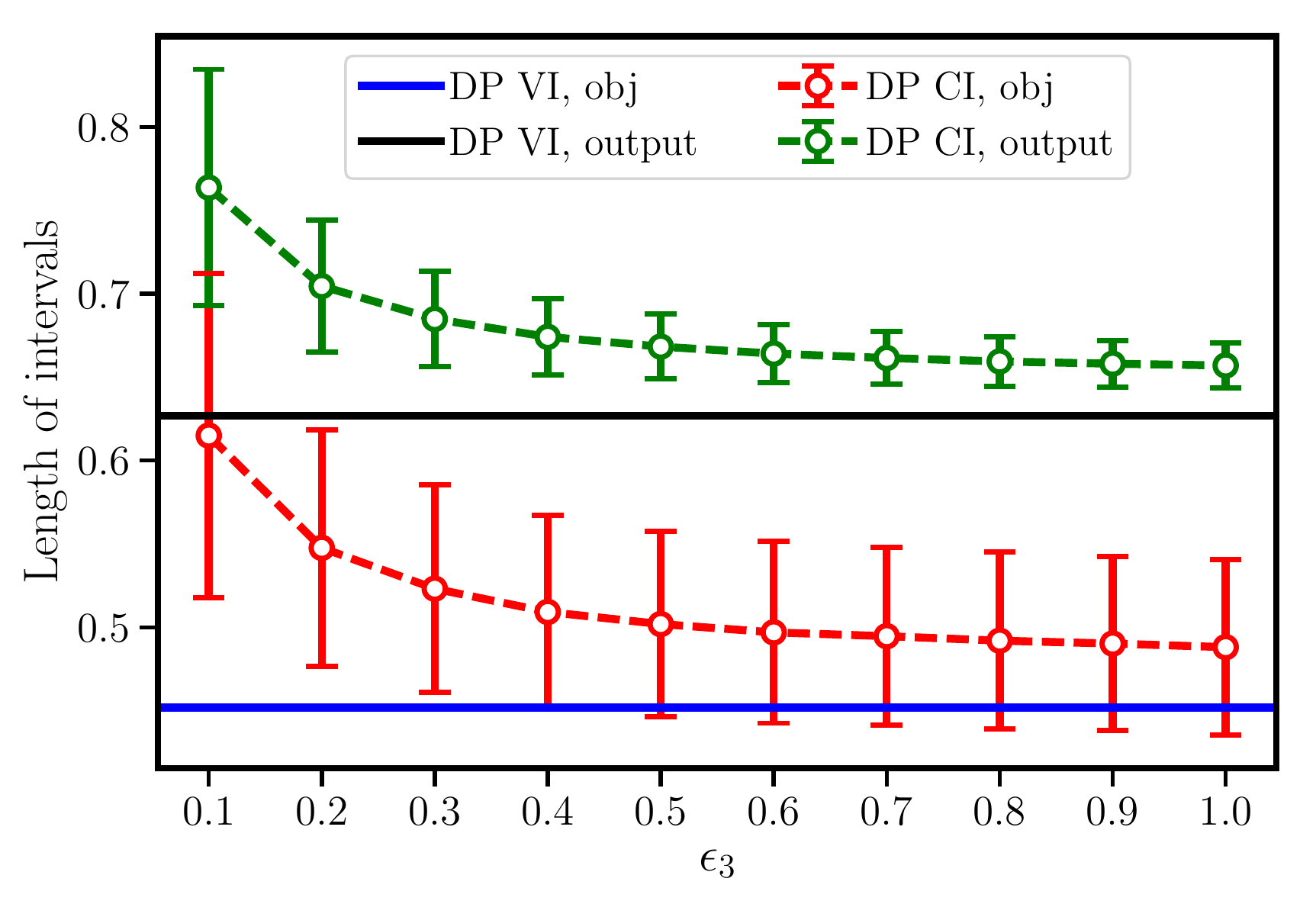}
}%
\vskip -8pt
\caption{Length of intervals for objective vs output perturbation with DP. Error bars correspond to one standard deviation for CI. (a)-(c): $\epsilon_1=0.707$, $\rho_1=0.25$, $\epsilon_2=\epsilon_3=0.25$, $\rho_2=\rho_3=0.03125$; (d): $n=30,162$, $d=10$, $\epsilon_2=\epsilon_3=0.25$, $\rho_2=\rho_3=0.03125$; (e)-(f): $n=45,211$, $d=10$, $\epsilon_1=0.707$, $\rho_1=0.25$, $\epsilon_3=0.25$ and $\rho_3=0.03125$ for (e), $\epsilon_2=0.25$ and $\rho_2=0.03125$ for (f). Common parameters: $c=0.001$, $h=1.0$.}
\label{fig:dp}
\end{figure*}
}
{
%%%%%%%%% tech report
\begin{figure*}[h!]
\centering
\subfloat[KDDCUP99, LR, $d=10$]{
\includegraphics[width=0.33\textwidth]{figs/kddcup-dp-lr-n-d10-c0_001}
}%
\hskip -8pt
\subfloat[KDDCUP99, SVM, $d=10$]{
\includegraphics[width=0.33\textwidth]{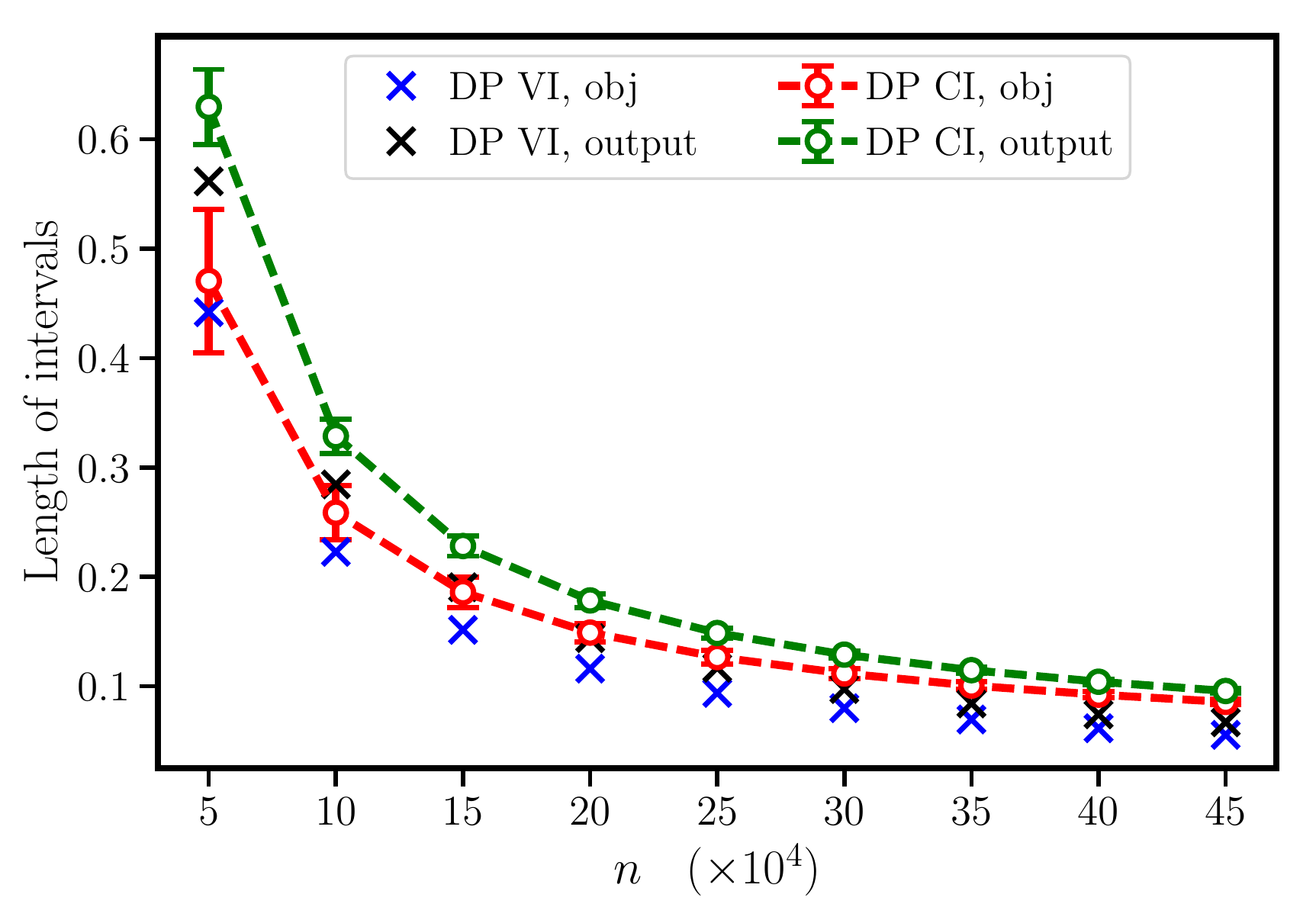}
}%
\hskip -8pt
\subfloat[US, SVM, $n=39,928$]{
\includegraphics[width=0.33\textwidth]{figs/us-dp-svm-n39928-d-c0_001}
}%
\hskip -8pt
\subfloat[BR, LR, $n=38,000$]{
\includegraphics[width=0.33\textwidth]{figs/br-dp-lr-n38000-d-c0_001}
}%
\hskip -8pt
\subfloat[Adult, LR, $\rho_1=\epsilon_1^2/2$]{
\includegraphics[width=0.33\textwidth]{figs/adult-dp-lr-n30162-d10-c0_001-budget1}
}%
\hskip -8pt
\subfloat[US, SVM, $\rho_1=\epsilon_1^2/2$]{
\includegraphics[width=0.33\textwidth]{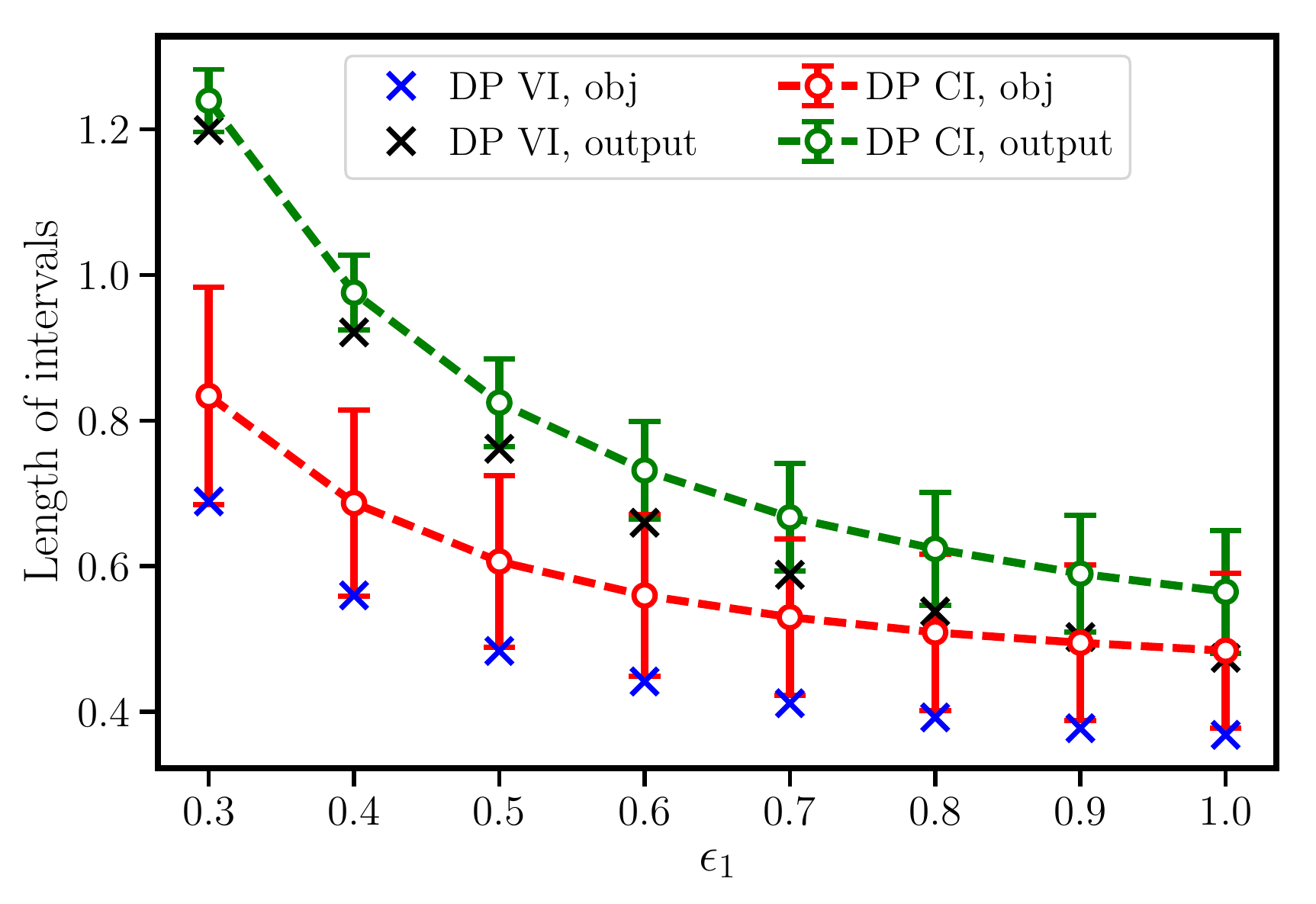}
}%
\hskip -8pt
\subfloat[BR, LR, $\rho_2=\epsilon_2^2/2$]{
\includegraphics[width=0.33\textwidth]{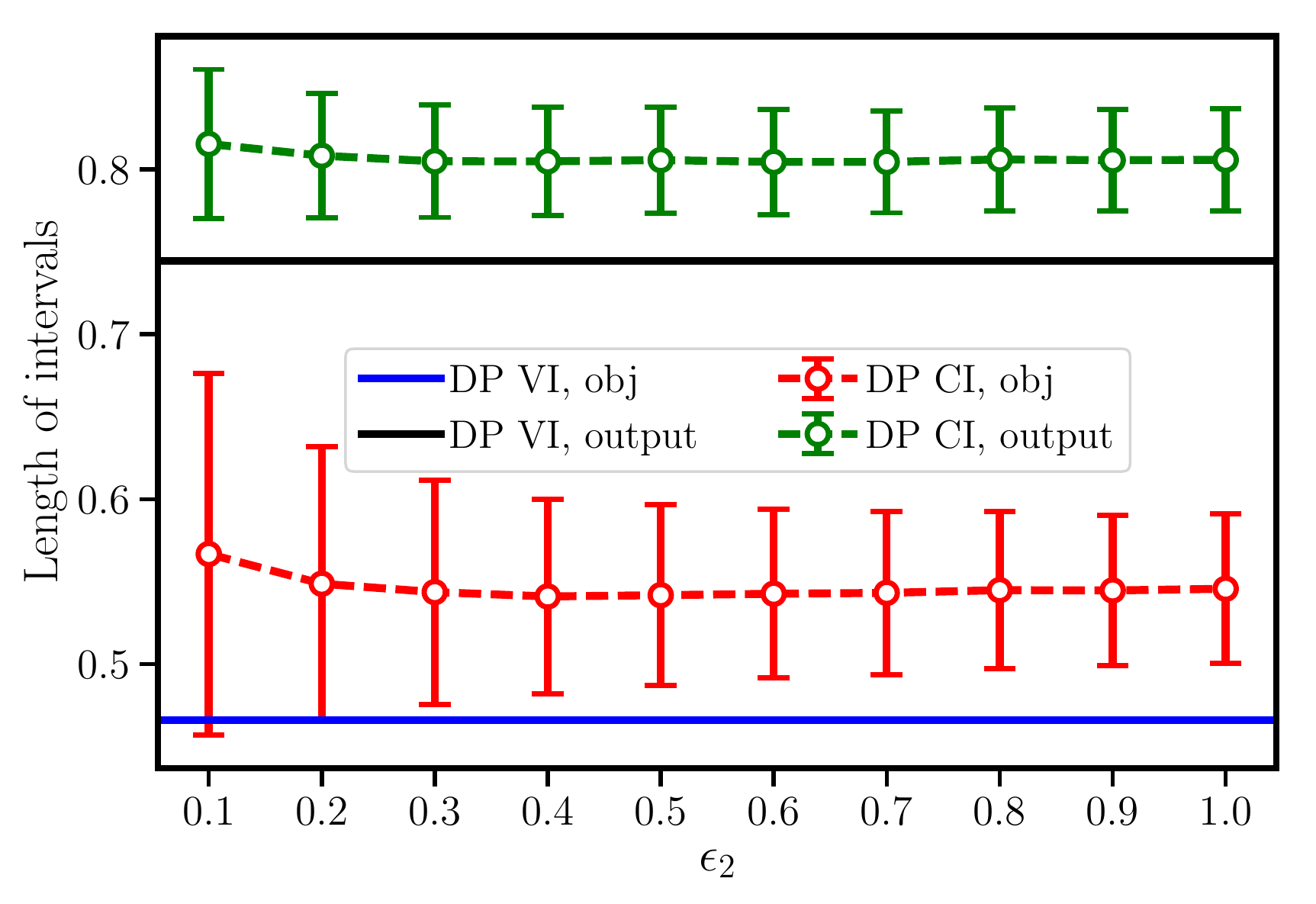}
}%
\hskip -8pt
\subfloat[Banking, SVM, $\rho_2=\epsilon_2^2/2$]{
\includegraphics[width=0.33\textwidth]{figs/banking-dp-svm-n45211-d10-c0_001-budget2}
}%
\hskip -8pt
\subfloat[Banking, LR, $\rho_3=\epsilon_3^2/2$]{
\includegraphics[width=0.33\textwidth]{figs/banking-dp-lr-n45211-d10-c0_001-budget3}
}%
\vskip -8pt
\caption{Length of intervals for objective vs output perturbation with DP. Error bars correspond to one standard deviation for CI. (a)-(d): $\epsilon_1=0.5$, $\rho_1=0.125$, $\epsilon_2=\epsilon_3=0.25$, $\rho_2=\rho_3=0.03125$; (e)-(f): $n=30,162$ for Adult, $n=39,928$ for US, $d=10$, $\epsilon_2=\epsilon_3=0.25$, $\rho_2=\rho_3=0.03125$; (g)-(i): $n=38,000$ for BR, $n=45,211$ for Banking, $d=10$, $\epsilon_1=0.5$, $\rho_1=0.125$, $\epsilon_3=0.25$ and $\rho_3=0.03125$ for (g) and (h), $\epsilon_2=0.25$ and $\rho_2=0.03125$ for (i). Common parameters: $c=0.001$, $h=1.0$.}
\label{fig:dp}
\end{figure*}
}

\subsubsection{DP vs zCDP with Output Perturbation ERM}

In Figure~\ref{fig:output}, we experiment with the different privacy definitions for parameters obtained  with output perturbation. From the figure, we see zCDP outperforms differential privacy with shorter intervals in length as well as smaller variance. The length of each confidence interval is either very close to or only a bit longer than their corresponding variability interval. One interesting phenomenon we notice is that the length of both intervals changes very slowly when we vary the dimensionality for the output perturbation technique with zCDP. This is due to the smaller tails of the Gaussian distribution. But when we get a large amount of data (e.g., when $n=450,000$ from \ConfOrTech{Figure~\ref{fig:output}a}{Figures \ref{fig:output}a and~\ref{fig:output}b}), the data overwhelms the noise and reduces the difference between the settings.

\subsubsection{Objective vs Output Perturbation with DP}

In Figure~\ref{fig:dp}, we use differential privacy to compare the confidence intervals that are achievable for models trained with objective perturbation against output perturbation. The figure shows that objective perturbation outperforms output perturbation in the length of intervals but also yields slightly larger variances in interval length.
%This may suggest objective perturbation is better tailored to protect differential privacy. 
Again, when we have enough samples (e.g., when $n=450,000$ in \ConfOrTech{Figure~\ref{fig:dp}a}{Figures \ref{fig:dp}a and~\ref{fig:dp}b}), the difference in performance due to different perturbation techniques is almost negligible.

\begin{figure*}[h!]
\centering
\subfloat[KDDCUP99, LR, $d=10$]{
\includegraphics[width=0.33\textwidth]{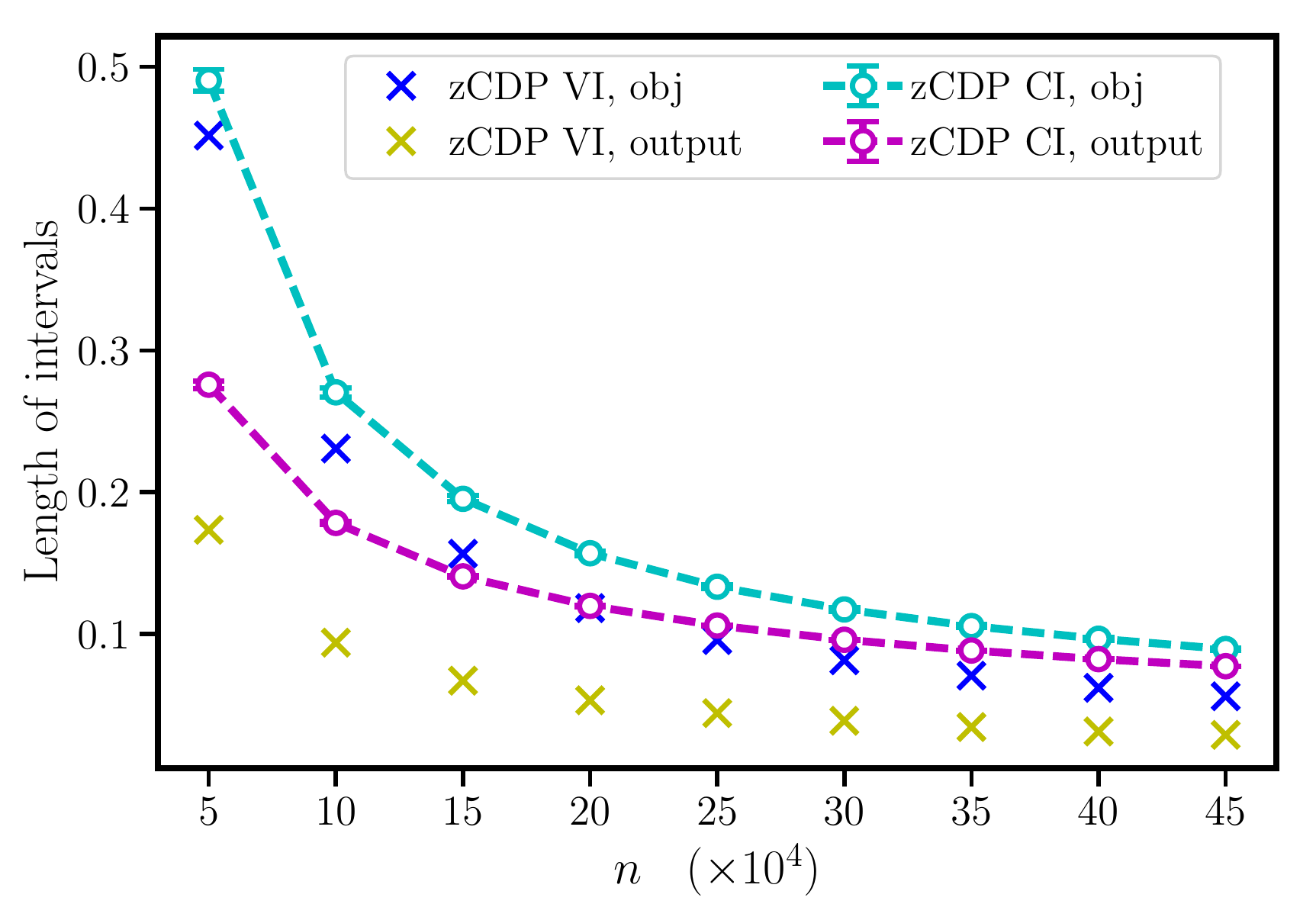}
}%
\hskip -8pt
\subfloat[KDDCUP99, SVM, $d=10$]{
\includegraphics[width=0.33\textwidth]{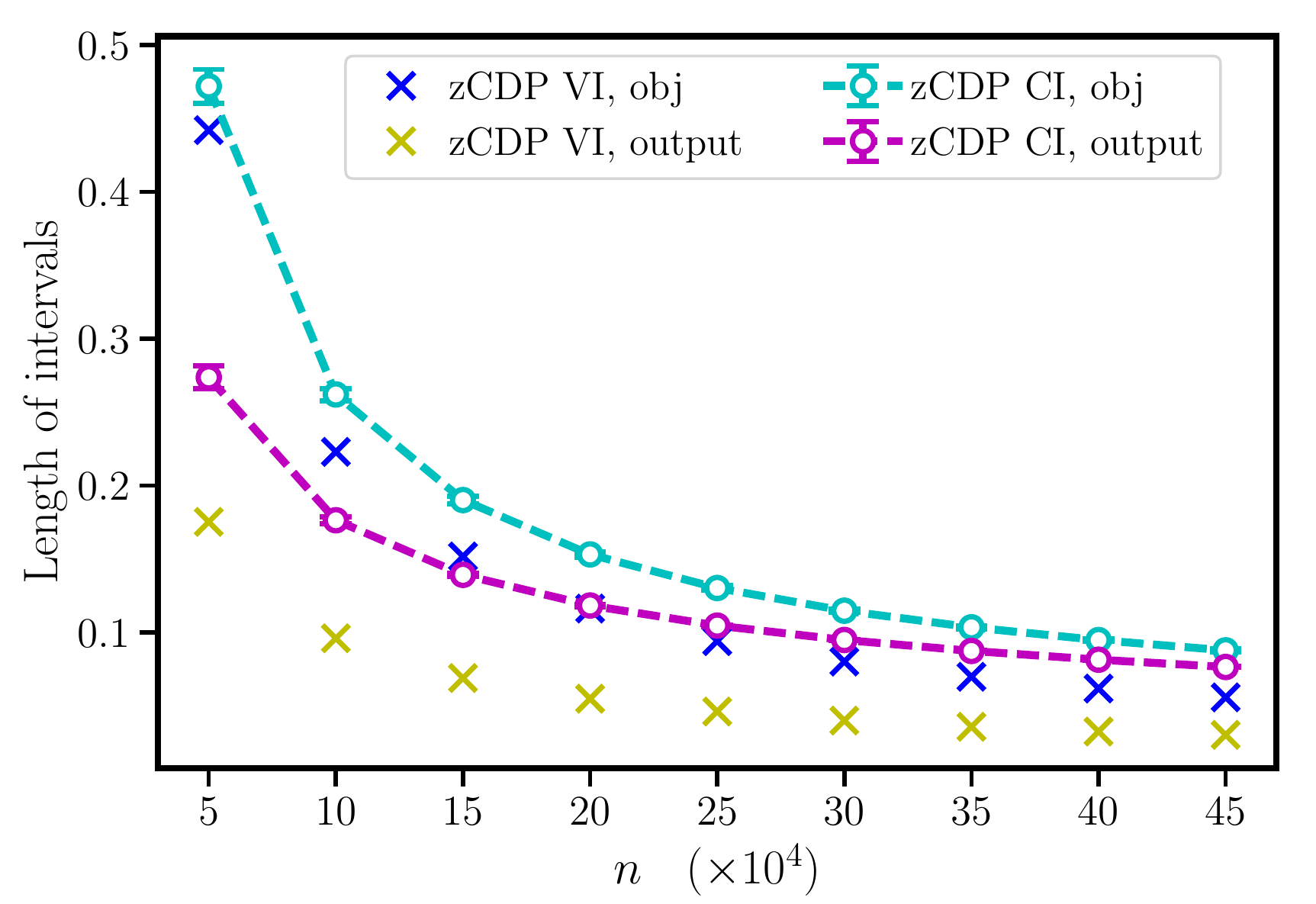}
}%
\hskip -8pt
\subfloat[BR, LR, $n=38,000$]{
\includegraphics[width=0.33\textwidth]{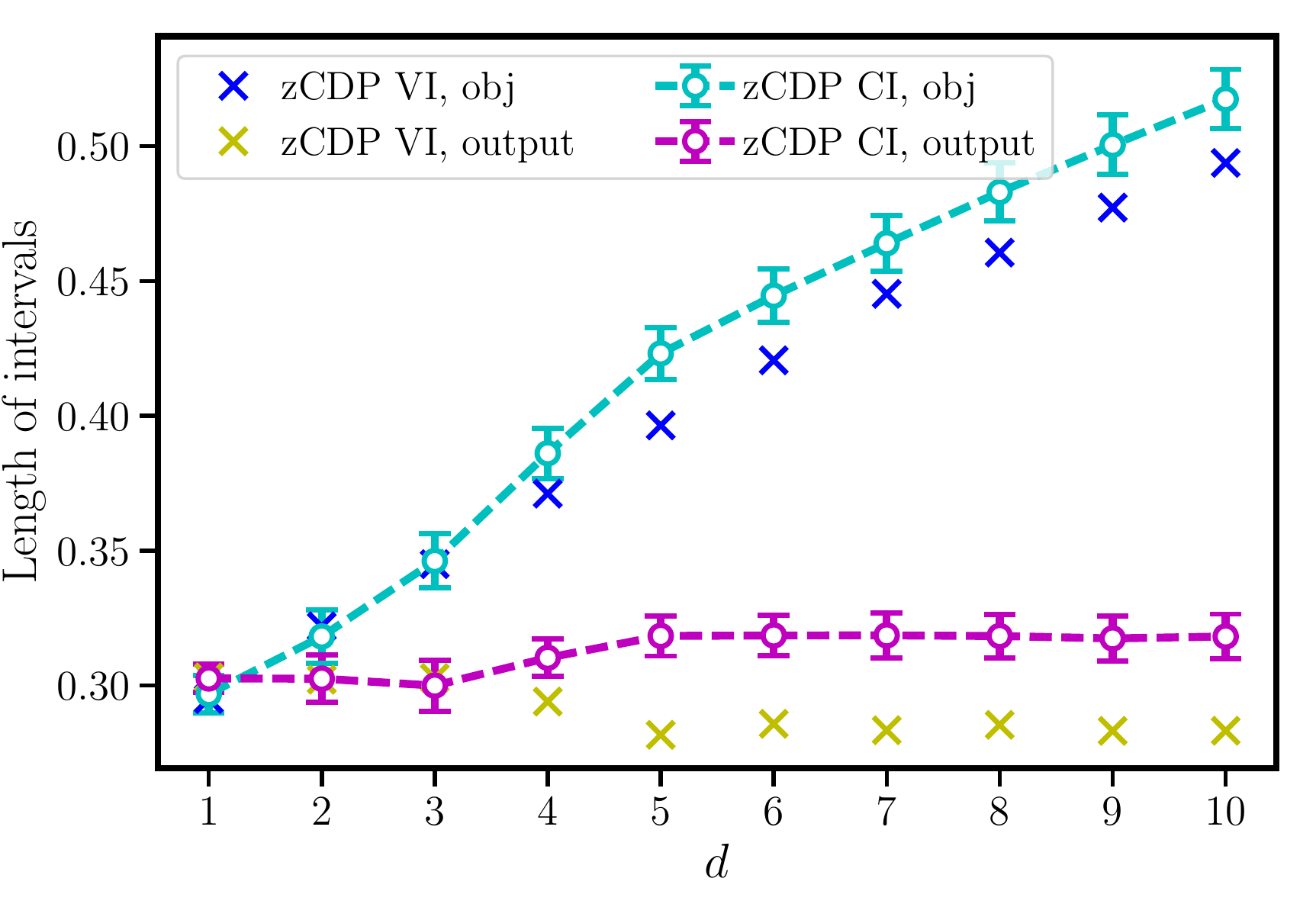}
}%
\hskip -8pt
\subfloat[US, LR, $n=39,928$]{
\includegraphics[width=0.33\textwidth]{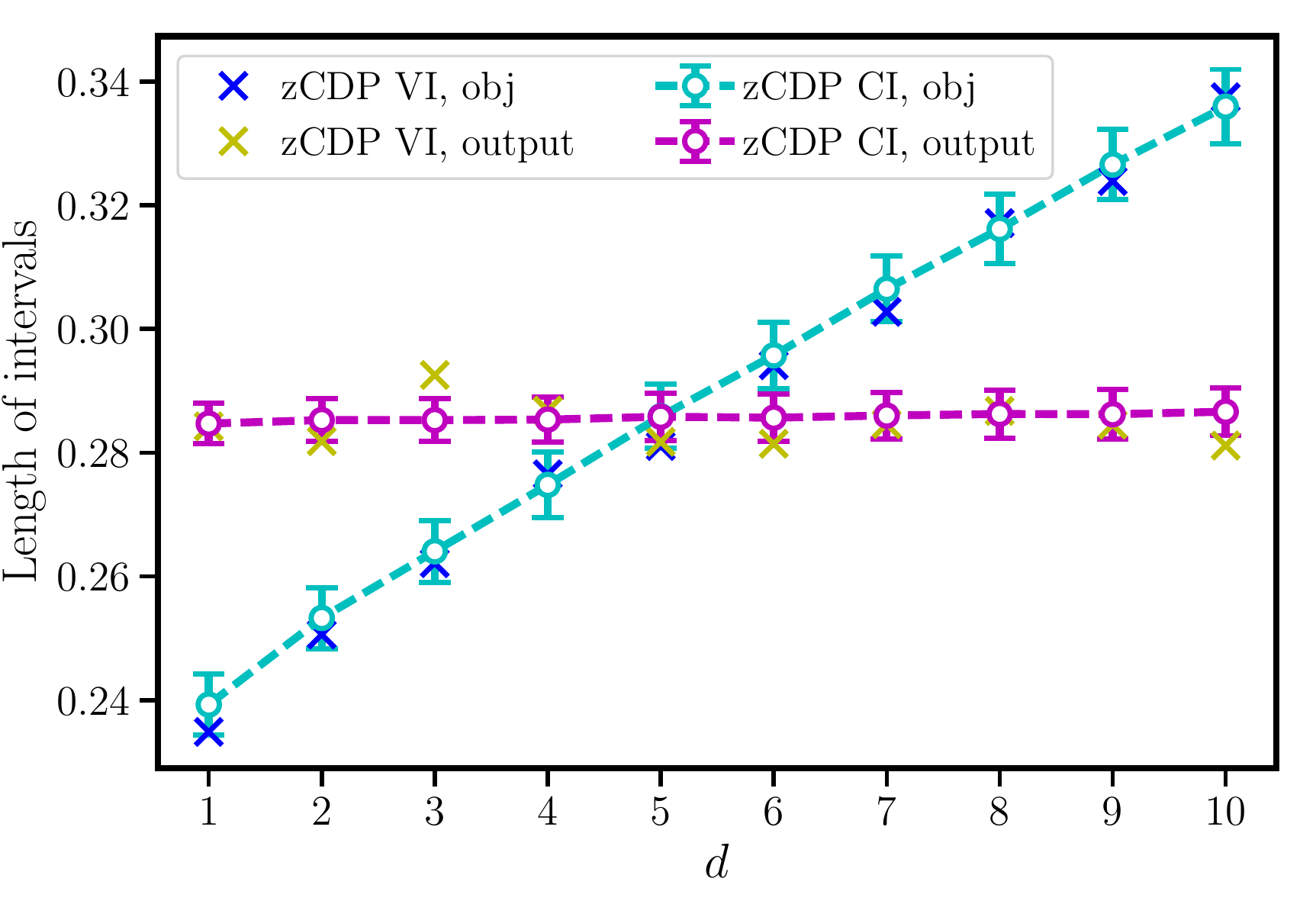}
}%
\hskip -8pt
\subfloat[Banking, LR, $\rho_1=\epsilon_1^2/2$]{
\includegraphics[width=0.33\textwidth]{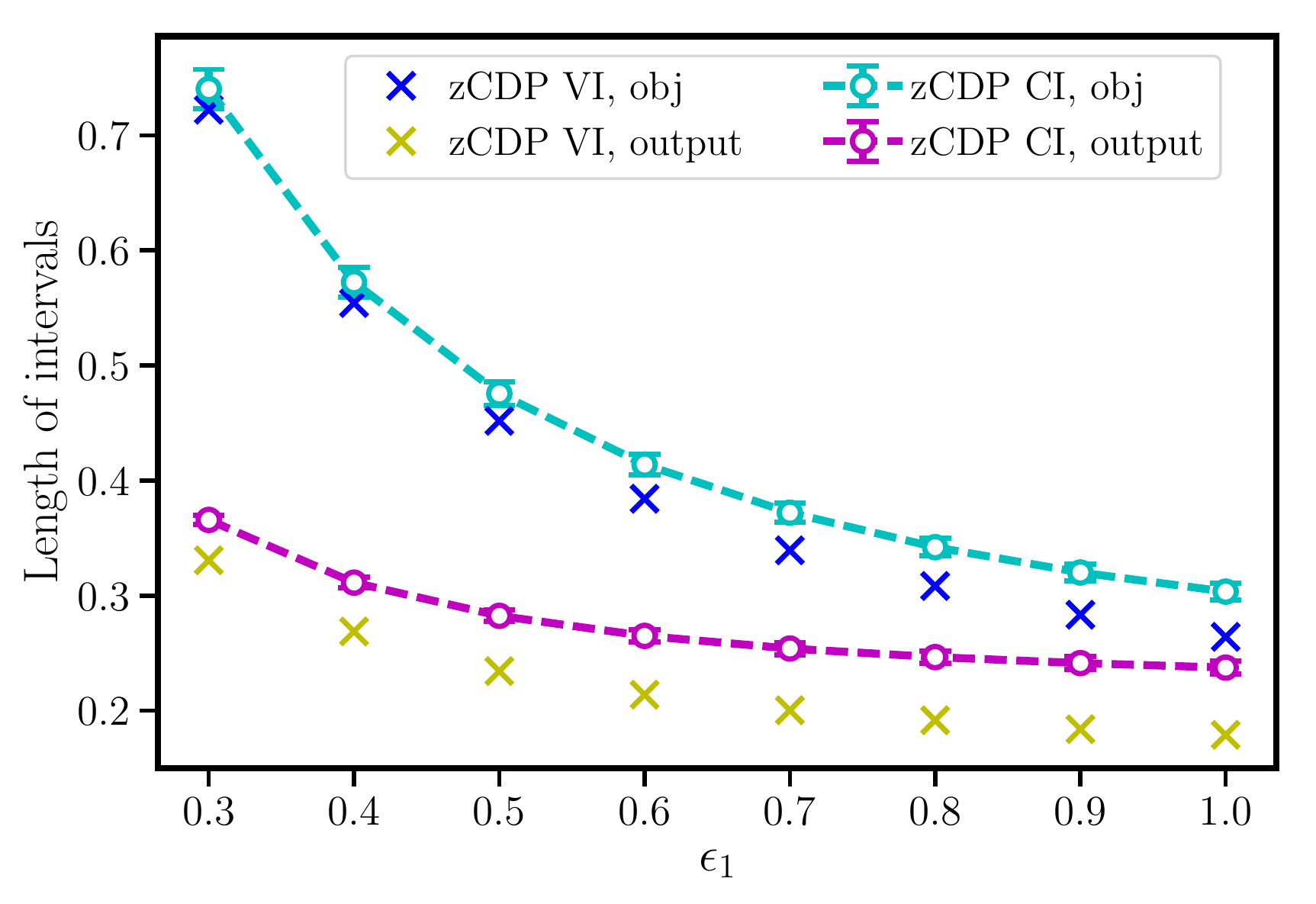}
}%
\hskip -8pt
\subfloat[Banking, SVM, $\rho_1=\epsilon_1^2/2$]{
\includegraphics[width=0.33\textwidth]{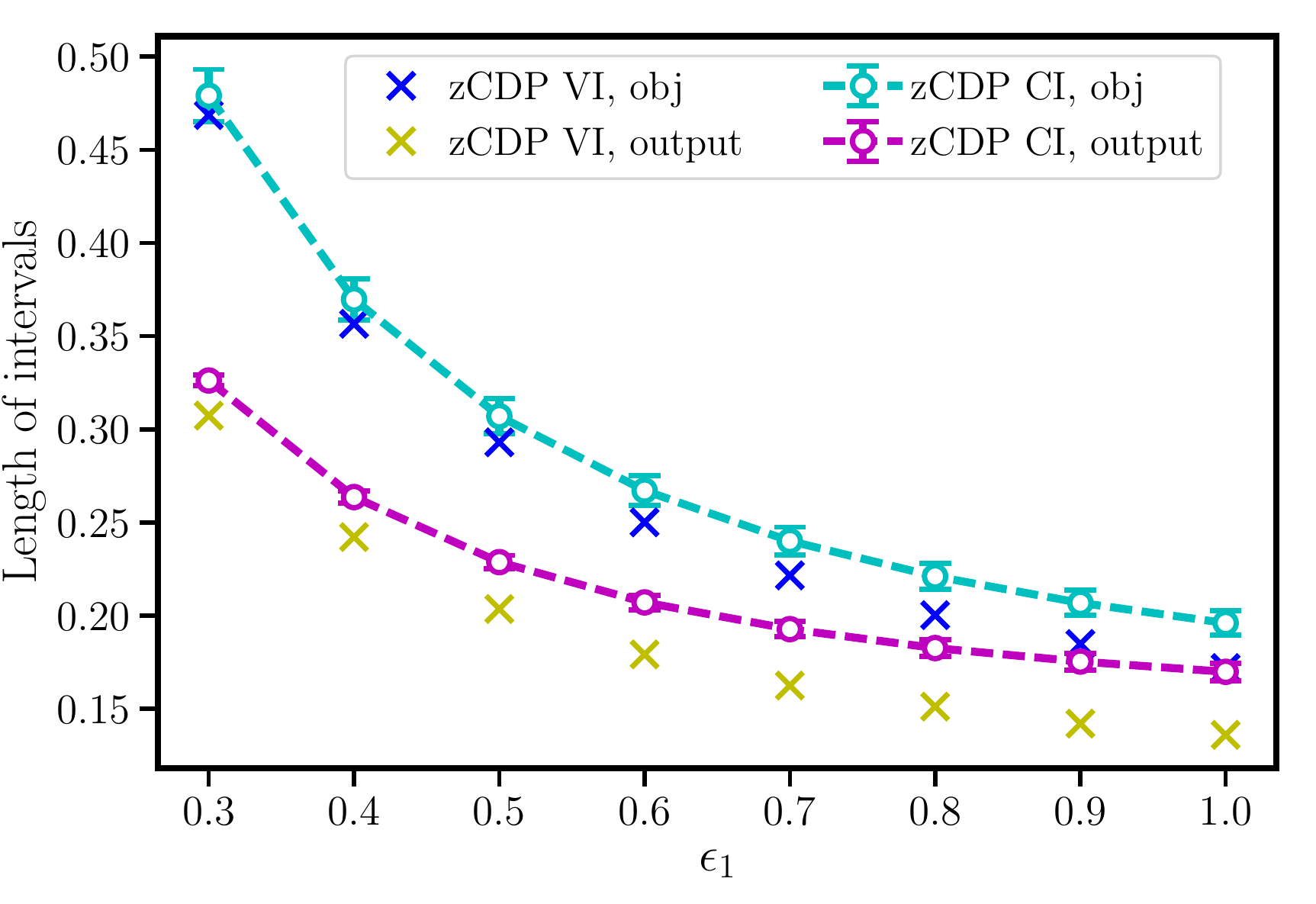}
}%
\hskip -8pt
\subfloat[Adult, SVM, $\rho_2=\epsilon_2^2/2$]{
\includegraphics[width=0.33\textwidth]{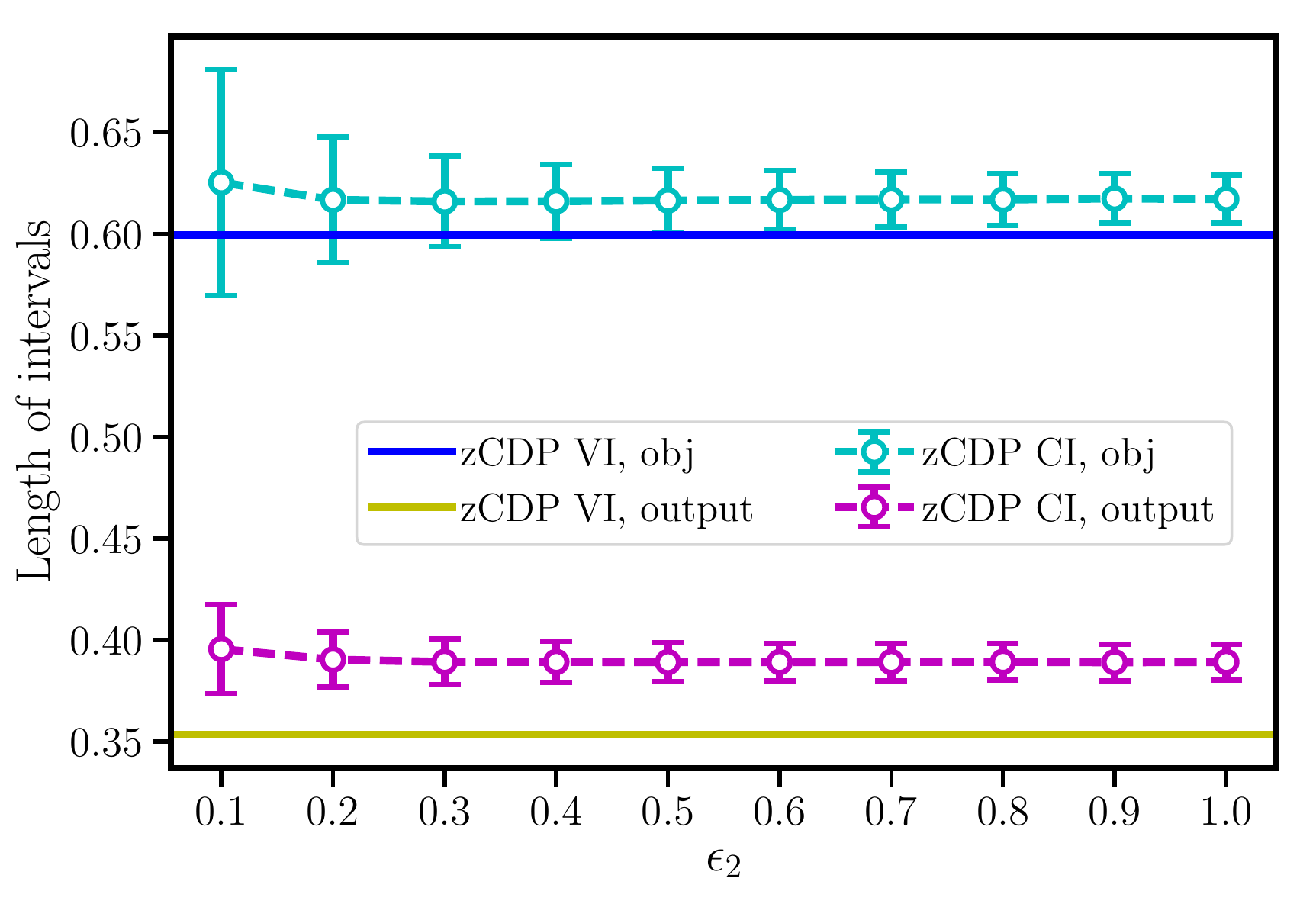}
}%
\hskip -8pt
\subfloat[Adult, LR, $\rho_3=\epsilon_3^2/2$]{
\includegraphics[width=0.33\textwidth]{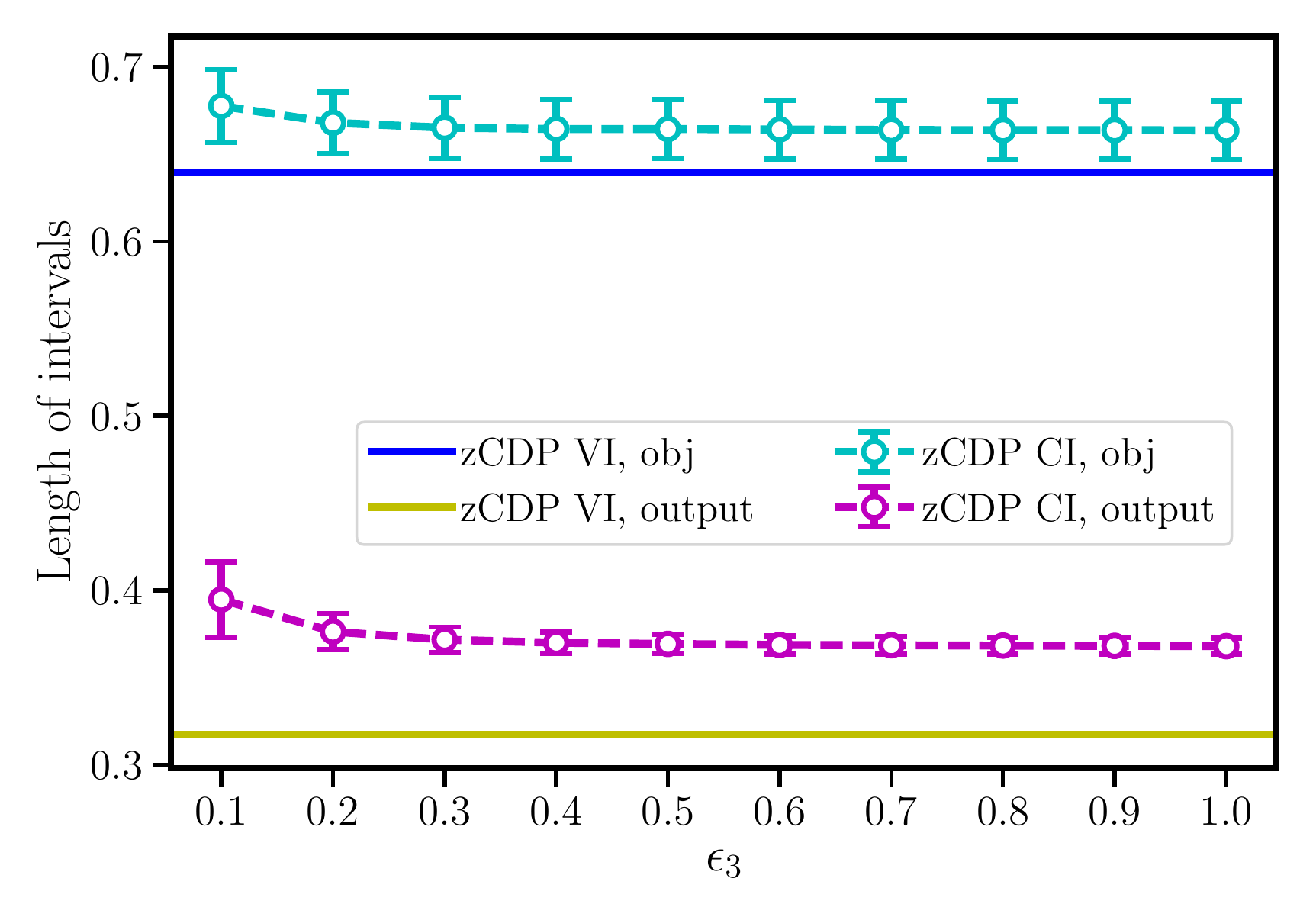}
}%
\hskip -8pt
\subfloat[US, SVM, $\rho_3=\epsilon_3^2/2$]{
\includegraphics[width=0.33\textwidth]{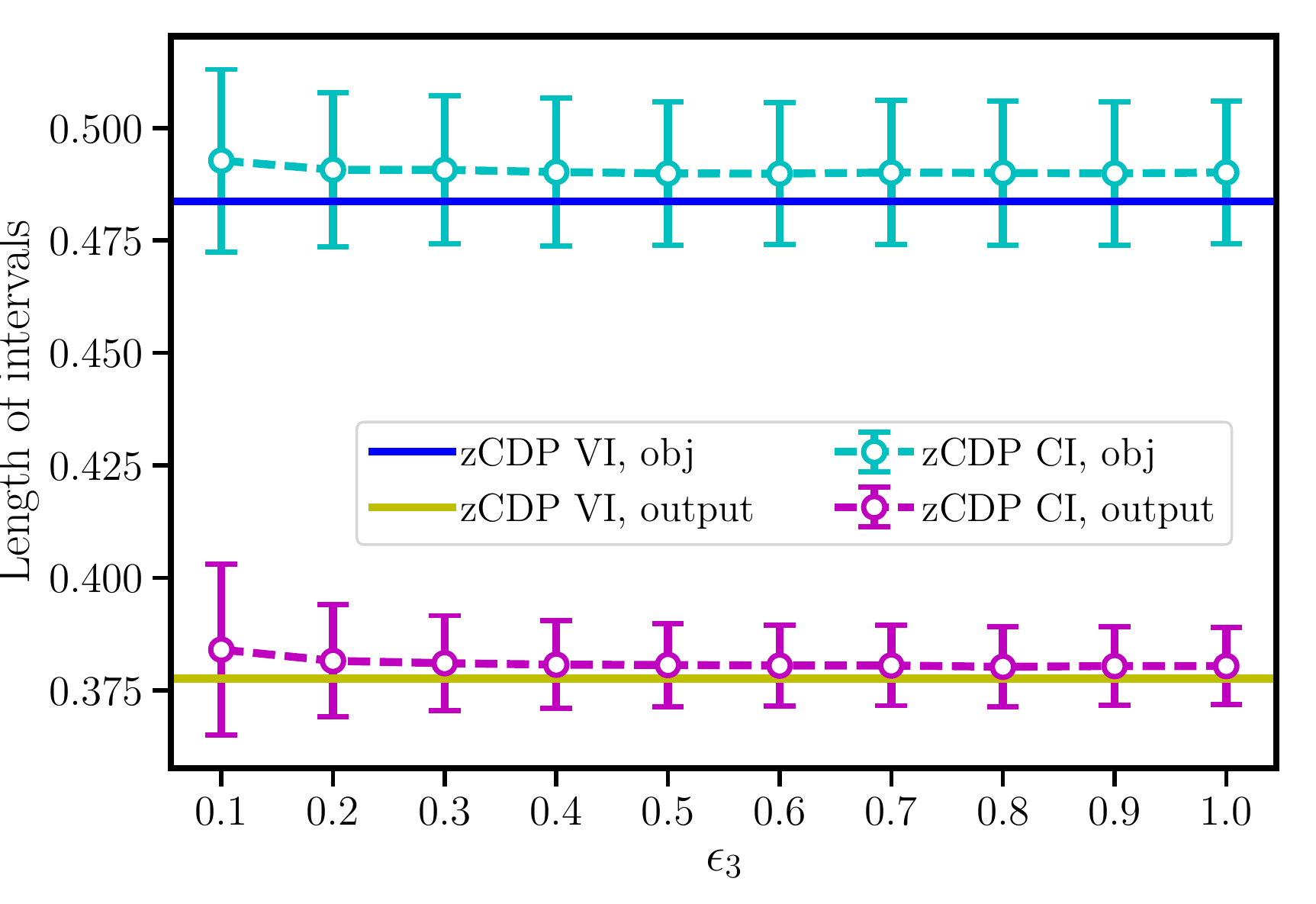}
}%
\vskip -8pt
\caption{Length of intervals for objective vs output perturbation with zCDP. Error bars correspond to one standard deviation for CI. (a)-(d): $\epsilon_1=0.5$, $\rho_1=0.125$, $\epsilon_2=\epsilon_3=0.25$, $\rho_2=\rho_3=0.03125$; (e)-(f): $n=45,211$, $d=10$, $\epsilon_2=\epsilon_3=0.25$, $\rho_2=\rho_3=0.03125$; (g): $n=30,162$, $d=10$, $\epsilon_1=0.5$, $\rho_1=0.125$, $\epsilon_3=0.25$, $\rho_3=0.03125$; (h)-(i): $n=30,162$ for Adult, $n=39,928$ for US, $d=10$, $\epsilon_1=0.5$, $\rho_1=0.125$, $\epsilon_2=0.25$, $\rho_2=0.03125$. Common parameters: $c=0.001$, $h=1.0$.}
\label{fig:zcdp}
\end{figure*}

\subsubsection{Objective vs Output Perturbation with zCDP}

In Figure~\ref{fig:zcdp}, we use zCDP and compare the confidence intervals for models learned using objective perturbation and output perturbation based ERM. Except for Figure~\ref{fig:zcdp}c and Figure~\ref{fig:zcdp}d, output perturbation outperforms objective perturbation when we protect zCDP. We note that the confidence intervals are close to the variability intervals, but the variability intervals for output perturbation are generally much shorter. For the two exceptions, with increasing dimensionality, the length of intervals with objective perturbation grows, but the length of intervals with output perturbation is almost steady, and there is a cross point for  the two methods. Together with the previous findings on $d$, it seems the dimensionality may add more uncertainty to the performance of the intervals than the other parameters. The effect of increasing sample size is the same as previous results.% The confidence intervals also have length quite close to the variability intervals.

\section{Conclusions}\label{sec:conclusions}
In this paper we proposed algorithms for generating confidence intervals of differentially private models that are learned using objective and output perturbations. Prior work only obtained confidence intervals for linear regression. Our experiments show that the confidence intervals obtain the desired coverage and provide intervals that are close to the true variability in the parameter coefficients. Future work includes generalizing these techniques to work with other model building algorithms.

%In this paper, we have solved the problem of generating confidence intervals for the model parameter of empirical risk minimization while protecting differential privacy and zero-concentrated differential privacy. We have demonstrated the validity of our private confidence intervals by coverage percentage through experiments on real datasets. We also conduct extensive experiments to compare different combinations of perturbation techniques and privacy enforcements. Output perturbation based empirical risk minimization with zero-concentrated differential privacy is optimal in our case. Our confidence intervals perform quite well by yielding similar sizes to that of the variability intervals.
%One possible future work direction would be designing better utility measures to quantify the private confidence intervals. 

%\end{document}  % This is where a 'short' article might terminate

% ensure same length columns on last page (might need two sub-sequent latex runs)
\balance

%ACKNOWLEDGMENTS are optional
%\section{Acknowledgments}

\bibliographystyle{abbrv}
\bibliography{ref}

\begin{thebibliography}{10}

\bibitem{ipums:us}
Ipums-usa, 2017.

\bibitem{ipums:br}
Minnesota population center. integrated public use microdata series,
  international: Version 6.5 brazil, 2017.

\bibitem{acharya2017differentially}
J.~Acharya, Z.~Sun, and H.~Zhang.
\newblock Differentially private testing of identity and closeness of discrete
  distributions.
\newblock {\em arXiv preprint arXiv:1707.05128}, 2017.

\bibitem{barrientos2017differentially}
A.~F. Barrientos, J.~P. Reiter, A.~Machanavajjhala, and Y.~Chen.
\newblock Differentially private significance tests for regression
  coefficients.
\newblock {\em arXiv preprint arXiv:1705.09561}, 2017.

\bibitem{bassily2014private}
R.~Bassily, A.~Smith, and A.~Thakurta.
\newblock Private empirical risk minimization: Efficient algorithms and tight
  error bounds.
\newblock In {\em Foundations of Computer Science (FOCS), 2014 IEEE 55th Annual
  Symposium on}, pages 464--473. IEEE, 2014.

\bibitem{blum2005practical}
A.~Blum, C.~Dwork, F.~McSherry, and K.~Nissim.
\newblock Practical privacy: the sulq framework.
\newblock In {\em Proceedings of the twenty-fourth ACM SIGMOD-SIGACT-SIGART
  symposium on Principles of database systems}, pages 128--138. ACM, 2005.

\bibitem{bun2016concentrated}
M.~Bun and T.~Steinke.
\newblock Concentrated differential privacy: Simplifications, extensions, and
  lower bounds.
\newblock In {\em Theory of Cryptography Conference}, pages 635--658. Springer,
  2016.

\bibitem{cai2017priv}
B.~Cai, C.~Daskalakis, and G.~Kamath.
\newblock Priv’it: Private and sample efficient identity testing.
\newblock In {\em International Conference on Machine Learning}, pages
  635--644, 2017.

\bibitem{chapelle2007training}
O.~Chapelle.
\newblock Training a support vector machine in the primal.
\newblock {\em Neural computation}, 19(5):1155--1178, 2007.

\bibitem{chaudhuri2011differentially}
K.~Chaudhuri, C.~Monteleoni, and A.~D. Sarwate.
\newblock Differentially private empirical risk minimization.
\newblock {\em Journal of Machine Learning Research}, 12(Mar):1069--1109, 2011.

\bibitem{chaudhuri2012near}
K.~Chaudhuri, A.~Sarwate, and K.~Sinha.
\newblock Near-optimal differentially private principal components.
\newblock In {\em Advances in Neural Information Processing Systems}, pages
  989--997, 2012.

\bibitem{chen2016differentially}
Y.~Chen, A.~Machanavajjhala, J.~P. Reiter, and A.~F. Barrientos.
\newblock Differentially private regression diagnostics.
\newblock In {\em Data Mining (ICDM), 2016 IEEE 16th International Conference
  on}, pages 81--90. IEEE, 2016.

\bibitem{d2015differential}
V.~D'Orazio, J.~Honaker, and G.~King.
\newblock Differential privacy for social science inference.
\newblock 2015.

\bibitem{DworkOurData}
C.~Dwork, K.~Kenthapadi, F.~McSherry, I.~Mironov, and M.~Naor.
\newblock Our data, ourselves: Privacy via distributed noise generation.
\newblock In {\em EUROCRYPT}, 2006.

\bibitem{dwork2006calibrating}
C.~Dwork, F.~McSherry, K.~Nissim, and A.~Smith.
\newblock Calibrating noise to sensitivity in private data analysis.
\newblock In {\em Theory of Cryptography Conference}, pages 265--284. Springer,
  2006.

\bibitem{dwork2014analyze}
C.~Dwork, K.~Talwar, A.~Thakurta, and L.~Zhang.
\newblock Analyze gauss: optimal bounds for privacy-preserving principal
  component analysis.
\newblock In {\em Proceedings of the forty-sixth annual ACM symposium on Theory
  of computing}, pages 11--20. ACM, 2014.

\bibitem{FriedmanDT}
A.~Friedman and A.~Schuster.
\newblock Data mining with differential privacy.
\newblock In {\em KDD}, 2010.

\bibitem{gaboardi2016differentially}
M.~Gaboardi, H.-W. Lim, R.~M. Rogers, and S.~P. Vadhan.
\newblock Differentially private chi-squared hypothesis testing: Goodness of
  fit and independence testing.
\newblock In {\em ICML'16 Proceedings of the 33rd International Conference on
  International Conference on Machine Learning-Volume 48}. JMLR, 2016.

\bibitem{jain2013differentially}
P.~Jain and A.~Thakurta.
\newblock Differentially private learning with kernels.
\newblock In {\em International Conference on Machine Learning}, pages
  118--126, 2013.

\bibitem{jain2014near}
P.~Jain and A.~G. Thakurta.
\newblock (near) dimension independent risk bounds for differentially private
  learning.
\newblock In {\em International Conference on Machine Learning}, pages
  476--484, 2014.

\bibitem{jiang2016wishart}
W.~Jiang, C.~Xie, and Z.~Zhang.
\newblock Wishart mechanism for differentially private principal components
  analysis.
\newblock In {\em Thirtieth AAAI Conference on Artificial Intelligence}, 2016.

\bibitem{kakizaki2017differentially}
K.~Kakizaki, K.~Fukuchi, and J.~Sakuma.
\newblock Differentially private chi-squared test by unit circle mechanism.
\newblock In {\em International Conference on Machine Learning}, pages
  1761--1770, 2017.

\bibitem{karwa2017finite}
V.~Karwa and S.~Vadhan.
\newblock Finite sample differentially private confidence intervals.
\newblock {\em arXiv preprint arXiv:1711.03908}, 2017.

\bibitem{kasiviswanathan2016efficient}
S.~P. Kasiviswanathan and H.~Jin.
\newblock Efficient private empirical risk minimization for high-dimensional
  learning.
\newblock In {\em International Conference on Machine Learning}, pages
  488--497, 2016.

\bibitem{kasiviswanathan2017private}
S.~P. Kasiviswanathan, K.~Nissim, and H.~Jin.
\newblock Private incremental regression.
\newblock In {\em Proceedings of the 36th ACM SIGMOD-SIGACT-SIGAI Symposium on
  Principles of Database Systems}, pages 167--182. ACM, 2017.

\bibitem{kifer2012private}
D.~Kifer, A.~Smith, and A.~Thakurta.
\newblock Private convex empirical risk minimization and high-dimensional
  regression.
\newblock In {\em Conference on Learning Theory}, pages 25--1, 2012.

\bibitem{Lichman:2013}
M.~Lichman.
\newblock {UCI} machine learning repository, 2013.

\bibitem{ligett2017accuracy}
K.~Ligett, S.~Neel, A.~Roth, B.~Waggoner, and S.~Z. Wu.
\newblock Accuracy first: Selecting a differential privacy level for accuracy
  constrained erm.
\newblock In {\em Advances in Neural Information Processing Systems}, pages
  2563--2573, 2017.

\bibitem{moro2014data}
S.~Moro, P.~Cortez, and P.~Rita.
\newblock A data-driven approach to predict the success of bank telemarketing.
\newblock {\em Decision Support Systems}, 62:22--31, 2014.

\bibitem{rogers2017new}
R.~Rogers and D.~Kifer.
\newblock A new class of private chi-square hypothesis tests.
\newblock In {\em Artificial Intelligence and Statistics}, pages 991--1000,
  2017.

\bibitem{rubinstein2009learning}
B.~I. Rubinstein, P.~L. Bartlett, L.~Huang, and N.~Taft.
\newblock Learning in a large function space: Privacy-preserving mechanisms for
  svm learning.
\newblock {\em arXiv preprint arXiv:0911.5708}, 2009.

\bibitem{sheffet2015private}
O.~Sheffet.
\newblock Private approximations of the 2nd-moment matrix using existing
  techniques in linear regression.
\newblock {\em arXiv preprint arXiv:1507.00056}, 2015.

\bibitem{sheffet2017differentially}
O.~Sheffet.
\newblock Differentially private ordinary least squares.
\newblock In {\em International Conference on Machine Learning}, pages
  3105--3114, 2017.

\bibitem{talwar2014private}
K.~Talwar, A.~Thakurta, and L.~Zhang.
\newblock Private empirical risk minimization beyond the worst case: The effect
  of the constraint set geometry.
\newblock {\em arXiv preprint arXiv:1411.5417}, 2014.

\bibitem{talwar2015nearly}
K.~Talwar, A.~G. Thakurta, and L.~Zhang.
\newblock Nearly optimal private lasso.
\newblock In {\em Advances in Neural Information Processing Systems}, pages
  3025--3033, 2015.

\bibitem{uhler2013privacy}
C.~Uhler, A.~Slavkovi{\'c}, and S.~E. Fienberg.
\newblock Privacy-preserving data sharing for genome-wide association studies.
\newblock {\em The Journal of privacy and confidentiality}, 5(1):137, 2013.

\bibitem{wang2018efficient}
D.~Wang, M.~Gaboardi, and J.~Xu.
\newblock Efficient empirical risk minimization with smooth loss functions in
  non-interactive local differential privacy.
\newblock {\em arXiv preprint arXiv:1802.04085}, 2018.

\bibitem{wang2017differentially}
D.~Wang, M.~Ye, and J.~Xu.
\newblock Differentially private empirical risk minimization revisited: Faster
  and more general.
\newblock In {\em Advances in Neural Information Processing Systems}, pages
  2719--2728, 2017.

\bibitem{wang2015differentially}
Y.~Wang, J.~Lee, and D.~Kifer.
\newblock Differentially private hypothesis testing, revisited.
\newblock {\em ArXiv e-prints}, 2015.

\bibitem{wu2015revisiting}
X.~Wu, M.~Fredrikson, W.~Wu, S.~Jha, and J.~F. Naughton.
\newblock Revisiting differentially private regression: Lessons from learning
  theory and their consequences.
\newblock {\em arXiv preprint arXiv:1512.06388}, 2015.

\bibitem{yu2014scalable}
F.~Yu, S.~E. Fienberg, A.~B. Slavkovi{\'c}, and C.~Uhler.
\newblock Scalable privacy-preserving data sharing methodology for genome-wide
  association studies.
\newblock {\em Journal of biomedical informatics}, 50:133--141, 2014.

\bibitem{yu2014differentially}
F.~Yu, M.~Rybar, C.~Uhler, and S.~E. Fienberg.
\newblock Differentially-private logistic regression for detecting multiple-snp
  association in gwas databases.
\newblock In {\em International Conference on Privacy in Statistical
  Databases}, pages 170--184. Springer, 2014.

\bibitem{zhang2013privgene}
J.~Zhang, X.~Xiao, Y.~Yang, Z.~Zhang, and M.~Winslett.
\newblock Privgene: differentially private model fitting using genetic
  algorithms.
\newblock In {\em Proceedings of the 2013 ACM SIGMOD International Conference
  on Management of Data}, pages 665--676. ACM, 2013.

\bibitem{zhang2012functional}
J.~Zhang, Z.~Zhang, X.~Xiao, Y.~Yang, and M.~Winslett.
\newblock Functional mechanism: regression analysis under differential privacy.
\newblock {\em Proceedings of the VLDB Endowment}, 5(11):1364--1375, 2012.

\bibitem{zhang2017efficient}
J.~Zhang, K.~Zheng, W.~Mou, and L.~Wang.
\newblock Efficient private erm for smooth objectives.
\newblock In {\em Proceedings of the 26th International Joint Conference on
  Artificial Intelligence}, pages 3922--3928. AAAI Press, 2017.

\end{thebibliography}

%\subsection{References}
%Generated by bibtex from your ~.bib file.  Run latex,
%then bibtex, then latex twice (to resolve references).

%APPENDIX is optional.
% ****************** APPENDIX **************************************
% Example of an appendix; typically would start on a new page
%pagebreak

\conferenceversion{\newpage}

\begin{appendix}
\section{Proof of Theorem~\lowercase{\ref{thm:erm2}}}\label{app:erm2}
\thmerm*
\begin{proof}
The loss function $f(\vec{x},y,\theta)=f(y\theta^T\vec{x})$. Because $f(\cdot)$ and $\|\theta\|_2^2$ are convex and differentiable, given any data set $\mathcal{D}$, the gradient of the objective function equals to 0 at the empirical minimizer $\tilde{\theta}$:
\begin{equation*}
\beta=-2nc\tilde{\theta}-\sum_{i=1}^ny_if^{\prime}(y_i\tilde{\theta}^T\vec{x}_i)\vec{x}_i.
\end{equation*}

To show $\epsilon$-differential privacy, we compute the ratio of the densities of $\tilde{\theta}$ under the two neighboring data sets $\mathcal{D}$ and $\mathcal{D}^{\prime}$.
\[\frac{g(\tilde{\theta}|\mathcal{D})}{g(\tilde{\theta}|\mathcal{D}^{\prime})}=\frac{\mathtt{v}(\beta|\mathcal{D})}{\mathtt{v}(\beta^{\prime}|\mathcal{D}^{\prime})}\cdot\frac{|\det(\mathbf{J}(\tilde{\theta}\rightarrow \beta|\mathcal{D}))|^{-1}}{|\det(\mathbf{J}(\tilde{\theta}\rightarrow \beta^{\prime}|\mathcal{D}^{\prime}))|^{-1}},\]
where $\mathbf{J}(\tilde{\theta}\rightarrow \beta|\mathcal{D})$ is the Jacobian matrix of the mapping from $\tilde{\theta}$ to $\beta$.

Given $\mathcal{D}$, the $(j,k)$-th entry of $\mathbf{J}(\tilde{\theta}\rightarrow \beta|\mathcal{D})$ is
\[\frac{\partial \beta^{(j)}}{\partial\tilde{\theta}^{(k)}}=-2nc\mathbf{1}_{j=k}-\sum_{i=1}^ny_i^2f^{\prime\prime}(y_i\tilde{\theta}^T\vec{x}_i)\vec{x}_i^{(j)}\vec{x}_i^{(k)},\]
where $\mathbf{1}$ is the indicator function. The Jacobian is well defined since $f(\cdot)$ is doubly differentiable.

Given $\mathcal{D}$ and $\mathcal{D}^{\prime}$, define
\begin{align*}
A&=2nc\mathbf{I}_d+\sum_{i=1}^{n-1}y_i^2f^{\prime\prime}(y_i\tilde{\theta}^T\vec{x}_i)\vec{x}_i\vec{x}_i^T,\\
\vec{u}\vec{u}^T&=y_n^2f^{\prime\prime}(y_n\tilde{\theta}^T\vec{x}_n)\vec{x}_n\vec{x}_n^T,\\
\vec{v}\vec{v}^T&=y_z^2f^{\prime\prime}(y_z\tilde{\theta}^T\vec{x}_z)\vec{x}_z\vec{x}_z^T.
\end{align*}
Then, $\mathbf{J}(\tilde{\theta}\rightarrow \beta|\mathcal{D})=-(A+\vec{u}\vec{u}^T)$, $\mathbf{J}(\tilde{\theta}\rightarrow \beta^{\prime}|\mathcal{D}^{\prime})=-(A+\vec{v}\vec{v}^T)$.

Therefore,
\begin{align*}
\frac{|\det(\mathbf{J}(\tilde{\theta}\rightarrow \beta|\mathcal{D}))|^{-1}}{|\det(\mathbf{J}(\tilde{\theta}\rightarrow \beta^{\prime}|\mathcal{D}^{\prime}))|^{-1}}=&\frac{|\det(A+\vec{v}\vec{v}^T)|}{|\det A+\vec{u}\vec{u}^T|}\\
=&\frac{|(1+\vec{v}^TA^{-1}\vec{v})\det(A)|}{|(1+\vec{u}^TA^{-1}\vec{u})\det(A)|}\\
=&\frac{|1+\vec{v}^TA^{-1}\vec{v}|}{|1+\vec{u}^TA^{-1}\vec{u}|}.
\end{align*}
We can see $A$ is a symmetric positive definite matrix and its eigenvalues are at least $2nc$. So the eigenvalues of $A^{-1}$ are at most $\frac{1}{2nc}$.
Since $y\in\{-1,1\}$, $\|\vec{x}\|_2\leq 1$ and $f^{\prime\prime}(\cdot)\leq t$,
\begin{align*}
\frac{|1+\vec{v}^TA^{-1}\vec{v}|}{|1+\vec{u}^TA^{-1}\vec{u}|}&\leq \frac{1+\|\vec{v}^TA^{-1}\vec{v}\|_2}{1}\\
&\quad\text{(by the triangle inequality and $\vec{u}^TA^{-1}\vec{u}\geq 0$)}\\
&\leq 1+\|v\|_2\|A^{-1}\vec{v}\|_2
\leq 1+\|v\|_2^2\|A^{-1}\|_2\\
&= 1+y_z^2f^{\prime\prime}(y_z\tilde{\theta}^T\vec{x}_z)\|\vec{x}_z\|_2^2\|A^{-1}\|_2\\
&\leq 1+t\|A^{-1}\|_2
\leq 1+\frac{t}{2nc}.
\end{align*}

Then by the definition of $\epsilon^{\prime}$, $1+\frac{t}{2nc}=e^{\epsilon-\epsilon^{\prime}}$.

Next, we bound the ratio of the densities of the noise vectors:
\[\beta^{\prime}-\beta=y_nf^{\prime}(y_n\tilde{\theta}^T\vec{x}_n)\vec{x}_n-y_zf^{\prime}(y_z\tilde{\theta}^T\vec{x}_z)\vec{x}_z.\]
Since $y\in\{-1,1\}$, $\|\vec{x}\|_2\leq 1$ and $f^{\prime}(\cdot)\leq 1$,
\[\|\beta^{\prime}\|_2-\|\beta\|_2\leq\|\beta^{\prime}-\beta\|_2\leq 2,\]
so
\[\frac{\mathtt{v}(\beta|\mathcal{D})}{\mathtt{v}(\beta^{\prime}|\mathcal{D}^{\prime})}=\frac{e^{-\epsilon^{\prime}\|\beta\|_2/2}}{e^{-\epsilon^{\prime}\|\beta^{\prime}\|_2/2}}=e^{\epsilon^{\prime}(\|\beta^{\prime}\|_2-\|\beta\|_2)/2}\leq e^{\epsilon^{\prime}}.\]

Therefore, 
\[\frac{g(\tilde{\theta}|\mathcal{D})}{g(\tilde{\theta}|\mathcal{D}^{\prime})}=\frac{\mathtt{v}(\beta|\mathcal{D})}{\mathtt{v}(\beta^{\prime}|\mathcal{D}^{\prime})}\cdot\frac{|\det(\mathbf{J}(\tilde{\theta}\rightarrow \beta|\mathcal{D}))|^{-1}}{|\det(\mathbf{J}(\tilde{\theta}\rightarrow \beta^{\prime}|\mathcal{D}^{\prime}))|^{-1}} \leq e^{\epsilon^{\prime}}\cdot e^{\epsilon-\epsilon^{\prime}}=e^{\epsilon}.\]
\end{proof}

\section{Proof of Lemma~\lowercase{\ref{lem:privspdmat}}}\label{app:privspdmat}
\privspdmat*
\begin{proof}
In Algorithm~\ref{alg:privspdmat}, only Line 8 touches the matrix $M$ since Lines 1 through 7 are sampling from the noise distribution and all other lines are just post-processing on the perturbed matrix $\tilde{M}$. So we just need to prove getting $\tilde{M}$ through Line 8 satisfies differential privacy and zCDP.

\begin{enumerate}
\item When protecting differential privacy:

To simplify, let $\tovec{(M)}$ be the vector representation of $M$ by stacking its rows. Let $M^{\prime}$ be the neighbor of $M$ which differs in only one entry. Then given the density of the noise in Equation~\ref{eq:noisedist} and the $L_2$ sensitivity $Sens(M)$,
\begin{align*}
\frac{\tovec{(\tilde{M})}|M}{\tovec{(\tilde{M})}|M^{\prime}} =&\frac{\mathtt{v}[\tovec{(\tilde{M}-M)}]}{\mathtt{v}[\tovec{(\tilde{M}-M^{\prime})}]}\\
=&\frac{e^{-\phi/Sens(M) \cdot\|\tovec{(\tilde{M}-M)}\|_2}}{e^{-\phi/Sens(M) \cdot\|\tovec{(\tilde{M}-M^{\prime})}\|_2}}\\
=&e^{\phi/Sens(M) \cdot(\|\tovec{(\tilde{M}-M^{\prime})}\|_2-\|\tovec{(\tilde{M}-M^{\prime})}\|_2)}\\
\leq & e^{\phi}.
\end{align*}

Therefore, Line 8 satisfies $\phi$-differential privacy.
\item When protecting zCDP:

By Proposition~\ref{pro:gaussmechanism}, Line 8 satisfies $\phi$-zCDP.
\end{enumerate}

The rest of the algorithm is just post-processing on the perturbed matrix. By the post-processing property of differential privacy and zCDP, Algorithm~\ref{alg:privspdmat} satisfies $\phi$-differential privacy and $\phi$-zCDP.
\end{proof}

\section{Proof of Theorem~\lowercase{\ref{thm:ciobj}}}\label{app:ciobj}
\thmciobj*
\begin{proof}
In Algorithm~\ref{alg:ciobj}, there are three parts that touch the true data. 

First, the computation of the minimizer $\tilde{\theta}$ to the objective function. Based on Theorem~\ref{thm:erm2}, the computation is $\phi_1$-differentially private as long as the loss function $f(\cdot)$ is convex and doubly differentiable with $|f^{\prime}(\cdot)|\leq 1$ and $|f^{\prime\prime}(\cdot)|\leq t$ for some finite $t$. By Proposition~\ref{pro:dptozcdp}, the computation also satisfies $(\phi_1^2/2)$-zCDP.

The next two parts are the computations of the Hessian and the covariance matrix. By Lemma~\ref{lem:privspdmat}, the computation of the Hessian satisfies $\phi_2$-differential privacy and $\phi_2$-zCDP, and the computation of the covariance matrix satisfies $\phi_3$-differential privacy and $\phi_3$-zCDP.

All other computations are post-processing and therefore do not violate differential privacy or zCDP. By the composition theorem of differential privacy and zCDP, Algorithm~\ref{alg:ciobj} satisfies $(\phi_1+\phi_2+\phi_3)$-differential privacy and $(\phi_1^2/2+\phi_2+\phi_3)$-zCDP.
\end{proof}

\section{Proof of Theorem~\lowercase{\ref{thm:cioutput}}}\label{app:cioutput}
\thmcioutput*
\begin{proof}
In Algorithm~\ref{alg:cioutput}, there are three parts that touch the true data. 

First, the computation of the minimizer to the objective function. By Theorem~\ref{thm:ermoutput}, this process satisfies $\phi_1$-differential privacy and $\phi_1$-zCDP as long as the loss function $f(\cdot)$ is convex and differentiable with $|f^{\prime}(\cdot)|\leq 1$.

The next two parts are the computations of the Hessian and the covariance matrix. By Lemma~\ref{lem:privspdmat}, the computation of the Hessian satisfies $\phi_2$-differential privacy and $\phi_2$-zCDP, and the computation of the covariance matrix satisfies $\phi_3$-differential privacy and $\phi_3$-zCDP.

All other computations are post-processing and therefore do not violate differential privacy or zCDP. By the composition theorem of differential privacy and zCDP, Algorithm~\ref{alg:cioutput} satisfies $(\phi_1+\phi_2+\phi_3)$-differential privacy and $(\phi_1+\phi_2+\phi_3)$-zCDP.
\end{proof}

\section{Proof of Lemma~\lowercase{\ref{lem:senscovlr}}}\label{app:senscovlr}
\senscovlr*
\begin{proof}
We compute the $L_2$ sensitivity for $\Sigma$ as
\begin{align*}
& \max_{\mathcal{D},\mathcal{D}^{\prime}}\Big\Vert\tovec{(\Sigma_\mathcal{D}-\Sigma_{\mathcal{D}^{\prime}})}\Big\Vert_2\\
=& \max_{\vec{x}_n,y_n,\vec{x}_z,y_z} \frac{1}{n}\Big\Vert \tovec{\left[\nabla(f(\vec{x}_n,y_n,\theta_0))[\nabla f(\vec{x}_n,y_n,\theta_0)]^T\right]}-\\
& \tovec{\left[\nabla(f(\vec{x}_z,y_z,\theta_0))[\nabla f(\vec{x}_z,y_z,\theta_0)]^T\right]}\Big\Vert_2\\
=& \max_{\vec{x}_n,y_n,\vec{x}_z,y_z}\frac{1}{n}\Big\Vert\left[S(-y_n\theta_0^T\vec{x}_n)^2\tovec{(\vec{x}_n\vec{x}_n^T)}-\right.\\
& \left.S(-y_z\theta_0^T\vec{x}_z)^2\tovec{(\vec{x}_z\vec{x}_z^T)}\right]\Big\Vert_2\\
\leq & \max_{y,\vec{x}}\frac{2}{n}S(-y\theta_0^T\vec{x})^2\Vert \tovec{(\vec{x}\vec{x}^T)} \Vert_2.
\end{align*}

Since $\Vert \vec{x}\Vert_2\leq 1$, we get $\Vert\tovec{(\vec{x}\vec{x}^T)}\Vert_2 = \sqrt{\sum_{1\leq j,k\leq d}x[j]^2x[k]^2} = \sqrt{(\sum_{j=1}^dx[j]^2)^2}\leq 1$. From the Cauchy-Schwarz inequality, we get $\Vert\theta^T\vec{x}\Vert_2\leq\Vert\theta\Vert_2\Vert\vec{x}\Vert_2\leq\Vert\theta\Vert_2$. We know either $\theta^T\vec{x}=\Vert\theta^T\vec{x}\Vert_2$ or $\theta^T\vec{x}=-\Vert\theta^T\vec{x}\Vert_2$, then $-\Vert\theta\Vert_2\leq\theta^T\vec{x}\leq \Vert\theta\Vert_2$. Based on the fact that the sigmoid function $S(t)$ is monotonically increasing in $t$ and $y\in\{-1,1\}$,
\begin{align*}
& \max_{y,\vec{x}}\frac{2}{n}S(-y\theta_0^T\vec{x})^2\Vert \tovec{(\vec{x}\vec{x}^T)} \Vert_2
\leq \frac{2}{n}S(\Vert\theta_0\Vert_2)^2.
\end{align*}
\end{proof}

\section{Proof of Lemma~\lowercase{\ref{lem:senshessianlr}}}\label{app:senshessianlr}
\senshessianlr*
\begin{proof}
We compute the $L_2$ sensitivity for $H[J_n(\tilde{\theta})]$ as:
\begin{align*}
& \max_{\mathcal{D},\mathcal{D}^{\prime}} \Big\Vert\tovec{(H[J_n(\mathcal{D},\tilde{\theta})]-H[J_n(\mathcal{D}^{\prime},\tilde{\theta})])}\Big\Vert_2\\
=& \max_{\vec{x}_n,y_n\vec{x}_z,y_z}\frac{1}{n}\Big\Vert\tovec{\left[H[f(\vec{x}_n,y_n,\tilde{\theta})]-H[f(\vec{x}_z,y_z,\tilde{\theta})]\right]}\Big\Vert_2\\
=& \max_{\vec{x}_n,y_n\vec{x}_z,y_z}\frac{1}{n}\Big\Vert\tovec{\left[S(-y_n\tilde{\theta}^T\vec{x}_n)S(y_n\tilde{\theta}^T\vec{x}_n)\vec{x}_n\vec{x}_n^T\right]} \\
& - \tovec{\left[S(-y_z\tilde{\theta}^T\vec{x}_z)S(y_z\tilde{\theta}^T\vec{x}_z)\vec{x}_z\vec{x}_z^T\right]}\Big\Vert_2\\
\leq & \max_{\vec{x},y}\frac{2}{n}S(-y\tilde{\theta}^T\vec{x})S(y\tilde{\theta}^T\vec{x})\Vert\tovec{(\vec{x}\vec{x}^T)}\Vert_2\\
\leq & \max_{\vec{x}}\frac{2}{n}S(\tilde{\theta}^T\vec{x})S(-\tilde{\theta}^T\vec{x}).
\end{align*}

In Appendix~\ref{app:senscovlr}, we have shown that $-\Vert\theta\Vert_2\leq\theta^T\vec{x}\leq \Vert\theta\Vert_2$. The function $S(\tilde{\theta}^T\vec{x})S(-\tilde{\theta}^T\vec{x})$ achieves the maximum $1/4$ at $\tilde{\theta}^T\vec{x}=0$. So,
\begin{align*}
\max_{\vec{x}}\frac{2}{n}S(\tilde{\theta}^T\vec{x})S(-\tilde{\theta}^T\vec{x}) \leq \frac{1}{2n}.
\end{align*}
\end{proof}

\section{Proof of Lemma~\lowercase{\ref{lem:senscovsvm}}}\label{app:senscovsvm}
\senscovsvm*
\begin{proof}
We compute the $L_2$ sensitivity for $\Sigma$ as:
\begin{align*}
& \max_{\mathcal{D},\mathcal{D}^{\prime}}\Big\Vert\tovec{(\Sigma^{\prime}_\mathcal{D}-\Sigma^{\prime}_{\mathcal{D}^{\prime}})}\Big\Vert_2\\
=& \max_{\vec{x}_n,y_n,\vec{x}_z,y_z}\frac{1}{n}\Big\Vert\tovec{\left[\nabla f(\vec{x}_n,y_n,\theta_0)[\nabla f(\vec{x}_n,y_n,\theta_0)]^T\right]}\\
& -\tovec{\left[\nabla f(\vec{x}_z,y_z,\theta_0)[\nabla f(\vec{x}_z,y_z,\theta_0)]^T\right]}\Big\Vert_2\\
\leq & \max_{y,\vec{x}}\frac{2}{n}\Big\Vert \tovec{\left[\nabla f(\vec{x},y,\theta_0)[\nabla f(\vec{x},y,\theta_0)]^T\right]} \Big\Vert_2.
\end{align*}
There are three cases:
\begin{enumerate}
\item If $y\theta^T\vec{x}>1+h$,
\[\max_{y,\vec{x}}\frac{2}{n}\Big\Vert \tovec{\left[\nabla f(\vec{x},y,\theta_0)[\nabla f(\vec{x},y,\theta_0)]^T\right]} \Big\Vert_2=0.\]
\item If $|1-y\theta^T\vec{x}|\leq h$,

\begin{align*}
& \max_{y,\vec{x}}\frac{2}{n}\Big\Vert \tovec{\left[\nabla f(\vec{x},y,\theta_0)[\nabla f(\vec{x},y,\theta_0)]^T\right]} \Big\Vert_2\\
=& \max_{y,\vec{x}}\frac{2}{n}\left[\frac{y}{2h}(y\theta_0^T\vec{x}-1-h)\right]^2\|\tovec(\vec{x}\vec{x}^T)\|_2\\
\leq& \max_{y,\vec{x}}\frac{1}{2nh^2}(y\theta_0^T\vec{x}-1-h)^2\\
\leq& \frac{1}{2nh^2} \cdot 4h^2\quad\text{since $y\theta_0^T\vec{x}-1\in[-h,h]$}\\
=& 2/n.
\end{align*}
\item If $y\theta^T\vec{x}<1-h$,
\begin{align*}
& \max_{y,\vec{x}}\frac{2}{n}\Big\Vert \tovec{\left[\nabla f(\vec{x},y,\theta_0)[\nabla f(\vec{x},y,\theta_0)]^T\right]} \Big\Vert_2\\
=&  \max_{y,\vec{x}}\frac{2}{n}y^2\|\tovec(\vec{x}\vec{x}^T)\|_2\\
\leq& 2/n.
\end{align*}

Therefore, in all cases, \[\max_{y,\vec{x}}\frac{2}{n}\Big\Vert \tovec{\left[\nabla f(\vec{x},y,\theta_0)[\nabla f(\vec{x},y,\theta_0)]^T\right]} \Big\Vert_2
\leq 2/n.\]
\end{enumerate}
\end{proof}

\section{Proof of Lemma~\lowercase{\ref{lem:senshessiansvm}}}\label{app:senshessiansvm}
\senshessiansvm*
\begin{proof}
The $L_2$ sensitivity for $H[J_n(\tilde{\theta})]$ can be computed as:
\begin{align*}
& \max_{\mathcal{D},\mathcal{D}^{\prime}}\Big\|\tovec\{H[J_n(\mathcal{D},\tilde{\theta})]-H[J_n(\mathcal{D}^{\prime},\tilde{\theta})]\}\Big\|_2\\
=&\max_{\vec{x}_n,y_n,\vec{x}_z,y_z} \frac{1}{n}\Big\|\tovec\{H[f(y_n\tilde{\theta}\vec{x}_n)]-H[f(y_z\tilde{\theta}\vec{x}_z)]\}\Big\|_2.
\end{align*}
There are two cases:
\begin{enumerate}
\item If $|1-y\theta^T\vec{x}|\leq h$,
\begin{align*}
& \max_{\vec{x}_n,y_n,\vec{x}_z,y_z} \frac{1}{n}\Big\|\tovec\{H[f(y_n\tilde{\theta}\vec{x}_n)]-H[f(y_z\tilde{\theta}\vec{x}_z)]\}\Big\|_2\\
=& \max_{\vec{x}_n,y_n,\vec{x}_z,y_z} \frac{1}{n}\Big\|\tovec\left\{\frac{y_n^2}{2h}\vec{x}_n\vec{x}_n^T-\frac{y_z^2}{2h}\vec{x}_z\vec{x}_z^T\right\}\Big\|_2\\
\leq & \max_{\vec{x},y}\frac{2}{n}\cdot\frac{1}{2h}\|\tovec(\vec{x}\vec{x}^T)\|_2\\
\leq & 1/(nh).
\end{align*}
\item Otherwise,
\[\max_{\vec{x}_n,y_n,\vec{x}_z,y_z} \frac{1}{n}\Big\|\tovec\{H[f(y_n\tilde{\theta}\vec{x}_n)]-H[f(y_z\tilde{\theta}\vec{x}_z)]\}\Big\|_2=0.\]
\end{enumerate}
Therefore, in all cases,
\[\max_{\vec{x}_n,y_n,\vec{x}_z,y_z} \frac{1}{n}\Big\|\tovec\{H[f(y_n\tilde{\theta}\vec{x}_n)]-H[f(y_z\tilde{\theta}\vec{x}_z)]\}\Big\|_2\leq 1/(nh).\]
\end{proof}

% Discussion:

% First , we evaluate how good our confidence intervals are:
% \begin{itemize}
% \item Do they have the right coverage? One way to check is set a certain $\epsilon$ for ERM. Then vary $\epsilon_2$ used for confidence intervals. See how coverage changes (at what point, $\epsilon$ vs n vs d) does the noise ruin the confidence intervals.
% \item Are our confidence intervals too long? How do we check? We create variabiltiy intervals. Sample data, get dp estimate. Repeat many times. For each parameter, find that interval that covers 95\% of its appearances. This is strictly a property of the ERM algorithm and is the best any confidence interval can do. We compare our conf interval length to the variability length.
% \end{itemize}

% Now that we established our confidence intervals are a good way to measure uncertainty in ERM parameter estimates, the next step is to use it to analyze how good ERM is.

% So we find which settings of epsilon, n, d, c give us accurate estimates, meaning that the confidence interval is short enough. For example, the confidence interval does not contain 0 and is far from 0. This means that we know the sign of the parameter very well. Or for example, the confidence interval has length 0.1  (meaning, if it contains 0, it really is close to 0)

% make sure to use the constant feature! is often has one of the most important coefficients (known as th ebias or the offset).

% Output perturbation confidence interval. Compare output perturbation confidence intervals to ERM confidence intervals (which is smaller)

\end{appendix}

\end{document}